\documentclass{article}

\usepackage{arxiv}

\usepackage{graphicx} 
\usepackage{mwe}
\usepackage{amsmath,amsfonts}
\usepackage{algorithmic}
\usepackage{algorithm}
\usepackage{array}
\usepackage[caption=false,font=normalsize,labelfont=sf,textfont=sf]{subfig}
\usepackage{textcomp}
\usepackage{stfloats}
\usepackage{url}
\usepackage{verbatim}
\usepackage{appendix}

\usepackage{cite}
\usepackage{hyperref}  

\usepackage{multicol}

\usepackage{authblk}

\usepackage{xcolor}

\usepackage{amsthm}
\usepackage{thmtools,thm-restate}

\DeclareMathOperator*{\argmin}{arg\,min}

\newtheorem{theorem}{Theorem}[section]
\newtheorem{lemma}[theorem]{Lemma}
\newtheorem{proposition}[theorem]{Proposition}

\newtheorem{remark}[theorem]{Remark}
\newtheorem{hypothesis}[theorem]{Hypothesis}

\newcommand{\R}{{\mathbb R}}
\newcommand{\E}{{\mathbb E}}
\newcommand{\Z}{{\cal Z}}

\title{MGD: Moment Guided Diffusion for Maximum Entropy Generation}

\author[1]{Etienne Lempereur}
\author[1]{Nathanaël Cuvelle--Magar}
\author[2]{Florentin Coeurdoux}
\author[3,4]{St\'ephane Mallat}
\author[5,6]{Eric Vanden-Eijnden}
\affil[1]{D\'epartement d’informatique, ENS, Universit\'e PSL, Paris, France}
\affil[2]{Capital Fund Management, Paris, France}
\affil[3]{Coll\`ege de France, Paris, France}
\affil[4]{Flatiron Institute, New York, USA}
\affil[5]{Courant Institute of Mathematical Sciences, New York University, New York, USA }
\affil[6]{ ML Lab, Capital Fund Management, Paris, France}
\begin{document}

\maketitle
\begin{abstract}
Generating samples from limited information is a fundamental problem across scientific domains. Classical maximum entropy methods provide principled uncertainty quantification from moment constraints but require sampling via MCMC or Langevin dynamics, which typically exhibit exponential slowdown in high dimensions. In contrast, generative models based on diffusion and flow matching efficiently transport noise to data but offer limited theoretical guarantees and can overfit when data is scarce. We introduce Moment Guided Diffusion (MGD), which combines elements of both approaches. Building on the stochastic interpolant framework, MGD samples maximum entropy distributions by solving a stochastic differential equation that guides moments toward prescribed values in finite time, thereby avoiding slow mixing in equilibrium-based methods. We formally obtain, in the large-volatility limit, convergence of MGD to the maximum entropy distribution and derive a tractable estimator of the resulting entropy computed directly from the dynamics. Applications to financial time series, turbulent flows, and cosmological fields using wavelet scattering moments yield estimates of negentropy for high-dimensional multiscale processes.
\end{abstract}

\begin{multicols}{2}

\section{Introduction}
\label{sec:intro}

Generating new realizations of a random variable $X \in \mathbb{R}^d$ from limited information arises across scientific domains, from synthesizing physical fields in computational science to creating scenarios for risk assessment in quantitative finance. Many approaches to this problem have been proposed, but two stand out for their success: the classical maximum entropy framework introduced by Jaynes \cite{jaynes1957information} when moment information is available, and the modern generative modeling approach with deep neural networks \cite{goodfellow2020generative,lipman2022flow,albergo2022stochastic,albergo2023stochastic,ho2020denoising,song2021scorebased,lai2025principles} 
that operate when raw data samples can be accessed. These approaches take different perspectives on the problem---principled uncertainty quantification versus flexible distribution learning---suggesting potential benefits from blending both.

The maximum entropy approach provides principled uncertainty quantification when the available information consists of moments $\mathbb{E}[\phi(X)] \in \R^r$ for a specified moment function (or observable) $\phi: \mathbb{R}^d \to \mathbb{R}^r$. Jaynes' principle selects the unique distribution that maximizes entropy, if it exists. It is the least committal choice consistent with available information. 
It provides principled protection against overfitting: generated samples are diverse within the constraint set but do not hallucinate correlations beyond what $\phi$ captures. This is particularly valuable when data is scarce.
This maximum entropy distribution has an exponential density $p_{\theta_*}(x) = {\cal Z}_{\theta_*}^{-1}\, e^{-{\theta_*}^\top \phi(x)}$, where $\theta_*$ are Lagrange multipliers and ${\cal Z}_{\theta_*}$ is the normalization constant. While theoretically elegant and providing rigorous control over uncertainty, this approach is not a generative model \textit{per~se}. 
Classical maximum entropy estimation \cite{kullback1997information,cover1999elements,bishop2006pattern} requires sampling from intermediate distributions to compute log-likelihood gradients, both for estimating the Lagrange multipliers $\theta_{*}$ and for generating samples from $p_{\theta_*}$. Unfortunately, samplers based on MCMC or on a Langevin equation suffer from critical slowing down \cite{zinn2021quantum,sokal1991beat}: sampling becomes prohibitively expensive in high dimension for non-convex Gibbs energies $\theta_*^\top \phi(x)$.

Recent generative modeling approaches emphasize flexible distribution learning when samples $(x^i)_{i\leq n}$ are available. Modern generative models---notably score-based diffusion~\cite{ho2020denoising,song2021scorebased,lai2025principles} and flow matching with stochastic interpolants~\cite{albergo2022stochastic,lipman2022flow,liu2022flow}---learn to sample from an approximation of the underlying distribution by transporting Gaussian noise to data samples along carefully designed paths using Ordinary Differential Equations (ODE) or Stochastic Differential Equations (SDE), with a drift estimated by quadratic regression with a neural network. 
This transport avoids the exponential scaling with barrier heights that plagues classical MCMC and Langevin sampling. However, this flexibility comes at a cost: they provide no explicit control over statistical moments and their approximation error remains theoretically uncontrolled, making them prone to overfitting when data is scarce \cite{kadkhodaie2024generalization}.

We introduce a Moment Guided Diffusion (MGD), which blends both paradigms.
MGD samples maximum entropy distributions when data samples are available, using a transport that guides moments estimated from these data. To achieve this, MGD relies on two key ingredients. First, it uses a diffusive process $X_t$ whose moments match those of a stochastic interpolant $I_t$ that continuously transforms Gaussian noise into data: $\mathbb{E}[\phi(X_t)] = \mathbb{E}[\phi(I_t)]$ for all $t \in [0,1]$. This diffusion steers the distribution of the process from noise to data along a homotopic path, {that is, a path that continuously deforms the initial distribution into the target one,} achieving non-equilibrium transport in finite time and avoiding the critical slowing down that plagues classical Langevin dynamics. Second, the SDE includes a tunable volatility $\sigma$ that controls convergence to the maximum entropy distribution. As $\sigma$ increases, under appropriate assumptions we prove that the process converges to the maximum entropy among all distributions satisfying the moment constraints. We conjecture that this convergence occurs at rate $O(\sigma^{-2})$, and provide numerical verification.

MGD also enables estimation of the entropy of the resulting distribution. We provide a tractable lower bound on the maximum entropy, computed directly from the MGD dynamics. We conjecture and numerically validate that this lower bound converges at rate $O(\sigma^{-2})$. This allows us to calculate the negentropy, which measures the non-Gaussianity of a random process as the difference between the entropy of a Gaussian with the same covariance and the entropy of the process~\cite{schrodinger1944life,hyvarinen2005estimation}. Prior to this work, numerical computation of this information-theoretic measure was prohibitively expensive for high-dimensional processes characterized by non-convex energies.

The MGD SDE is a nonlinear (McKean-Vlasov) equation whose drift depends on moments of its own solution. These moments are estimated empirically using interacting particles, and the dynamics is discretized in time. The computational cost scales as $O(\sigma^2)$, with a constant independent of both the data dimension and the non-convexity of the Gibbs energy.

MGD is related to microcanonical sampling algorithms~\cite{bruna2019multiscale}, which also generate samples in high dimension without estimating Lagrange parameters. However, the two methods differ in important ways. Microcanonical algorithms transport a Gaussian distribution toward a distribution satisfying the moment constraints using a gradient descent on the moment mismatch, which requires infinite time. Despite good numerical results in high dimension~\cite{Morel2022ScaleDA,brochard2022generalized,Cheng2023ScatteringSM}, they are not guaranteed to converge to the maximum entropy distribution, nor can they estimate the maximum entropy value. MGD, by contrast, achieves finite-time transport along a homotopic path and provides a tractable entropy estimator.

We apply MGD to high-dimensional multiscale stochastic processes, generating financial time-series and physical fields from maximum entropy models conditioned by wavelet scattering moments~\cite{bruna2019multiscale,Cheng2023ScatteringSM}. MGD produces accurate models of complex non-Gaussian processes with long-range correlations, including financial time series (S\&P 500), turbulent flows~\cite{villaescusa2020quijote}, and cosmological fields~\cite{schneider2006coherent}. For these fields, we provide the first estimates of negentropy.

\begin{table*}
\centering
\caption{Comparison of sampling approaches for complex distributions.}
\label{table:comparison}
\begin{tabular}{|l|c|c|c|c|}
\hline
\textbf{Approach} & \textbf{Input} & \textbf{Max-ent guarantee} & \textbf{Moment control} & \textbf{Sampling} \\
\hline
Maximum entropy (classical) & Moments $m$ & \checkmark & \checkmark & Equilibrium (MCMC) \\
Diffusion models & Dataset $(x^i)$ & \texttimes & \texttimes & Non-equilibrium \\
Moment Guided Diffusion & Dataset $(x^i)$ & \checkmark & \checkmark & Non-equilibrium (guided) \\
\hline
\end{tabular}
\end{table*}

The remainder of this paper is organized as follows. 
Section~\ref{sec:maxent} reviews classical maximum entropy sampling via MCMC and Langevin dynamics, as well as modern generative models based on diffusion and stochastic interpolants. 
Section~\ref{sec:m-flow} introduces the MGD transport and its numerical implementation. 
Section~\ref{sec-entropy} presents the entropy estimator, discusses the convergence of MGD as the volatility increases, and states our conjectures on the convergence rate. 
Section~\ref{sec:convergence_numerics} provides numerical verification of these conjectures. 
Section~\ref{sec:scat} applies MGD to high-dimensional multiscale processes---financial time series, turbulent flows, and cosmological fields---using wavelet scattering moments, and estimates their negentropy.
Technical proofs and additional details are provided in Appendix.

\section{Background: Classical Maximum Entropy and Modern Generative Modeling}
\label{sec:maxent}

We review the classical sampling approach of maximum entropy distributions with Langevin dynamics (Section \ref{sec:maxlikelihood}) and modern generative modeling based on transport via flow matching and stochastic interpolants (Section \ref{sec:stoch-interpolants}).

\subsection{Maximum Entropy Estimation via Langevin Dynamics}
\label{sec:maxlikelihood}

Given a moment function $\phi: \mathbb{R}^d \to \mathbb{R}^r$ with target expectation $m  \in \mathbb{R}^r$, the \emph{maximum entropy principle} seeks the probability density function (PDF)  $p$ which satisfies
the moment constraints
\begin{equation}
\label{eq:mom:def}
\mathbb{E}_{p}[\phi] = \int \phi(x) p(x)\, dx = m,
\end{equation}
while maximizing the differential entropy
\begin{equation}
\label{eq:ent:def}
H(p) = -\int p(x) \log p(x)\, dx.
\end{equation}
Since infinitely many densities satisfy the moment constraints, entropy maximization acts as a concave regularization that selects a unique solution. Introducing Lagrange multipliers $\theta \in \mathbb{R}^r$ for these constraints, the Lagrangian 
\begin{equation}
\label{eq:lagrangian}
\mathcal{L}(p, \theta) = H(p) - \theta^\top \big(\mathbb{E}_{p}[\phi] - m\big)
\end{equation}
has, if a maximizer exists, a unique maximum at $(p_*, \theta_*)$, where the maximum entropy density $p_* = p_{\theta_*}$ takes the exponential form: 
\begin{equation}
\label{eq:maxent}
p_\theta(x) = {\cal Z}_\theta^{-1}\,e^{-\theta^\top  \phi(x)}, \quad \text{with} \ {\cal Z}_\theta = \int_{\mathbb{R}^d} e^{-\theta^\top  \phi(x)}\, dx.
\end{equation}

The optimal parameter $\theta_*$ equivalently maximizes $\mathcal{L}(p_\theta,\theta) = - \theta^\top  m - \log {\cal Z}_\theta$. While direct evaluation is intractable because it requires computing the normalization constant ${\cal Z}_\theta$, the gradient can be estimated by sampling from $p_\theta$, since 
\begin{equation}
\label{eq:gradient}
\nabla_\theta\mathcal{L}(p_\theta,\theta) = \mathbb{E}_{p_\theta}\big[\phi\big]-m,
\end{equation}
because $\nabla_\theta \log {\cal Z}_\theta = -\mathbb{E}_{p_\theta}\big[\phi\big]$.

Sampling from~$p_\theta$ is typically performed using MCMC methods \cite{robert1999monte} based e.g.\ on Langevin dynamics, i.e.\ via solution of the SDE
\begin{equation}
\label{eq:langevin}
dX_t = -\sigma^2 \theta^\top  \nabla \phi(X_t)\,dt + \sqrt{2} \sigma\,dW_t,
\end{equation}
where $W_t$ is a standard Brownian motion, $\sigma$ is a volatility parameter, and $\nabla$ denotes the gradient with respect to $x\in\mathbb{R}^d$. Under suitable conditions, the law of $X_t$ converges to the distribution with density $p_\theta$ as $t \to \infty$ and, by ergodicity, $\mathbb{E}_{p_\theta}[\phi]$ can be estimated by a time average along the trajectory. In practice, the SDE~\eqref{eq:langevin} is discretized using an Euler-Maruyama scheme \cite{kloeden1977numerical,higham2001algorithmic}, and a Metropolis-Hastings accept-reject step is added to correct for discretization bias---this is the Metropolis Adjusted Langevin Algorithm (MALA) \cite{besag1994comments}.

Unfortunately, Langevin dynamics and MCMC algorithms more generally suffer from critical slowing down for non-convex energies, leading to prohibitively long equilibration times. In particular, MALA scales poorly with dimension \cite{chewi2025analysis,li2022sqrtd}, with sampling time growing exponentially in most cases.

This is particularly problematic for parameter estimation, since sampling must be repeated at each iteration of the optimization over $\theta$ to update $\mathbb{E}_{p_\theta}[\phi]$ as $\theta$ changes. The computational cost of both parameter estimation and sample generation typically becomes impractical for high-dimensional distributions.

When samples $(x^i)_{i\leq n}$ of $p$ are available, score matching~\cite{hyvarinen2005estimation} offers an alternative approach to the estimation of $\theta_*$. It avoids sampling $p_\theta$ by minimizing the Fisher divergence $\mathcal{I}(p, p_\theta) = \mathbb{E}_p\big[|\nabla \log p_\theta - \nabla \log p|^2\big]$. After integration by parts, the Fisher divergence can be written, up to a constant, as an expectation over the data: $\mathcal{I}(p, p_\theta) = \mathbb{E}_p\big[|\nabla \log p_\theta|^2 + 2 \Delta \log p_\theta\big] + \rm{cst}$, where $\Delta$ denotes the Laplacian. The resulting score matching parameter~$\tilde \theta_*$ that minimizes the Fisher divergence is a solution of the linear system~\cite{hyvarinen2005estimation} 
\begin{equation}
\label{eq:scorematching}
\mathbb{E}_p\big[\nabla \phi  \cdot\nabla \phi^{\top} \big] \tilde \theta_* =\mathbb{E}_p\big[\Delta \phi\big].
\end{equation}

{
The expectations in this equation can be estimated empirically using data
samples $(x^i)_{i\leq n}$, without sampling intermediate distributions. It is
therefore a much faster algorithm but, if the Gibbs energy of $p$ is non-convex,
this estimator has high variance~\cite{koehler2022statistical}, making it
unreliable. Furthermore, $\tilde\theta_* = \theta_*$ only if the data distribution already
belongs to the exponential family, that is $p = p_{\theta_*}$.

This condition is usually not satisfied. The density $p_*$ is then a model of
the data distribution $p$, not $p$ itself, and the discrepancy
$D_{\rm KL}(p\,\|\,p_*)$ measures the misspecification inherent to the choice of
moment function $\phi$. For maximum entropy models, the quality of the estimation of $p$ from data is
limited by two distinct factors: the finite amount of data available, which
limits how accurately $p_*$ itself can be estimated, and the modeling error
$D_{\rm KL}(p\,\|\,p_*)$, which persists no matter how much data is used.}

\subsection{{Stochastic Interpolants}}
\label{sec:stoch-interpolants}

Since the seminal work of Ho, Song and collaborators~\cite{ho2020denoising,song2021scorebased}, complex data generation has been addressed by transporting samples between Gaussian white noise and a target distribution $p$, through reversal of a stochastic noising process. Transport-based generative models have since been developed under various names---flow matching~\cite{lipman2022flow}, stochastic interpolants~\cite{albergo2022stochastic}, and rectified flows~\cite{liu2022flow}. These methods define time-dependent interpolations between two distributions and sample from them using flows (ODEs) or diffusions (SDEs). We adopt the stochastic interpolant formulation in what follows.

A variance preserving stochastic interpolant $I_t$ between samples $Z$ from a prior distribution (typically Gaussian noise, $Z\sim\mathcal{N}(0,\mathrm{Id})$) and data $X \sim p$ is defined by
\begin{equation}
\label{eq:stochinterpolant}
  I_t = \cos(\alpha_{t}) \,Z + \sin(\alpha_t)\,X, \quad t\in [0,1],
\end{equation} 
where $\alpha_t$ is a $\mathcal{C}^1([0,1])$ function with boundary conditions $\alpha_0 = 0$ and $\alpha_1 = \frac\pi2$ (for example $\alpha_t = \frac\pi2 t$), so that $I_0 = Z$ and $I_1 = X$. The key observation made in~\cite{albergo2022stochastic,albergo2023stochastic} is that the PDF $p_t(x)$ of the interpolant $I_t$ can be sampled via an SDE whose coefficients are estimable from data. Specifically, let $X_t$ satisfy the SDE
\begin{equation}
\label{eq:stochinterpolant_sde}
dX_t = b_t(X_t)\,dt  + \sigma^2 \nabla \log p_t(X_t)\, dt + \sqrt{2}\, \sigma\, dW_t,
\end{equation}
where $W_t$ is a Brownian noise and  $\sigma \geq 0$ is a tunable volatility with
\begin{equation}
\label{eq:driftsde}
    b_t (x) = \mathbb{E}\big[\dot I_t | I_t = x\big],
\end{equation}
where the dot denotes the time derivative and $\mathbb{E}\big[\cdot| I_t = x\big]$ the expectation over the law of $I_t$ conditional on $I_t=x$. Then, if $X_0=I_0=Z$,  $X_t$ and $I_t$ share the same PDF $p_t$ for all $t\in[0,1]$. By Stein's formula, the score $\nabla \log p_t$ can also be expressed as a conditional expectation:
\begin{equation}
\label{eq:score}
   \nabla \log p_t(x) =  -\frac{1}{\cos(\alpha_t)}  \mathbb{E}\big[Z | I_t = x\big].
\end{equation}
Since a conditional expectation is the minimizer of a quadratic loss, $b_t$ and $\nabla \log p_t$ can be learned by minimizing this loss, typically by representing them in a rich parametric class such as a deep neural network.

Unlike the Langevin SDE~\eqref{eq:langevin}, which follows equilibrium dynamics whose law converges to $p_\theta$ only as $t \to \infty$, the SDE~\eqref{eq:stochinterpolant_sde} defines a non-equilibrium transport that reaches the target distribution at time $t=1$. Crucially, this transport avoids the critical slowing down that plagues Langevin dynamics. Because the interpolant $I_t$ mixes data with Gaussian noise, the distribution $p_t$ varies smoothly from a simple Gaussian at $t=0$ to the target at $t=1$. Particles following the SDE~\eqref{eq:stochinterpolant_sde} are guided along this smooth path, with the complex structure of the target emerging gradually as $t \to 1$. For example, in multimodal distributions, particles are positioned inside the correct modes early (when the landscape is smooth) and remain there as the modes sharpen. We illustrate this in Figure~\ref{fig:diffusion_example}.

Stochastic interpolants thus provide a fast sampler that approximates the data distribution. In theory, the drift $b_t$ reproduces the full density $p_t$ of $I_t$ at each time, and hence the target density $p$ at $t=1$. With sufficient training data, deep neural networks generalize well on complex datasets~\cite{kadkhodaie2024generalization}. It results from an implicit regularization produced by the  stochastic gradient descent of the neural network optimization~\cite{bonnaire2025why,favero2025biggerisntmemorizingearly}, which is not well understood. In the low-data regime, however, the learned model may overfit and memorize the training samples. Maximum entropy models offer a complementary approach: they provide explicit regularization through entropy maximization, leading to analytic exponential distributions with controlled approximation error. The next section shows that they can also be sampled via stochastic interpolation.

\begin{figure}[H]
    \centering
    \includegraphics[width=1\linewidth]{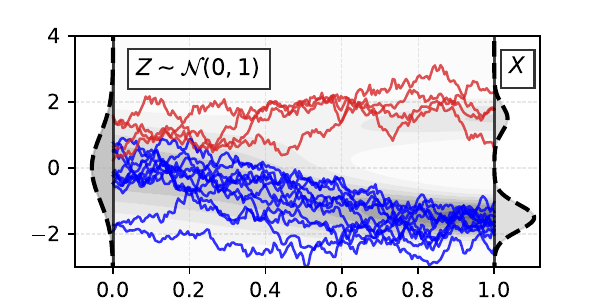}
    \caption{Illustration of trajectories ( in blue or red) of $X_t$ satisfying Equation (\ref{eq:stochinterpolant_sde}) for an interpolant $I_t$ defined with $\alpha_t = \pi t/2$ between white noise $Z$ and a bimodal unbalanced Gaussian mixture $X$, for $\sigma=1$. We display in gray in the background the density of $I_t$. When $t$ goes to $0$, the modes progressively disappear. At early times $t$, particles evolve freely in space, but they become trapped in the modes when the density $p_t$ becomes bimodal. Red particles are confined in the upper mode and blues in the lower one.}
    \label{fig:diffusion_example}
\end{figure}

\section{Moment Guided Diffusion}
\label{sec:m-flow}

In this section, we introduce Moment Guided Diffusion (MGD), which guides moments exactly along an interpolation path while injecting Langevin noise. We show that this preserves moments at each time; convergence to the maximum entropy distribution as the volatility increases will be discussed in Section~\ref{sec-entropy}. Section~\ref{sec:m-flow_diffusion} defines the MGD SDE {and discusses the existence of its solutions. Section~\ref{sec:variational} gives a variational characterization of MGD from a
maximum entropy principle on path space.} 
Section~\ref{sec:numalg} introduces a discretized algorithm and discusses its numerical cost.

\subsection{{The MGD SDE}}
\label{sec:m-flow_diffusion}
A stochastic interpolant SDE~\eqref{eq:stochinterpolant_sde} produces $X_t$ with the same distribution as $I_t$, thereby reproducing all moments. MGD uses the same interpolant $I_t$, but imposes only that a finite number of moments is preserved:
\begin{equation}
\label{eq:th-moments}
\forall t\in[0,1]: \quad \mathbb{E}\big[\phi(X_t)\big] = \mathbb{E}\big[\phi(I_t)\big] \underset{\mathrm{def}}{:=} m_t.
\end{equation}
{ There exist infinitely many SDEs whose solutions satisfy this moment constraint. In this paper, we adapt the homotopic SDE~\eqref{eq:driftsde} in order to sample from maximum entropy distributions. This amounts to approximating $p_t$ by an exponential family model, which suggests the approximation $\nabla\log p_t \approx -\theta_t^\top \nabla\phi$, and consequently $b_t \approx \eta_t^\top \nabla\phi$. Such a choice of drift bridges SDE~\eqref{eq:driftsde} and the Langevin SDE, in which the drift is replaced by an analogous time-dependent term. MGD is defined by the SDE~\eqref{eq:sde} of the following theorem, which has a
drift linear in $\nabla\phi$ and satisfies the moment
constraint~\eqref{eq:th-moments}. A variational characterization of MGD is later presented in Proposition~\ref{prop:minimal-rate}.}  



\begin{restatable}[Moment Guided Diffusion]{theorem}{propmoments}
\label{prop:moments}

{ Assume that $\phi$ is twice differentiable and that $t \mapsto m_t$ is differentiable.} Consider the SDE
\begin{equation}
    \label{eq:sde}
    dX_t  =  \big(\eta_t^\top-\sigma^2 \theta_t^\top\big) \nabla \phi(X_t)\, dt  + \sqrt {2}\, \sigma\, dW_t,
\end{equation}
with $X_0 = Z$, where $W_t$ is a Brownian noise, {$\sigma\geq0$}, and $\eta_t$ and $\theta_t$ solve
\begin{align}
\label{eq:eta}
G_t\, \eta_{t} &= \frac{d}{dt}m_t,
\\
\label{eq:theta}
G_t\, \theta_{t} &= \mathbb{E}\big[\Delta \phi(X_t)\big],
\end{align}
where $G_t$ is the Gram matrix
\begin{equation}
\label{eq:gram}
    G_t =\mathbb{E}\big[\nabla \phi(X_t) \cdot \nabla \phi(X_t)^\top \big].
\end{equation}
{ Assume that this system admits a solution on $[0,1]$. Then $X_t$ satisfies the moment constraint~\eqref{eq:th-moments}.}
\end{restatable}

The proof of Theorem~\ref{prop:moments} is given in Appendix~\ref{app:proof}:
{It\^o's lemma shows that seeking a solution
of~\eqref{eq:sde} satisfying the moment constraint~\eqref{eq:th-moments} leads
to the system~\eqref{eq:eta}--\eqref{eq:theta}, and the converse holds:
if $\eta_t$ and $\theta_t$ solve this system then
$\frac{d}{dt}\mathbb{E}[\phi(X_t)] = \frac{d}{dt}m_t$, and~\eqref{eq:th-moments}
follows since $X_0 = I_0 = Z$ gives $\mathbb{E}[\phi(X_0)] = m_0$.}


Note also that the existence of a solution needs to be assumed. Since $G_t$, $\eta_t$, and $\theta_t$ depend on the law of $X_t$, \eqref{eq:sde} is a nonlinear (McKean-Vlasov) SDE~\cite{mckean1966class,chaintron2022propagation} whose well-posedness is non-trivial. This is further discussed in Remark \ref{rk:existence} below. In~Section~\ref{sec:gauss_converg}, we show that for the quadratic feature map $\phi(x) = (x_i x_j)_{1 \leq i \leq j \leq d}$,
where $x_i$ denotes the $i$-th coordinate of $x \in \mathbb{R}^d$, MGD
admits an analytical solution.

{ 
\begin{remark}[Existence of solutions]
\label{rk:existence}

For general $\phi$, sufficient conditions for existence are established in
Appendix~\ref{app:proofs}, for a regularized
dynamics and under strong regularity assumptions on $\phi$. They concern the
case of fixed moments $m_t$, that is $\eta_t = 0$, in which the dynamics
increases the entropy of the distribution while preserving its moments:
Theorem~\ref{th:convergence} shows that, for bounded $\mathcal{C}^4$ moment
functions and with an additional confining potential, this version of MGD admits
strong solutions $X_t$, converging, in the limit of large $\sigma$, to the maximum entropy distribution with
moments $m_t$.

These assumptions are far stronger than what the numerical results require for
the original MGD SDE~\eqref{eq:sde}. The moment functions used in this paper are
only twice differentiable almost everywhere, and
Sections~\ref{sec:convergence_numerics} and~\ref{sec:scat} report existence and
convergence for all of them. This includes the $\ell_1$ norms that enforce sparsity (see  Section~\ref{subsub_nonsmooth})
and the scattering covariances, whose derivatives are discontinuous.

\end{remark} }

The drift in~\eqref{eq:sde} has two components, mirroring those
of~\eqref{eq:stochinterpolant_sde}. The transport term $\eta_t$ plays the role
of $b_t$, steering the process so that the moments move from $m_t$ to
$m_{t+dt}$, while the term $\sigma^2\theta_t$ plays the role of
$\nabla\log p_t$ and counterbalances the moment modification induced by the
white noise $\sigma\,dW_t$. { The second correspondence holds only approximately. Writing $p_t^\sigma$ for
the density of $X_t$, we see from~\eqref{eq:theta} that
$-\theta_t^\top\nabla\phi$ is the score matching approximation of
$\nabla\log p_t^\sigma$, and the two differ unless $X_t$ is of maximum entropy.
Because the drift is confined to the family spanned by $\nabla\phi$, MGD matches
the moments $\mathbb{E}[\phi(I_t)]$ but does not reproduce the law of $I_t$, as
the stochastic interpolant SDE does.
}

{\begin{remark}[Numerical tractability] Confining the drift to the family spanned by $\nabla\phi$ makes the MGD SDE
numerically tractable. It reduces the computation of the drift to the two
$r\times r$ linear systems~\eqref{eq:eta} and~\eqref{eq:theta}, whose
coefficients are empirical averages over a particle ensemble. The number of
parameters to be estimated is thus twice the number of constraints, irrespective
of the dimension $d$, and no neural network needs to be trained. MGD can
therefore be estimated from limited datasets, precisely the regime in which
learning the full score is unreliable. Section~\ref{sec:numalg} details the
resulting algorithm and its cost.
\end{remark}}

{ The volatility $\sigma$ controls the convergence to maximum entropy: the
Brownian noise increases entropy, whereas the guidance $\eta_t$ may reduce it,
so that the noise dominates when $\sigma$ is large. At $\sigma = 0$ and for
general $\phi$, the law $p_1^\sigma$ of $X_1$ is therefore not the maximum
entropy distribution with moments $\mathbb{E}[\phi(X)]$. The quadratic feature
map is an exception: if $p_0$ is Gaussian, then $X_t$ does not depend on
$\sigma$, so one may set $\sigma = 0$ and reduce~\eqref{eq:sde} to an ODE that
already samples $p_*$. For general $\phi$, we conjecture
(Section~\ref{sec-entropy}) and verify numerically
(Section~\ref{sec:convergence_numerics}) that $p_1^\sigma \to p_*$ as
$\sigma \to \infty$. MGD thus eliminates the sampling error
$D_{\mathrm{KL}}(p_1^\sigma \,\|\, p_*)$, whereas the modeling error
$D_{\mathrm{KL}}(p \,\|\, p_*)$ of Section~\ref{sec:maxlikelihood} remains a
choice dictated by $\phi$. Throughout this paper, ``exact sampling'' refers to
sampling from $p_*$, not from $p$.}

\begin{remark}
\label{remark:theta}
Observe that $\theta_1$ in~\eqref{eq:theta} coincides with the score matching parameter in~\eqref{eq:scorematching} computed for $X_1$. If the distribution of $X_1$ converges to $p_*$ as $\sigma \to \infty$, then $\theta_1$ converges to $\theta_*$. A finite sample estimator of $\theta_1$ is thus a nearly consistent estimator of $\theta_*$ for large $\sigma$. However, as noted in Section~\ref{sec:maxent}, score matching estimators have high variance for non-convex energies. 
Crucially, MGD's sampling accuracy depends on the empirical estimation of $m_t$, not on the accuracy of $\theta_t$ (see Section~\ref{sec:numerics-cost}.)

\end{remark}

{\subsection{Variational Characterization of MGD}
\label{sec:variational}
This section gives a variational characterization of MGD. Many SDEs satisfy the
moment constraint~\eqref{eq:th-moments}, which therefore leaves the drift
undetermined. We select one with a maximum entropy principle on path space,
applied incrementally. The reference is the free diffusion at volatility
$\sigma$, which carries no drift and hence maximal entropy. At each time, we
move the moments to their prescribed value while staying as close as possible to
it. This selects MGD, whose drift minimizes the rate of relative entropy
production with respect to free diffusion.

Let $\mathbb{P}$ denote the law of the paths of the MGD SDE~\eqref{eq:sde} and
$\mathbb{P}_{\rm ref}$ that of the reference free diffusion
\begin{equation}
\label{eq:refprocess}
    dX^{\rm ref}_t = \sqrt{2}\,\sigma\, dW_t, \qquad X^{\rm ref}_0 = Z .
\end{equation}
By Girsanov's theorem, the relative entropy of $\mathbb{P}$ with respect to
$\mathbb{P}_{\rm ref}$ is the accumulated squared drift,
\begin{equation}
\label{eq:girsanov}
    D_{\rm KL}\big(\mathbb{P} \,\|\, \mathbb{P}_{\rm ref}\big)
    = \tfrac{1}{4\sigma^2}\int_0^1
    \mathbb{E}\big[\big|(\eta_t-\sigma^2\theta_t)^\top\nabla\phi(X_t)\big|^2\big]
    \,dt~,
\end{equation}
so that the integrand $\tfrac{1}{4\sigma^2}
    \mathbb{E}\big[\big|(\eta_t-\sigma^2\theta_t)^\top\nabla\phi(X_t)\big|^2\big]$ is the rate at which MGD produces
relative entropy with respect to free diffusion.

We show that this rate is the smallest possible, at each time. Suppose the
process has followed~\eqref{eq:sde} up to time $t$, so that the law of $X_t$ is
given, and consider replacing the drift at that instant by an arbitrary $u_t$,
so that at time $t$
\begin{equation}
\label{eq:uprocess}
    dX_t = u_t(X_t)\,dt + \sqrt{2}\,\sigma\,dW_t .
\end{equation}
Such a drift is admissible if it carries the moments to their prescribed value
at $t+dt$, that is if $\tfrac{d}{dt}\mathbb{E}[\phi(X_t)] = \tfrac{d}{dt}m_t$
under~\eqref{eq:uprocess}. It then produces relative entropy with respect to
free diffusion at the rate $\mathbb{E}\big[|u_t(X_t)|^2\big]/4\sigma^2$, which is minimal for MGD's drift:

\begin{restatable}[Minimal entropy production rate]{proposition}{propminimalrate}
\label{prop:minimal-rate}
Fix $t \in [0,1]$. Under the hypotheses of Theorem~\ref{prop:moments}, let $X_s$
solve the MGD SDE~\eqref{eq:sde} for $s \leq t$, and assume that the Gram
matrix $G_t$ in~\eqref{eq:gram} is finite and invertible. Among the drifts $u_t$
in~\eqref{eq:uprocess} such that
$\tfrac{d}{dt}\mathbb{E}[\phi(X_t)] = \tfrac{d}{dt}m_t$, the quantity
$\mathbb{E}\big[|u_t(X_t)|^2\big]$ admits a unique minimizer,
\begin{equation}
    \label{eq:mgd_drift}
    u_t = \big(\eta_t^\top - \sigma^2\theta_t^\top\big)\nabla\phi ,
\end{equation}
which is the drift of~\eqref{eq:sde} itself.
\end{restatable}

The proof is given in Appendix~\ref{app:proof}. 
It\^o's formula turns the constraint $\tfrac{d}{dt}\mathbb{E}[\phi(X_t)] = \tfrac{d}{dt}m_t$ into a linear condition on $u_t$,
\begin{equation*}
\label{eq:instant-constraint}
    \mathbb{E}\big[u_t(X_t)\cdot\nabla\phi(X_t)\big]
    + \sigma^2\,\mathbb{E}\big[\Delta\phi(X_t)\big] = \tfrac{d}{dt} m_t .
\end{equation*}
The minimization problem of Proposition~\ref{prop:minimal-rate} is thus
quadratic under a linear constraint, and its solution is an orthogonal
projection. Writing $p_t^\sigma$ for the density of $X_t$, the minimizer is the
projection in $L^2(p_t^\sigma)$ onto the span of the $\nabla\phi_i$. The $r$
constraints require exactly these $r$ directions, and any component outside the
span would increase the rate without helping to satisfy them.

Applying this at every $t \in [0,1]$, starting from $X_0 = Z$, yields MGD:
minimizing the entropy production rate while carrying the moments along $m_t$
produces the SDE~\eqref{eq:sde}, whose linear drift is therefore not a modeling
choice but a consequence of the minimization. This is the maximum entropy principle of Section~\ref{sec:maxlikelihood}
transposed to path space. There, the density $p_*$ maximizes the differential
entropy~\eqref{eq:ent:def} under the moment constraints~\eqref{eq:mom:def},
which amounts to minimizing the relative entropy to a uniform reference. Here
the reference is the free diffusion~\eqref{eq:refprocess}, and the criterion is
applied incrementally rather than once.

A different variational definition of MGD may be derived
from~\cite{mis2026measureflowpathrecovery}, where Bayes--Hilbert Galerkin theory
is used to recover a probability measure flow from finite information. In that
framework, MGD would appear as the projection of a maximum entropy flow observed
through the partial information carried by its moments, a viewpoint we find
promising.

\begin{remark}[Global variational characterization]
\label{rem:global-variational}
Proposition~\ref{prop:minimal-rate} minimizes the rate of entropy production at
each time. One could instead minimize the total relative entropy accumulated
over $[0,1]$, over the diffusions $X^u_t$ with drift $u_t$ and path law
$\mathbb{P}^u$:
\begin{equation}
\label{eq:global-problem}
    \min_{u}\ D_{\rm KL}\big(\mathbb{P}^u \,\|\, \mathbb{P}_{\rm ref}\big)
    ~ \text{s.t.} ~
    \mathbb{E}\big[\phi(X^u_t)\big] = m_t \ \ \text{on } [0,1].
\end{equation} This is a different problem, discussed in Appendix~\ref{app:global-system}.
The MGD drift at time $t$ is determined by the past alone, through the law of
$X_t$, whereas the globally optimal drift at time $t$ depends on the whole
trajectory over $[0,1]$. Its solution is a drift field on $\mathbb{R}^d$, with no reason to be linear in
$\nabla\phi$. When it is, for every moment trajectory, it coincides with the MGD drift, and Proposition~\ref{prop:closure} shows that $\mathrm{span}\{\phi,1\}$
must then be stable under the products $\nabla\phi_i\cdot\nabla\phi_j$ and
contain the $\Delta\phi_k$. For polynomial $\phi$ this holds only up to degree two
(Proposition~\ref{prop:degree}), so MGD solves the global problem exactly for
quadratic $\phi$ (Proposition~\ref{prop:quad-global}) and approximates it in
general.
\end{remark}

}
\subsection{Discretization of MGD}
\label{sec:numalg}

We solve numerically the MGD nonlinear (McKean-Vlasov) SDE~\eqref{eq:sde}  by iteratively updating a finite ensemble of interacting particles. To update the particles, we estimate $\eta_t$ in \eqref{eq:eta} and $\theta_t$ 
in \eqref{eq:theta} with empirical means over these particles. We also avoid computing $\mathbb{E}[\Delta\phi]$, which is costly or ill-defined when $\phi$ is not smooth. This is achieved with a two-step predictor-corrector scheme, which we first describe using exact expectations before discussing finite-particle estimations.

Given $X_t$ and some small time step $h>0$, let $Y_t$ be obtained via
\begin{equation}
\label{eq:Y:update}
Y_t = X_t + h \, \eta_t^\top \nabla \phi(X_t) + \sqrt{2} \sigma (W_{t+h}- W_t)~,
\end{equation}
with $\eta_t$ the solution to~\eqref{eq:eta}. This is the Euler-Maruyama scheme for the MGD SDE~\eqref{eq:sde} with $\theta_t = 0.$ As such, the update \eqref{eq:Y:update} does not preserve the moments, i.e. $\mathbb{E} [ \phi(Y_t)] \not =  m_{t+h}$. We enforce this moment condition by adding to $Y_t$ a term similar to the one involving $\theta_t$ in the MGD SDE~\eqref{eq:sde}, i.e.\ setting
\begin{equation}
\label{eq:XY:update}
X_{t+h} = Y_t - h \, \sigma^2 \hat \theta^\top  \nabla \phi(Y_t)~,
\end{equation}
and requiring that 
\begin{equation}
\label{eq:XY:update2}
\mathbb{E}\big[\phi(X_{t+h})\big] = m_{t+h}.
\end{equation}
Substituting~\eqref{eq:XY:update} into~\eqref{eq:XY:update2} gives an equation for $\hat \theta$. Solving this exactly is costly (it is nonlinear in $\hat\theta$) and unnecessary, since the Euler-Maruyama update is only accurate to weak order 1 in $h$. Working to the same order of accuracy, we Taylor expand the left-hand side of~\eqref{eq:XY:update2} to obtain
\begin{equation}
\label{eq:theta:update}
 h \sigma^2  \, \mathbb{E}\big[ \nabla \phi(Y_t)\cdot \nabla \phi(Y_t)^\top\big]\, \hat \theta =  \mathbb{E}\big[\phi(Y_t)\big] - m_{t+h}~.
 \end{equation}
{ Since $\mathbb{E}\big[\phi(X_{t})\big] = m_{t}$, the right-hand side equals
$h\sigma^2\, \mathbb{E}\big[\Delta \phi(X_t)\big] + o(h)$. Dividing \eqref{eq:theta:update} by $h\sigma^2$
leads to \eqref{eq:theta} up to a factor $o(1)$, showing that
$\hat\theta = \theta_t + o(1)$.}

In the numerical scheme, however, it is more convenient to solve~\eqref{eq:theta:update} directly since this avoids computing $\mathbb{E}[\Delta \phi(X_t)]$. This is important when $\phi$ includes $\ell^1$ norms or absolute values, for which $\Delta\phi$ is a sum of Dirac functions whose expectation is hard to estimate unless the number of samples is very large.

To turn this into a practical scheme, we need to choose the time step $h$. Since the drift is proportional to $\sigma^2$ for large $\sigma$, so is its Lipschitz constant. We therefore set the number of steps to $n_\sigma = a\sigma^2+b$, where $b$ ensures the limiting ODE ($\sigma \to 0$) is accurately solved. The computational cost of MGD thus scales as $O(\sigma^2)$.

\begin{algorithm}[H]
\caption{Moment-Guided Diffusion (MGD)}\label{alg:mgd}
\begin{algorithmic}
\STATE \textbf{Input:} volatility $\sigma$; number of steps $n_\sigma = O(\sigma^{2})$; time step $h=1/n_\sigma$; number of replicas $n_{\rm{rep}}$; moments  $m_t = \mathbb{E}[\phi(I_t)]$
\STATE \textbf{Initialize:} $x_0^i \sim \mathcal{N}(0,\mathrm{Id})$ for $i = 1, \ldots, n_{\rm{rep}}$
\STATE \textit{Predictor}
\FOR{$k = 0, \ldots, n_\sigma-1$}
    \STATE Compute $\hat{G}_k = \frac{1}{n_{\rm{rep}}} \sum_{i=1}^{n_{\rm{rep}}} \nabla \phi(x_k^i) \cdot\nabla \phi(x_k^i)^\top $
    \STATE Solve $\hat{G}_k \hat{\eta}_k = \frac{d}{dt}{m}_t|_{t=kh}$ for $\hat{\eta}_k$
    
    \FOR{$i = 1, \ldots, n_{\rm{rep}}$}
    \STATE Sample $\xi_k^i \sim \mathcal{N}(0, \mathrm{Id})$
    \STATE Set $y_k^i = x_k^i + h \, \hat{\eta}_k^\top  \nabla \phi(x_k^i) + \sqrt{2h}\sigma \, \xi_k^i$
    \ENDFOR
    \STATE \textit{Corrector (project to preserve moments)}
    \STATE Compute $\hat{G}'_k = \frac{1}{n_{\rm{rep}}} \sum_{i=1}^{n_{\rm{rep}}}\nabla \phi(y_k^i) \cdot\nabla \phi(y_k^i)^\top $
    \STATE Solve $h\sigma^2 \hat{G}'_k \hat{\theta}_k = \frac{1}{n_{\rm{rep}}} \sum_{i=1}^{n_{\rm{rep}}} \phi(y_k^i) - m_{(k+1)h}$ for $\hat{\theta}_k$
    \FOR{$i = 1, \ldots, n_{\rm{rep}}$}
    \STATE Set $x_{k+1}^i = y_k^i + h \, \hat{\theta}_k^\top  \nabla \phi(y_k^i)$
    \ENDFOR
\ENDFOR
\STATE \textbf{Output:} Samples $(x_{n_\sigma}^i)_{1\leq i\leq n_{\rm{rep}}}$
\end{algorithmic}
\end{algorithm}

The scheme is summarized in Algorithm~\ref{alg:mgd}. It iteratively evolves a population of $n_{\rm{rep}}$ particles $(x_{k}^i)_{1 \leq i \leq n_{\rm{rep}}}$ (replicas), whose empirical measure approximates the distribution of $X_{k/n_\sigma}$, using moments $m_t$ estimated from training data. A key property is that the moment condition~\eqref{eq:XY:update2} remains valid when expectations are replaced by empirical averages over particles, since~\eqref{eq:theta:update} holds for empirical distributions. As a result, the empirical mean of $\phi$ over the particles converges to $m_{t+h}$ as the step size $h \to 0$. This exact moment tracking controls the dynamical stability of the algorithm: a divergence of particles would manifest as a moment mismatch (see Remark~\ref{remark:theta}).
Alternative implementations are discussed in Appendix~\ref{app:alt_implementations} and some numerical details in Appendix~\ref{app:exp}.

\section{Maximum Entropy: Convergence and Bounds}
\label{sec-entropy}

{
Let $p_t^\sigma$ be the PDF of the solution $X_t$ of the MGD SDE~\eqref{eq:sde} for a volatility $\sigma$. Section~\ref{sec:gauss_converg} proves the existence of solutions of MGD and the convergence towards the maximum entropy distribution $p_{*}$ in the specific case of quadratic moment function $\phi$. The convergence towards the maximum entropy distribution for generic moment function $\phi$ is then addressed in 
Section~\ref{sec:convergence_th}. Finally, Section~\ref{sec:convergence_entropy} computes a tractable lower bound of the entropy $H(p_1^\sigma)$ and conjectures its convergence  towards $H(p_*)$ when $\sigma$ increases.}

{
\subsection{Convergence for Quadratic Moment Function $\phi$}
\label{sec:gauss_converg}

Consider the quadratic feature map $\phi(x) = (x_i x_j)_{1 \leq i \leq j \leq d}$, where $x_i$ denotes the $i$-th coordinate of $x \in \mathbb{R}^d$. In this case, the maximum entropy distributions under the constraints $m_t = \mathbb{E}(\phi(I_t))$ are Gaussian, with the same covariance as $I_t$. For this particular choice of $\phi$, we show that MGD admits an analytical solution.

\begin{theorem}[Existence for quadratic $\phi$]
\label{th:gaussian}
Consider the stochastic interpolant $I_t=\cos(\alpha_t) Z + \sin(\alpha_t) X$, where $Z$ and $X$ are independent, have zero mean, and positive-definite covariance matrices $C_0$ and $C_1$. Then, let
\begin{equation}
    \label{eq:cov:int}
    C_t = \cos^2(\alpha_t) C_0 + \sin^2(\alpha_t) C_1
\end{equation}
be the covariance of the stochastic interpolant $I_t$. 

Then the MGD SDE~\eqref{eq:sde} associated with the quadratic function $\phi(x) = (x_ix_j)_{1\leq i \leq j\leq d}$ reads
\begin{equation}
    \label{eq:sde:quad}
    dX_t = \big(\tfrac12 \dot C_t C_t^{-1}- \sigma^2 C_t^{-1}\big) X_t dt + \sqrt{2}\sigma dW_t,
\end{equation}
with $X_0 = Z$ and where $\dot C_t = dC_t/dt$.

For each $\sigma$, this is an SDE whose unique solution exists. 
\end{theorem}
The proof is given in Appendix~\ref{proof:th-gaussian}. Note that it implies that here we have
\begin{align*}
    \eta_t^\top \nabla \phi(x) &= \tfrac12 \dot C_t C_t^{-1}x,\\
    \theta_t^\top \nabla \phi(x) &= C_t^{-1}x.
\end{align*}
Since the MGD SDE~\eqref{eq:sde:quad} is linear with additive noise, if $Z$ is Gaussian, $X_t$ is also Gaussian, with mean zero and covariance $\mathbb{E}[X_t X_t^\top] = C_t$ for any $\sigma\ge0$; indeed a direct calculation with It\^o's formula shows that
\begin{align*}
    \frac{d}{dt}\mathbb{E}[X_t X_t^\top] &= \big(\tfrac12 \dot C_t C_t^{-1}- \sigma^2 C_t^{-1}\big)\mathbb{E}[X_t X_t^\top]\\
    & + \mathbb{E}[X_t X_t^\top] \big(\tfrac12  C_t^{-1} \dot C_t- \sigma^2 C_t^{-1}\big)\\ & + 2\sigma^2 \text{Id},
\end{align*}
whose unique solution is $\mathbb{E}[X_t X_t^\top] = C_t$. In this case, for any $\sigma\geq 0$, $X_t$ exactly samples the maximum entropy distribution associated with moments $m_t$.

Interestingly, this result also holds if the base distribution is non-Gaussian, provided that we let $\sigma\to\infty$.

\begin{theorem}[Convergence for quadratic $\phi$]
\label{th:gaussian:2}
Given any base distribution with a positive-definite covariance $C_0$ that commutes with the covariance $C_1$ of the target distribution, the PDF  $p^\sigma_t(x)$ of the solution to the MGD SDE~\eqref{eq:sde:quad} satisfies 
\begin{equation}
    \label{eq:lim:kl:G}
    \lim_{\sigma\to\infty} D_{\text{KL}}(p_1^\sigma\|p_*) =0
\end{equation}
where $p_*$ is the PDF of the maximum entropy distribution associated with the quadratic $\phi$, i.e. the Gaussian distribution with mean zero and covariance $C_1 = \mathbb{E}[XX^\top]$.
\end{theorem}

The proof is given in Appendix~\ref{proof:th-gaussian}. Note that we make the assumption that $C_0C_1=C_1C_0$ so that these two matrices are co-diagonalizable; this facilitates the proof, but the theorem remains valid if this assumption is lifted.

}
\subsection{Convergence towards the Maximum Entropy Distribution}
\label{sec:convergence_th}

A central claim of this paper is that, as $\sigma \to \infty$, the distribution $p_1^\sigma$ of the MGD output converges to the maximum entropy distribution $p_*$.

Next, we provide heuristic support for this claim via a formal Taylor expansion, then state it precisely as Conjecture~\ref{conj:convergence}. The conjecture is verified numerically in Section~\ref{sec:convergence_numerics}.

{
In this section and the next, we assume that the MGD SDE~\eqref{eq:sde} admits
solutions. Whether it does depends on the moment function $\phi$, as discussed
in Remark~\ref{rk:existence}.
}

The time evolution of the PDF $p_t^\sigma$ of the solution of the MGD SDE~\eqref{eq:sde} is governed by the Fokker-Planck equation:
\begin{equation}
\label{eq:fpe}
     \partial_tp_{t}^{\sigma} =\nabla\cdot(p_{t}^{\sigma}((-\eta_{t}+\sigma^2\theta_{t})^\top \nabla\phi))+\sigma^2\Delta p_{t}^{\sigma}.
\end{equation}
Formally taking $\sigma \to \infty$ and keeping only the leading-order terms, the Fokker-Planck equation reduces to
\begin{equation}
\nabla\cdot(p_{t}^{*}({\theta_{t}^{*}}^\top \nabla\phi))+\Delta p_{t}^{*} = 0,
\end{equation}
where $p_t^*$ and $\theta_t^*$ denote the (formal) limits of $p_t^\sigma$ and $\theta_t$ as $\sigma \to \infty$.
The solution of this limit equation is an exponential distribution:
\begin{equation}
\label{eq:limitproba}
p_{t}^{*} = ({{\cal Z}_t^{*}})^{-1}  e^{-{\theta_{t}^{*}}^\top \phi},
\end{equation}
with normalizing constant ${{\cal Z}_t^{*}}$. This suggests that $p_{t}^{\sigma}$ converges to an exponential distribution with moments $m_t$, and hence to the maximum entropy distribution satisfying these constraints. In particular, this gives $p_{1}^* = p_* $ and $\theta_1^* = \theta_*$. Expanding $p_{t}^{\sigma} = p_t^*(1 + q_t\sigma^{-2} + o(\sigma^{-2}))$, for some $q_t$ that does not depend upon $\sigma$, Appendix~\ref{app:conjecture} provides a formal calculation showing that the Kullback-Leibler divergence satisfies $D_{\rm KL}(p_1^\sigma \| p_*) = O(\sigma^{-2})$. This leads to the following conjecture:

\begin{restatable}[Max entropy]{conjecture}{conv_ent}
\label{conj:convergence}
{ Assume that MGD admits solutions with density $p_t^\sigma$ for all $\sigma$ large enough, and that maximum entropy distributions with moments $m_t$ exists for all $t\in[0,1]$. }

Then, there exists $C > 0$ such that for all $\sigma > 0$
\begin{equation}
\label{eq:conj_entrop}
   D_{\rm KL}(p_1^\sigma \| p_{*})  \leq C\, \sigma^{-2} .
\end{equation}
\end{restatable}

A numerical verification of Conjecture~\ref{conj:convergence} is given in Section~\ref{sec:convergence_numerics}. { Since $p_1^\sigma$ and $p_*$ share the same moments $\mathbb{E}(\phi(X))$, and since $p_*(x)={\Z_*}^{-1}e^{-\theta_*^\top\phi(x)}$, we have }
\begin{equation}
D_{\rm KL}(p_1^\sigma \| p_{*})= H(p_*) - H(p_1^\sigma),
\end{equation}
so~\eqref{eq:conj_entrop} is equivalent to
\begin{equation}
\label{eq:conj_entrop_H}
  H(p_*) - H(p_1^\sigma)  \leq C\, \sigma^{-2}.
\end{equation}

Since the numerical cost of MGD is $O(\sigma^2)$ (see Section \ref{sec:numalg}), the cost required to reach a given error is proportional to the constant $C$ appearing in Conjecture \ref{conj:convergence}. This constant depends on the moment function $\phi$ and the moments $m_t$, and becomes large when $\phi$ is not expressive enough to capture the homotopic transport of mass at early times $t$---before the maximum entropy distribution with moments $m_t$ becomes multimodal.

If $x \in \R$, a truncated monomial basis $\phi(x) = (x^k)_{k\leq r}$ provides this flexibility, as illustrated in Section~\ref{sec:num-convergence}. 
If $x \in \R^d$, since the number of monomials grows polynomially with $d$, this strategy becomes computationally prohibitive for $d$ large. 
A wavelet scattering spectra $\phi$~\cite{Cheng2023ScatteringSM} computes 
$O(\log^3 d)$ low-order multiscale moments that are similar to fourth order moments.
In Section~\ref{sec:scat}, we show that for real-world high-dimensional datasets from physics and finance, it is sufficiently rich to capture this homotopic transport with a small $C$.

Modeling the transport of mass does not require $\phi$ to provide an accurate model of the full distribution of $I_t$. We show in Section~\ref{sec:cv_modelisation_error} that $C$ can be small although the model misspecification $D_{\mathrm{KL}}(p \| p_*)$ is large.

\subsection{Entropy Estimation}
\label{sec:convergence_entropy}

We now compute a tractable lower bound on the entropy $H(p_1^\sigma)$ and conjecture that it converges to $H(p_*)$ as $\sigma\to\infty$.
\begin{restatable}{proposition}{propentropy}
\label{prop:entropy}
Assume the MGD SDE~\eqref{eq:sde} admits a unique strong solution for all $t\in[0,1]$. Then, 
\begin{equation}
\label{eq:dHdt_equality}
\begin{split}
     \frac{d}{d t}H(p_t^\sigma) &= \theta_{t}^\top \frac{d}{dt} m_t   
     \\
& +\sigma^2\mathbb{E}\big[|\nabla\log p_t^{\sigma}(X_t) +\theta_t^\top \nabla\phi(X_t)|^2\big],
\end{split}
\end{equation}
and hence
\begin{equation}
\label{eq:dHdt_inequality}
     \frac{d}{d t}H(p_t^\sigma) \geq \theta_{t}^\top \frac{d}{dt} m_t. 
\end{equation}
\end{restatable}
The proof, given in Appendix~\ref{app:proof}, uses the Fokker-Planck equation~\eqref{eq:fpe} to compute the evolution of the differential entropy of $p^\sigma_t$. When the moments are constant ($\frac{d}{dt} m_t=0$), the entropy increases along the dynamics. In this case, $H(p_t^\sigma)$ also increases with $\sigma$, as shown by a time-rescaling argument in Proposition~\ref{proposition:dentropydD_app}. Sections~\ref{sec:convergence_numerics} and~\ref{sec:scat} provide numerical verification that $\frac{d}{d\sigma}H(p_t^{\sigma}) \geq 0$ more generally.

Integrating~\eqref{eq:dHdt_inequality} over $[0,1]$ yields a lower bound on the entropy of the sampled distribution $p_1^\sigma$:
\begin{equation}
\label{eq:entropy}
    H(p_1^\sigma) \geq H^{\sigma}_* \underset{\rm def}{:=}   H(p_0) +\int_{0}^1\theta_{t}^\top \frac{d}{dt}m_t\, dt.
\end{equation}
This lower bound can be computed directly from the MGD parameters along the dynamics. From~\eqref{eq:dHdt_equality}, the gap between $H(p_1^\sigma)$ and its lower bound is
\begin{equation}
\label{eq:entropy_gap}
H(p_1^\sigma) - H^{\sigma}_* =  \sigma^2 
\int_0^1 \mathbb{E}\big[|\nabla\log p_t^{\sigma}(X_t) +\theta_t^\top \nabla\phi(X_t)|^2\big]\, dt.
\end{equation}
This integral is the time-averaged Fisher divergence between $p_t^\sigma$ and the exponential distribution with energy $\theta_t^\top \phi$. If Conjecture~\ref{conj:convergence} holds, this Fisher divergence vanishes as $\sigma \to \infty$, provided that $p_0$ itself has an exponential form. In particular, this holds when $X_0 = Z$ is Gaussian and $\phi$ includes quadratic terms so that $\theta_0^\top \phi(x) = |x|^2/2$ for some $\theta_0$. The lower bound $H^{\sigma}_*$ then converges towards $H(p_*)$.
\begin{restatable}[Entropy bound]{conjecture}{entrbound}  
\label{conj:entropybound}
{ Assume that MGD admits solutions with density $p_t^\sigma$ for all $\sigma$ large enough, and that maximum entropy distributions with moments $m_t$ exists for all $t\in[0,1]$. }

If $Z$ has density $p_0 = {\cal Z}_0^{-1} e^{-\theta_0^\top \phi}$, then there exists $C > 0$ such that for all $\sigma > 0$,
\begin{equation}
\label{eq:H_epsilon_bound_conj}
H(p_*) -  H^{\sigma}_* \leq C\, \sigma^{-2}.
\end{equation}
\end{restatable}
Supporting arguments from the same Fokker--Planck analysis are given in Appendix~\ref{app:conjecture}, with numerical verification in Section~\ref{sec:convergence_numerics}. In practice, monitoring the convergence of $H_*^\sigma$ provides a diagnostic for the convergence of $p^\sigma_1$ to $p_*$.

\section{Numerical Validation}
\label{sec:convergence_numerics}

Section \ref {sec:num-convergence} studies the numerical convergence properties of Moment Guided Diffusions towards maximum entropy distributions, over distributions of one-dimensional $x \in \R$. We use a cosine interpolant defined by
\begin{equation*}
   \quad I_t = \cos(\tfrac12\pi t)\,Z + \sin(\tfrac12\pi t)\,X ~ 
\end{equation*}
and solve the MGD SDE~\eqref{eq:sde} with the numerical integrator specified in Section~\ref{sec:numalg}.
Section~\ref{sec:numerics-cost} shows empirically that the numerical complexity of the MGD sampling does not suffer from the non-convexity of the distributions as opposed to an MCMC sampling algorithm.

\subsection{Convergence towards Maximum Entropy Distributions}
\label{sec:num-convergence}

The MGD algorithm samples a distribution with density $p_1^\sigma$. We study its numerical convergence to the maximum entropy distribution $p_*(x) = Z^{-1}_{\theta_*} e^{-\theta_*^\top \phi(x)}$ and verify Conjectures~\ref{conj:convergence} and \ref{conj:entropybound} for different choices of data distributions and moment functions~$\phi$.

\subsubsection{Non-log-concave Density}

\begin{figure}[H]
    \centering
    
    \includegraphics[width=1\linewidth]{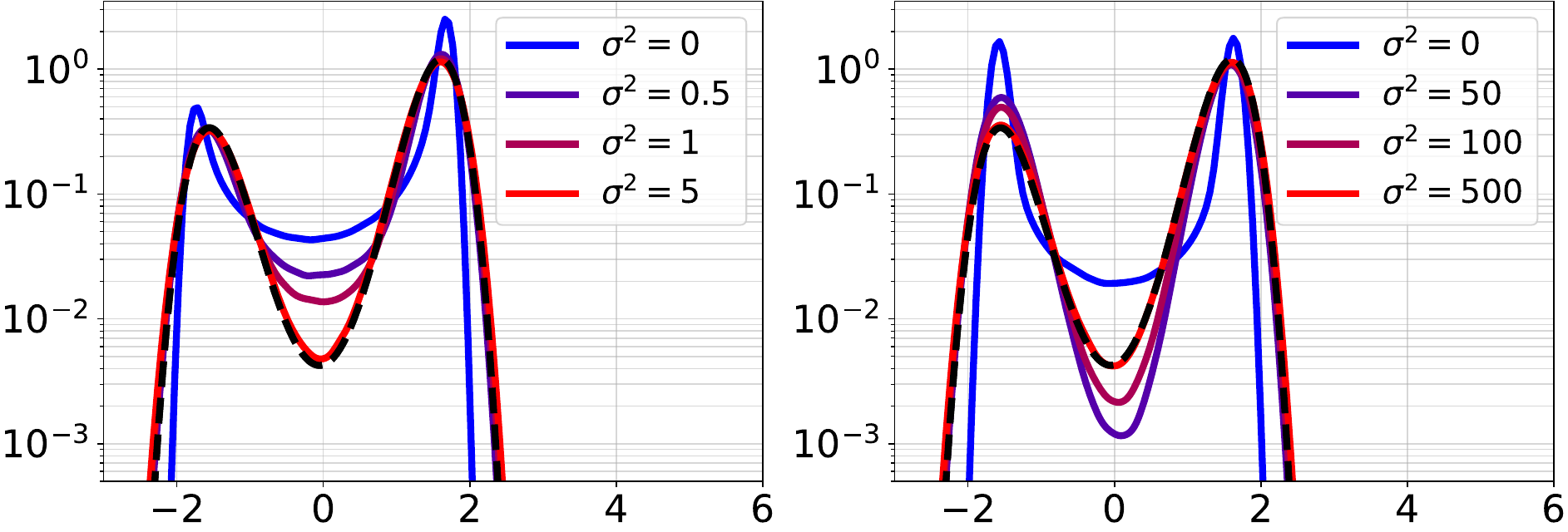}
  
    ~~~~~~(a)~~~~~~~~~~~~~~~~~~~~~~~~~~~~~~~~~~~~~~~~~~~(d)~~

    \vspace{.2cm}
    
    \includegraphics[width=1\linewidth]{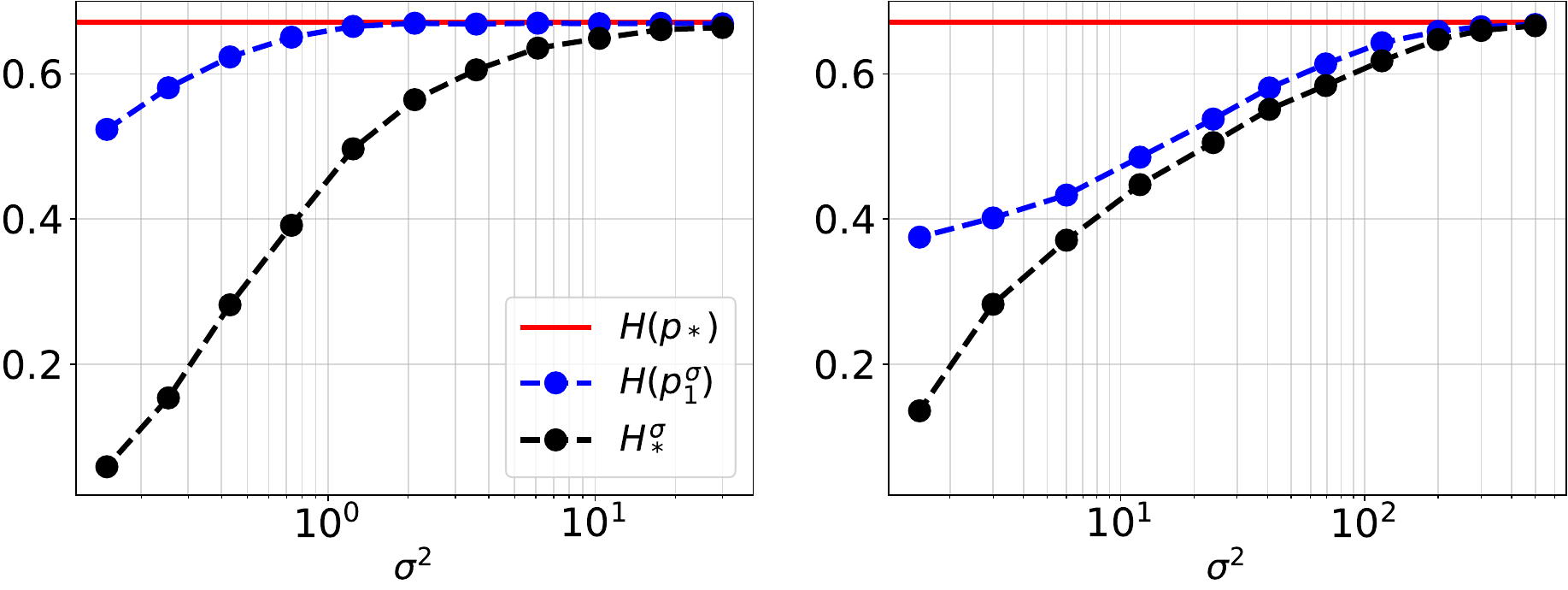}
    
    ~~~~~~(b)~~~~~~~~~~~~~~~~~~~~~~~~~~~~~~~~~~~~~~~~~~~(e)~~

    \vspace{.2cm}
    
    \includegraphics[width=1\linewidth]{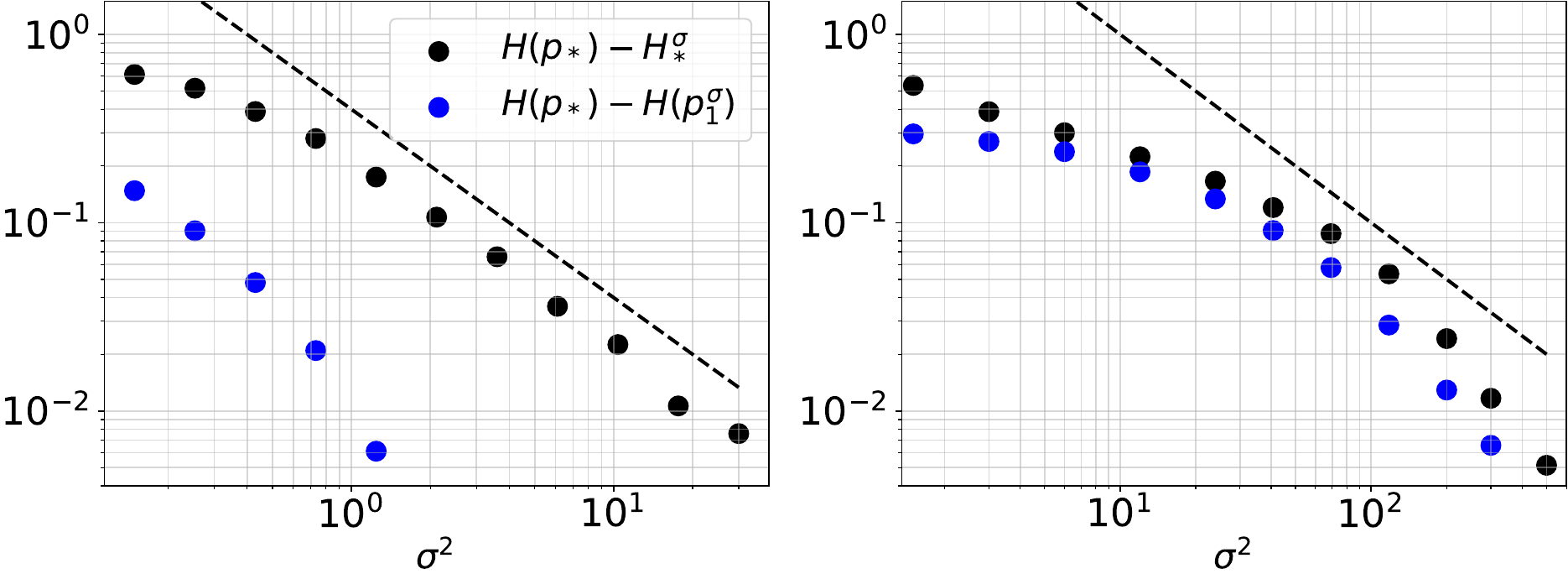}
    
    \vspace{.05cm}~~~~~~(c)~~~~~~~~~~~~~~~~~~~~~~~~~~~~~~~~~~~~~~~~~~(f)~

    \caption{Convergence of MGD towards the maximum entropy bimodal distribution $p_*(x) = \mathcal{Z}^{-1} e^{-\tfrac{4}{5}(x^4 - 5x^2 - x/2)}$ for $X \sim p = p_*$. Left column: moment function $\phi(x) = (x, x^2, x^3, x^4)$. Right column: $\phi(x) = (x^2, \log p(x))$. (a,d)~Log-density $\log p_*(x)$ (dashed) and $\log p_1^\sigma(x)$ for increasing $\sigma$ (blue to red). (b,e)~Maximum entropy $H(p_*)$ (red line), sampled entropy $H(p_1^\sigma)$ (blue dots), and lower bound $H_*^\sigma$ from~\eqref{eq:entropy} (black dots) versus $\sigma^2$. (c,f)~Entropy gaps $H(p_*) - H(p_1^\sigma)$ (blue) and $H(p_*) - H_*^\sigma$ (black) versus $\sigma^2$; the dashed line shows $\sigma^{-2}$ decay.}

    \label{fig:scalar_convergence}
\end{figure}

We consider data $X \sim p$ distributed according to an unbalanced bimodal density $p(x) = \mathcal{Z}^{-1} e^{-\tfrac{4}{5}(x^4 - 5x^2 - x/2)}$ for $x \in \mathbb{R}$, and the four-dimensional monomial map $\phi(x) = (x, x^2, x^3, x^4)$, whose moments are $\mathbb{E}[\phi(X)] \approx (0.8, 2.4, 2.2, 6.4)$. With this choice, the maximum entropy density satisfies $p_*(x) = p(x)$. (Note that $I_t$ for $t \in (0,1)$ is not distributed according to the maximum entropy distribution with moments $m_t$.)

The log-density $\log p_*(x)$ (dotted line) in Figure~\ref{fig:scalar_convergence}(a) exhibits two modes, reflecting a non-convex Gibbs energy. For small $\sigma$, the density $p_1^\sigma$ concentrates in two separate modes. Although these modes do not have the correct shape (they are too peaked, reflecting the lack of entropy of $p_1^\sigma$), their relative weight is correct. As $\sigma$ increases, the added noise allows particles to spread correctly inside the modes, and $p_1^\sigma$ progressively converges towards $p_*$, with near-superposition at $\sigma^2 = 5$.

Figure~\ref{fig:scalar_convergence}(b) quantifies this convergence via the entropy $H(p_1^\sigma)$ (blue), computed numerically from the distributions above. These values lie below $H(p_*) = 0.67$ (red). We observe that $\frac{d}{d\sigma}H(p_1^\sigma) \geq 0$ and that $H(p_1^\sigma) \to H(p_*)$ as $\sigma^2$ increases. Figure~\ref{fig:scalar_convergence}(c) shows $D_{\mathrm{KL}}(p_1^\sigma \| p_*) = H(p_*) - H(p_1^\sigma)$ (blue dots), which decays as $O(\sigma^{-2})$ on the log-log scale (black dashed line shows $\sigma^{-2}$ decay), validating Conjecture~\ref{conj:convergence}.

The lower bound $H_*^\sigma$ on $H(p_1^\sigma)$, computed from~\eqref{eq:entropy}, is shown as black dots in Figure~\ref{fig:scalar_convergence}(b). As expected, it lies below $H(p_1^\sigma)$ (blue) and also converges to $H(p_*)$. Figure~\ref{fig:scalar_convergence}(c) shows that $H(p_*) - H_*^\sigma$ (black dots) also decays as $O(\sigma^{-2})$, validating Conjecture~\ref{conj:entropybound}.

\subsubsection{Slower Convergence}\label{sec:slow-convergence}
In the previous example, $p_1^\sigma$ converges towards $p_*$ with negligible error for $\sigma^2 \ge 2$. We now show that the convergence constant $C$ in Conjecture~\ref{conj:convergence} depends critically on the choice of moment functions $\phi$.

When $p$ is known, a seemingly natural choice is $\phi(x) = (x^2, \log p(x))$, since this suffices to represent both the data density $p$ and the initial Gaussian density $p_0$, yielding $p_*(x) = p(x)$. For the bimodal density $p(x) = \mathcal{Z}^{-1} e^{-\tfrac{4}{5}(x^4 - 5x^2 - x/2)}$ with this $\phi$, Figures~\ref{fig:scalar_convergence}(e) and (f) confirm that $D_{\mathrm{KL}}(p_1^\sigma \| p_*)$ and $H(p_*) - H_*^\sigma$ both decay as $C\sigma^{-2}$  for $\sigma^2 \geq 50$, validating Conjectures~\ref{conj:convergence} and~\ref{conj:entropybound}. However, the constant $C$ is much larger than in the previous example: small errors require $\sigma^2 \geq 500$, or approximately $10^2$ times more integration steps.

Figure~\ref{fig:scalar_convergence}(d) shows the densities $p_1^\sigma$ for several values of $\sigma$. Although $p_1^\sigma$ is bimodal for $\sigma^2 \leq 1$, the relative weights of the two modes are off by one order of magnitude. This occurs because $\phi$ is not expressive enough to displace mass at early times $t$ of the MGD, before $p_t^\sigma$ becomes multimodal. For larger values of $\sigma^2$ (above $10^2$), MGD becomes analogous to a Langevin dynamic, recovering the correct relative weights through random switching of particles between modes.

\subsubsection{Non-smooth $\phi$}\label{sec:nonsmooth}
\label{subsub_nonsmooth}

\begin{figure}[H]
    \centering
    \includegraphics[width=1\linewidth]{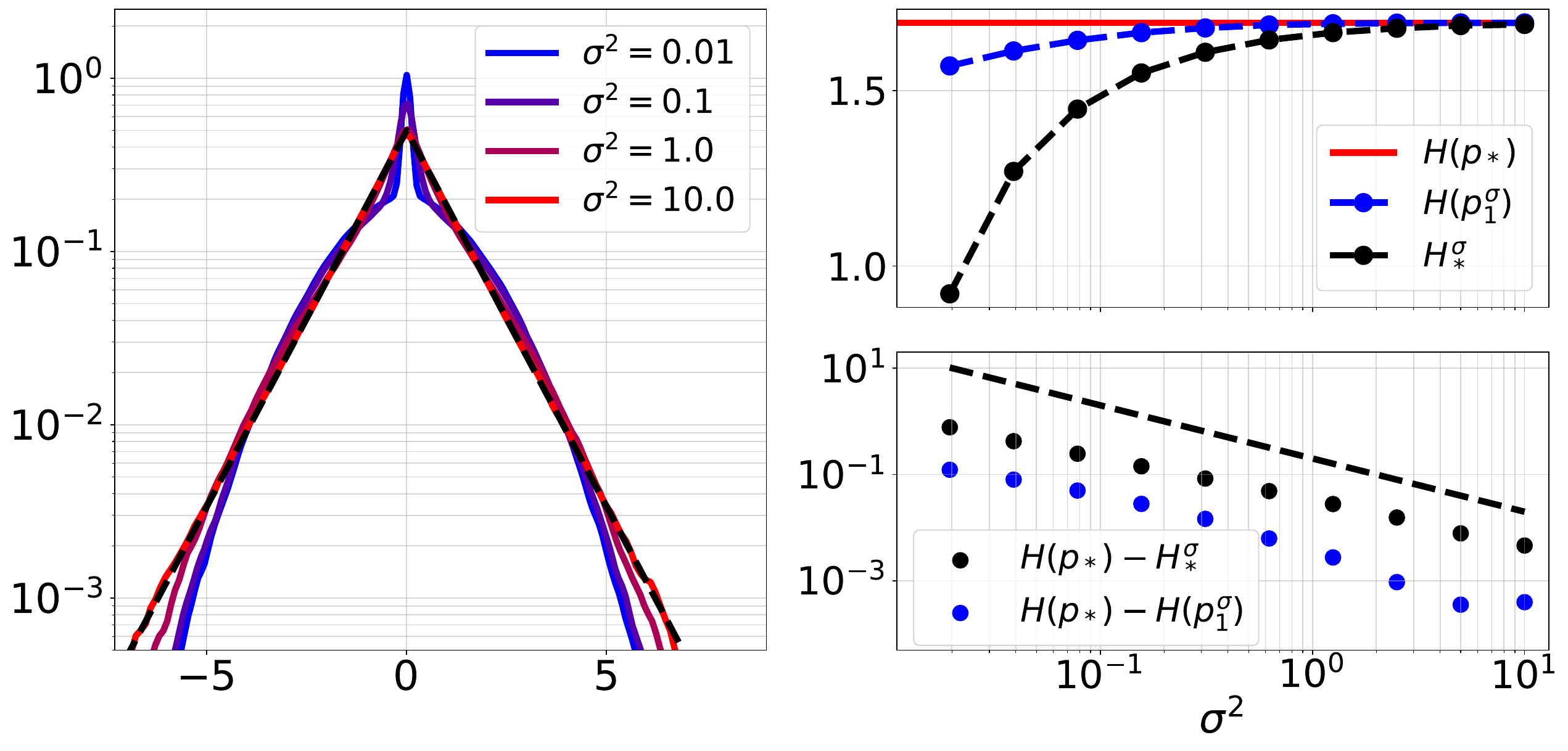}
    
    ~~~~(a)~~~~~~~~~~~~~~~~~~~~~~~~~~~~~~~~~~~~~~~~~~(b)
  
    \caption{Convergence of MGD towards the Laplacian maximum entropy distribution $p_*(x) = \tfrac{1}{2}e^{-|x|}$ for $X \sim p = p_*$. (a)~Log-density $\log p_*(x)$ (dashed) and $\log p_1^\sigma(x)$ for increasing $\sigma$ (blue to red). (b, top)~Maximum entropy $H(p_*)$ (red line), sampled entropy $H(p_1^\sigma)$ (blue dots), and lower bound $H_*^\sigma$ from~\eqref{eq:entropy} (black dots) versus $\sigma^2$. (b, bottom)~Entropy gaps $H(p_*) - H(p_1^\sigma)$ (blue) and $H(p_*) - H_*^\sigma$ (black) versus $\sigma^2$; the dashed line shows $\sigma^{-2}$ decay.}
   
    \label{fig:scalar_convergence_U_x}
\end{figure}

The MGD numerical scheme (Section~\ref{sec:numalg}) avoids computing $\Delta\phi$, which is essential when $\phi$ includes modulus or $\ell_1$ norms, as in the scattering spectra of Section~\ref{sec:scat}. We verify here that MGD correctly samples maximum entropy distributions defined by non-smooth $\phi$, and that Conjectures~\ref{conj:convergence} and~\ref{conj:entropybound} hold.

We consider data distributed according to the Laplacian density $p(x) = \frac{1}{2}e^{-|x|}$, which is the maximum entropy distribution $p = p_*$ for $\phi(x) = (x^2, |x|)$ with $\mathbb{E}[\phi(X)] = (2, 1)$. Figure~\ref{fig:scalar_convergence_U_x}(a) shows $\log p_1^\sigma(x)$ for various $\sigma$. As $\sigma$ increases, the curves converge to $\log p_*(x)$ (dashed), nearly superimposing at $\sigma^2 = 10$. For small $\sigma$, the density $p_1^\sigma$ exhibits a sharper spike near zero and shorter tails, reflecting insufficient entropy.

Convergence is quantified by $D_{\mathrm{KL}}(p_1^\sigma \| p_*) = H(p_*) - H(p_1^\sigma)$, which decreases to zero as $\sigma^2$ increases (Figure~\ref{fig:scalar_convergence_U_x}(b, top)). Figure~\ref{fig:scalar_convergence_U_x}(b, bottom) confirms $H(p_*) - H(p_1^\sigma) = O(\sigma^{-2})$ for $\sigma^2 \geq 10^{-2}$, validating Conjecture~\ref{conj:convergence}. The lower bound $H_*^\sigma$ from~\eqref{eq:entropy} (black dots) also converges to $H(p_*)$ at the same rate, validating Conjecture~\ref{conj:entropybound}.
Since the Laplacian is log-concave, there is no mass to displace between wells, so even a Langevin dynamic would converge quickly.

\subsection{Rate of Convergence and Multimodality}
\label{sec:numerics-cost}

We verify numerically that the computational cost of MGD does not depend on energy barrier heights, unlike MCMC methods. We use truncated monomial moment generating functions $\phi(x) = (x^k)_{k\leq r}$ (for $r=4$).

Figure~\ref{fig:alphaconvergence}(a) shows the log-density of unbalanced bimodal distributions
\[
p_*(x) = \mathcal{Z}_\beta^{-1} e^{-\beta(x^4 - 5x^2 -x/2)},
\]
with two modes separated by a barrier of height proportional to $\beta$. For MGD, we consider for simplicity that $X$ is distributed according to the maximum entropy distribution $p = p_*$. The computational cost of both MGD and the Metropolis Adjusted Langevin Algorithm (MALA) is proportional to the number $n_{\rm steps}$ of discretization steps. The cost per step differs between algorithms (typically higher for MGD), but it does not depend on $\beta$
so we do not take it into account.

\begin{figure}[H]
    \centering
   
    \includegraphics[width=1\linewidth]{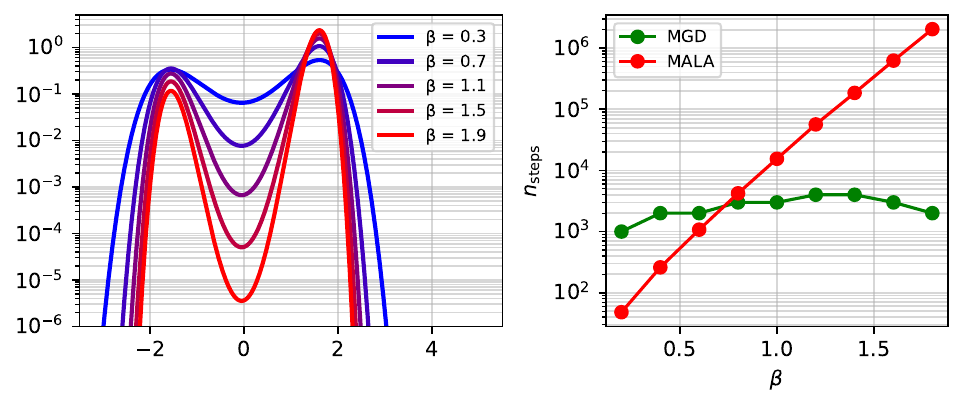}

    ~~~~~~~~~~~~~(a)~~~~~~~~~~~~~~~~~~~~~~~~~~~~~~~~~~~~~~~~~~(b)~~~~
    \caption{(a) Log-density $\log p_*(x)$ for $p_*(x) =\mathcal{Z}_\beta^{-1} e^{-\beta(x^4 - 5x^2 -x/2)}$ with increasing $\beta$ (blue to red). The two modes are separated by a barrier of height proportional to $\beta$. (b) Number of discretization steps $n_{\rm steps}$ required to reach a fixed Kullback--Leibler divergence from $p_*$, for MALA (red) and MGD (green), as a function of $\beta$. For MALA, $n_{\rm steps}$ grows exponentially with $\beta$; for MGD, it remains nearly constant.}
    \label{fig:alphaconvergence}
\end{figure} 

MALA computes samples via discretized Langevin dynamics initialized from Gaussian white noise, with an accept-reject operation which eliminates the discretization bias. The step size is tuned to achieve an optimal acceptance rate of approximately $0.57$ \cite{roberts1998optimal}. Although the sampled distribution $\tilde{p}$ converges to $p_*$ as $n_{\rm steps}$ increases, this convergence depends exponentially on $\beta$. Indeed, crossing an energy barrier by adding Gaussian noise has probability exponentially small in the barrier height.

We measure the minimum number of steps $n_{\rm steps}$ required to reach a fixed error $D_{\mathrm{KL}}(\tilde{p} \| p_*) =10^{-3}$. Figure~\ref{fig:alphaconvergence}(b) confirms that for MALA (red), $n_{\rm steps}$ grows exponentially with $\beta$, making it computationally prohibitive for non-convex distributions— especially in higher dimensions.

For MGD, we know that $n_{\rm steps}=n_\sigma=O( \sigma^2)$ and $D_{\mathrm{KL}}(p_1^\sigma \| p_*) = O(\sigma^{-2})$. We choose $\sigma$ so that the discretized MGD satisfies $D_{\mathrm{KL}}(\tilde{p}_1^\sigma \| p_*) =  10^{-3}$.
 We run this experiment for the moment generating function $\phi(x) = (x,x^2,x^3,x^4)$.
Figure~\ref{fig:alphaconvergence}(b) shows that for MGD (green), $n_{\rm{steps}}$ remains approximately constant as $\beta$ increases. It verifies that the MGD computational cost does not suffer from multimodality. 

The homotopic transport (Section~\ref{sec:convergence_th}) is able to distribute samples into correct modes early, enabling efficient sampling. However, this property  requires $\phi$ to be sufficiently rich to capture the mass transport at early times; for $\phi(x) = (x^2,\log p(x))$, MGD would revert to MCMC-like behavior.

\section{Generation of Multiscale Processes in Finance and Physics}
\label{sec:scat}

{ This section applies the MGD algorithm to sample high-dimensional maximum entropy distributions, using multiscale maximum entropy models based on wavelet scattering moments, reviewed in Section~\ref{sec:scat_def}. We apply these models to financial time series (Section~\ref{sec:1d}) as well as two-dimensional turbulent and cosmological fields (Section~\ref{sec:scat_def}).} To validate Conjectures~\ref{conj:convergence} and \ref{conj:entropybound}, we compute the lower bound of the entropy of sampled distributions, and study its convergence as the volatility $\sigma$ increases. We also estimate negentropy, which quantifies order and non-Gaussianity (Section~\ref{sec-negentropy}).  Finally, we show numerically in Section~\ref{sec:cv_modelisation_error} that the convergence of MGD to the maximum entropy distribution $p_*$ does not depend upon the model error $D_{\mathrm{KL}}(p \| p_*)$.

\subsection{Negentropy Rate}
\label{sec-negentropy}
The negentropy was introduced in statistical physics by Erwin Schrödinger \cite{schrodinger1944life} to 
measure the distance of system to equilibrium, and to give a measure of order and information. The negentropy usually cannot be measured for high-dimensional systems because estimating the entropy is generally intractable. 

The negentropy of a random vector $X$ is defined as the difference between the entropy $H(p)$ of the density $p$ of $X$ and the entropy $H(g)$ of the gaussian density $g$ having the same covariance $\Sigma$ as $p$. The negentropy rate is normalized by the dimension $d$ of $X$ and can be rewritten as the Kullback Leibler divergence between $p$ and the Gaussian $g$:
\begin{equation}
 \Delta H(p) = d^{-1} (H(g) - H(p)) = d^{-1} D_{\rm KL} (p \| g) \geq 0 ,
\end{equation}
where the Gaussian entropy is given by
\begin{equation}
    H(g) = \frac d 2 \log \Big(2 \pi e \,({\rm det} \Sigma)^{1/d} \Big) .
\end{equation}

The negentropy rate converges when $d$ goes to infinity for extensive processes for which the entropy rate $d^{-1} H(p)$ converges. It is invariant to the action of an invertible linear operator on $X$ and hence does not depend upon the covariance of $X$, if it is invertible. In that sense it is an intrinsic measure of non-Gaussian properties of $X$. 

If $p_*$ is the maximum entropy distribution conditioned by the moment value
$\E_p[\phi]$ then $H(p_*) \geq H(p)$ and $H(p_*) - H(p) = D_{\rm KL}(p \| p_*)$ and, as a result,
\begin{equation}
\label{eq_full_neg_entrop}
  \Delta H(p) =  d^{-1}\big(H(g) - H(p_*) +  D_{\rm KL}(p \| p_*)\big) .
\end{equation}
This implies that  $d^{-1}\big (H(g)-H(p_*)\big)$ is a lower bound of the negentropy $\Delta H(p)$ of $p$, which depends upon the accuracy of the maximum entropy model $p_*$ defined by $D_{\rm KL}(p \| p_*)$. The following sections give an estimate of this negentropy rate with the MGD algorithm, by computing the lower bound $H_*^\sigma$ in (\ref{eq:entropy}) of  $H(p_*)$, and
\begin{equation}
\label{neg-entrop-estim}
\Delta H^\sigma_* = d^{-1}\big (H(g)-H_*^\sigma\big) .
\end{equation}
The convergence of $H_*^\sigma$ when $\sigma$ increases is equivalent to the convergence of $\Delta H^\sigma_*$. In particular, Conjecture~\ref{conj:entropybound} states that $\Delta H^\sigma_*$ should converge at rate $O(\sigma^{-2})$.

\subsection{Wavelet Scattering Spectra}
\label{sec:scat_def}

The wavelet scattering transform was introduced in \cite{mallat2012group} for signal classification and modeling. We compute maximum entropy models {of the physical and financial data considered in this section} from wavelet scattering moments \cite{Morel2022ScaleDA, Cheng2023ScatteringSM}. These moments capture dependencies across scales using a complex wavelet transform. Until now, such high-dimensional maximum entropy distributions could only be sampled with a microcanonical gradient descent algorithm \cite{bruna2019multiscale}, which introduces approximation errors. We briefly review complex wavelet transforms in one and two dimensions before defining wavelet scattering moments.

\subsubsection{Wavelet Transform} A wavelet $\psi(u)$ is a function with fast decay in $u \in \mathbb{R}^\kappa$ satisfying $\int \psi(u) \, du = 0$. Its Fourier transform is centered at a frequency $\xi \neq 0$ with fast decay away from $\xi$. Here $\kappa = 1$ for time series and $\kappa = 2$ for images.
In numerical applications we use a \textit{Morlet} wavelet 
\[
\psi(u) =\frac{1}{(2\pi\sigma^2)^{\kappa/2}} e^{-\frac{|u|^2}{2\sigma^2}} \left(e^{i\xi^\top u} - c\right),
\]
where $c$ is adjusted
so that $\int \psi(u) \, du = 0$.  
As in \cite{Cheng2023ScatteringSM, Morel2022ScaleDA}, we set $\sigma = 0.8$, and $\xi = {3}/{4}$ if $\kappa =1$ and $\xi = (3/4, 0)$ if $\kappa = 2$.
Figure~\ref{fig:wavelet}(a,b) shows the real and imaginary parts of $\psi$ in one and two dimensions.

In one dimension, wavelets are dilated by a scale $2^j$:
\[
   \psi_\lambda(u) = 2^{-j/2} \psi(2^{-j} u).
\]
The Fourier transform of $\psi_\lambda$ is centered at frequency $\lambda = 2^{-j} \xi$. In two dimensions, the wavelet is also rotated by an angle $\ell \pi / L$ for $0 \leq \ell < L$:
\begin{equation}
    \psi_\lambda(u) = 2^{-j\kappa/2} \psi(2^{-j} R_\ell u),
\end{equation}
with center frequency $\lambda = 2^{-j} R_\ell \xi$. We use $L = 4$ orientations in all experiments.

The wavelet transform of $X$ is an invertible linear operator which captures variations at all scales $2^j$ and orientations $\ell \pi / L$, equivalently filtering into frequency bands of constant octave bandwidth centered at each $\lambda$ \cite{mallat1999wavelet}. It is computed via discrete convolutions on the sampling grid of $X$ of size $d$, with periodic boundary conditions:
\begin{equation}
    X^\lambda(u) \underset{\rm def}{:=} X * \psi_\lambda(u).
\end{equation}
The scale $2^j$ satisfies $1 < 2^j \leq d$, so there are at most $\log_2 d$ wavelet
frequencies $\lambda$ in one dimension, and at most $L \log_2 d$ in two dimensions. The lowest frequencies are captured by a low-pass filter $\psi_0$ centered at $\lambda = 0$. 

\begin{figure}[H]
    \centering
    \includegraphics[width=0.7\linewidth]{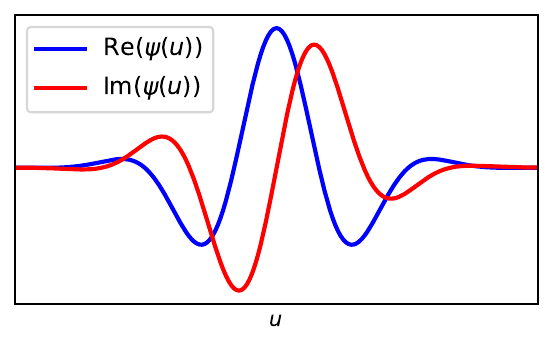}

     (a)

     \vspace{.5cm}
   
   \hspace{-.3cm} \includegraphics[width=0.7\linewidth]{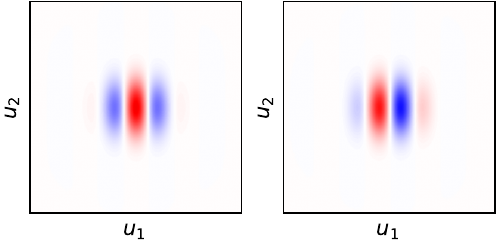} 

   (b)
   
    \caption{(a): One-dimensional Morlet wavelet $\psi$. The wavelet is a complex function whose real and imaginary parts are respectively in blue and red. (b): real (left) and imaginary (right) parts of a two-dimensional Morlet wavelet.}
    \label{fig:wavelet}
\end{figure}

\subsubsection{Wavelet Scattering Spectra}
We summarize the calculation of empirical wavelet scattering moments used in the numerical experiments of Sections~\ref{sec:1d} and \ref{sec:2d}.

The modulus of complex wavelet coefficients $|X^\lambda|$ measures the amplitude of local signal variations at multiple scales and orientations. The first two empirical scattering moments are empirical means of $|X^\lambda(u)|$ and $|X^\lambda(u)|^2$:
\begin{equation}
\label{eq:scat_spectra1}
\begin{split}
\phi_1(X) &= \Big(d^{-1} \sum_u |X^\lambda(u)| \Big)_{\lambda}, \\
\phi_2(X) &= \Big(d^{-1} \sum_u |X^\lambda(u)|^2 \Big)_{\lambda}.
\end{split}
\end{equation}
These empirical averages converge to expected values as $d$ increases, under appropriate ergodicity assumptions. The dimension of $\phi_1$ and $\phi_2$ is $O(\log d)$.
The ratio $ \sum_u |X^\lambda(u)| /  \sum_u |X^\lambda(u)|^2 $ decreases when the sparsity of $X^\lambda$ increases.

Interactions across scales are captured by a second wavelet transform of each modulus, $|X^\lambda| * \psi_{\lambda'}$, which measures variations of $|X^\lambda(u)|$ at lower frequencies $|\lambda'| < |\lambda|$. We get $O(\log_2^2 d)$ cross-scale correlations with wavelet coefficients
$X^{\lambda'}$ at the frequency $\lambda'$
\begin{equation}
\label{eq:scat_spectra2}
\phi_3(X) = \Big(d^{-1} \sum_u |X^\lambda| * \psi_{\lambda'}(u) \, X^{\lambda'}(u)^* \Big)_{\lambda', \lambda},
\end{equation}
for all $\lambda, \lambda'$ with $|\lambda| > |\lambda'|$.
The imaginary parts of these moments are sensitive to the transformation $X(u) \rightarrow X(-u)$, allowing them to characterize temporal asymmetries for 1D signals and spatial asymmetries for 2D fields. 

We also compute $O(\log_2^3 d)$ cross-scale correlations between modulus
wavelet coefficients at different frequencies $\lambda$ and $\lambda''$,
filtered by a same wavelet of frequency $\lambda'$
\begin{equation}
\phi_4(X) = \Big(d^{-1} \sum_u |X^\lambda| * \psi_{\lambda'}(u) \, |X^{\lambda''}| * \psi_{\lambda'}(u)^* \Big)_{\lambda, \lambda', \lambda''},
\end{equation}
for all $|\lambda| > |\lambda'|$ and $|\lambda''| > |\lambda'|$.
Observe that if we replace  $|X^\lambda|$ by $|X^\lambda|^2$ then
$\phi^3(X)$ and $\phi^4(X)$ are empirical moments of order $3$ and $4$.
As explained in \cite{Cheng2023ScatteringSM, Morel2022ScaleDA}, using
$|X^\lambda|$ defines lower variance estimators which have similar properties.

The full vector of empirical scattering moments
\begin{equation}
\label{eq:wave-scat}
\phi(X) = \big(\phi_1(X), \phi_2(X), \phi_3(X), \phi_4(X)\big),
\end{equation}
has a dimension $r = O(\log_2^3 d)$.
These empirical moments are invariant to translations of $X$. It results that a maximum entropy distribution conditioned by $m = \E_p [\phi]$ is necessarily stationary. 
We shall see that the richness of scattering empirical moments
is sufficient to ensure a quick convergence of the MGD homotopic transport discussed in Section~\ref{sec:convergence_th}.

\subsection{Generation of Multiscale Time Series in Finance}
\label{sec:1d}

Financial time series are examples of one-dimensional multiscale sequences with strong non-Gaussian properties, including bursts of activity and time-reversal asymmetry. If $\mathrm{P}(u)$ denotes the daily closing price at time $u$, then $X(u) = \log \mathrm{P}(u) - \log \mathrm{P}(u-1)$ is the corresponding log-return. Figure~\ref{fig:sampling_1D}(a) displays S\&P 500 daily log-returns from January 2000 to February 2024, a series of $d = 6060$ time steps exhibiting strong intermittency and fat tails. Stochastic models of such time series are crucial for risk management, pricing, and hedging of contingent claims. Often, as with the S\&P 500, only a single realization of the process is available.

Wavelet moments can be estimated from empirical sums under the assumption that the increments are stationary and ergodic \cite{Morel2022ScaleDA}, so that $\phi(X) \approx \mathbb{E}[\phi(X)]$. Figure~\ref{fig:sampling_1D}(b) shows a sample $X_1^\sigma$ generated by MGD with $\sigma^2 = 5.5$, using $r = 217$ empirical wavelet scattering moments \eqref{eq:wave-scat} computed with the Morlet wavelet of Figure~\ref{fig:wavelet}(a). The intermittent behavior and bursts are qualitatively reproduced. In the following, we do not analyze the accuracy of this wavelet scattering model, which is studied in \cite{Morel2022ScaleDA}, but rather focus on the entropy properties of MGD samples as the volatility $\sigma$ increases.

\begin{figure}[H]
    \centering
    \includegraphics[width=.45\textwidth]{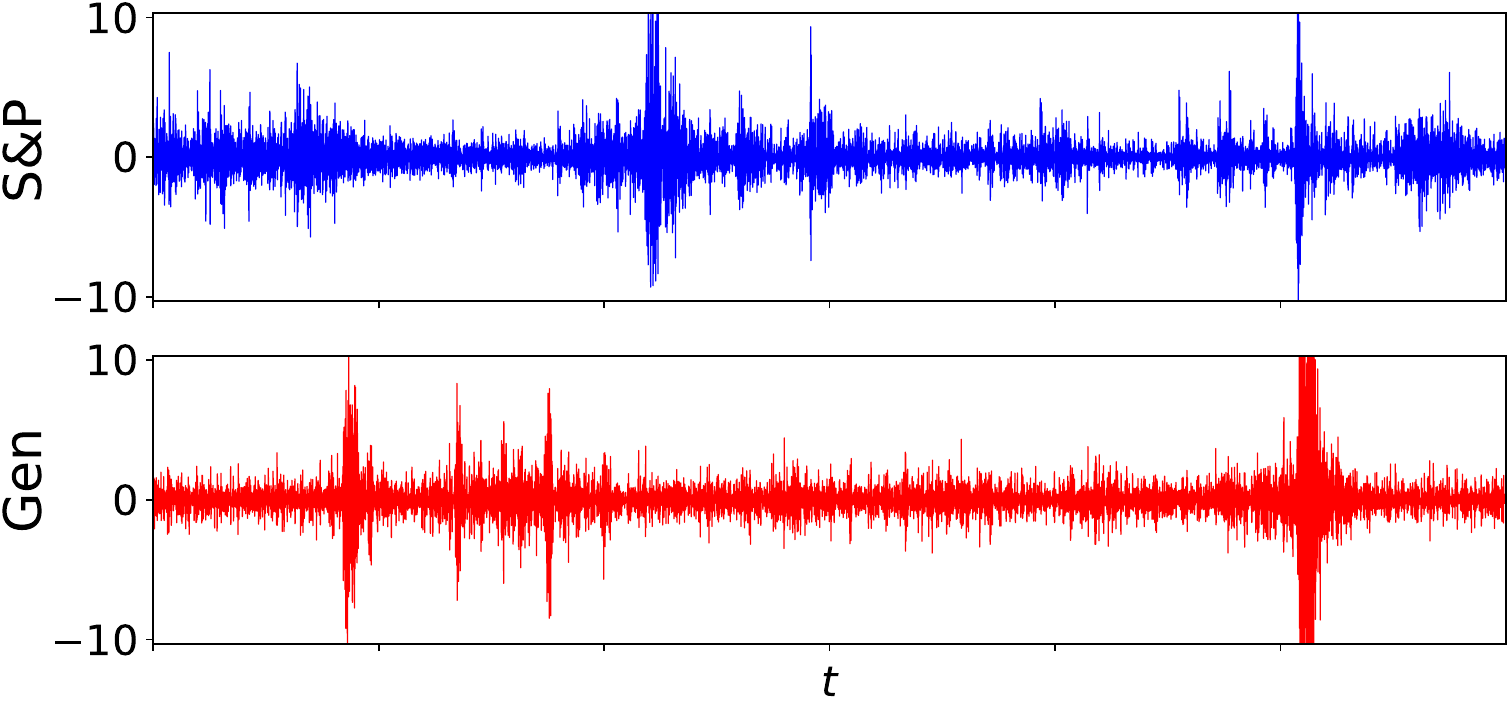}
   \caption{(a) S\&P 500 daily log-returns ($d = 6060$) from January 2000 to February 2024. (b) Sample generated by MGD with $\sigma^2 = 5.5$, using $r = 217$ empirical wavelet scattering moments \eqref{eq:wave-scat} computed from (a). Intermittency is reproduced.}
    \label{fig:sampling_1D}
\end{figure}

Unlike the numerical examples of Section~\ref{sec:convergence_numerics}, here the true distribution $p$ of $X$ and the maximum entropy distribution $p_*$ constrained by wavelet scattering moments are unknown. Nor can we compute the entropy $H(p_1^\sigma)$ directly; only the lower bound $H_*^\sigma$ from \eqref{eq:entropy} is accessible. We therefore test convergence of the sampled density $p_1^\sigma$ through the entropy lower bound $H_*^\sigma$ as $\sigma$ increases.

Figure~\ref{fig:analysis_1D}(a) shows that the negentropy estimate $\Delta H_*^\sigma = d^{-1}(H(g) - H_*^\sigma)$ decreases before reaching a plateau for $\sigma^2 \geq 2.5$, indicating that $H_*^\sigma$ increases and then stabilizes. At $\sigma_{\max}^2 = 5.5$, the negentropy estimate is $\Delta H_*^{\sigma_{\max}} = 0.05$, small compared to other non-Gaussian processes reported in Table~\ref{table:negentropy}. This is expected, as Gaussian models are often used as first-order approximations of financial time series. Nevertheless, this negentropy captures non-Gaussian phenomena such as the bursts of activity visible in Figure~\ref{fig:sampling_1D}(a).

\begin{figure}[H]
    
    \centering
    \includegraphics[width=0.245\textwidth]{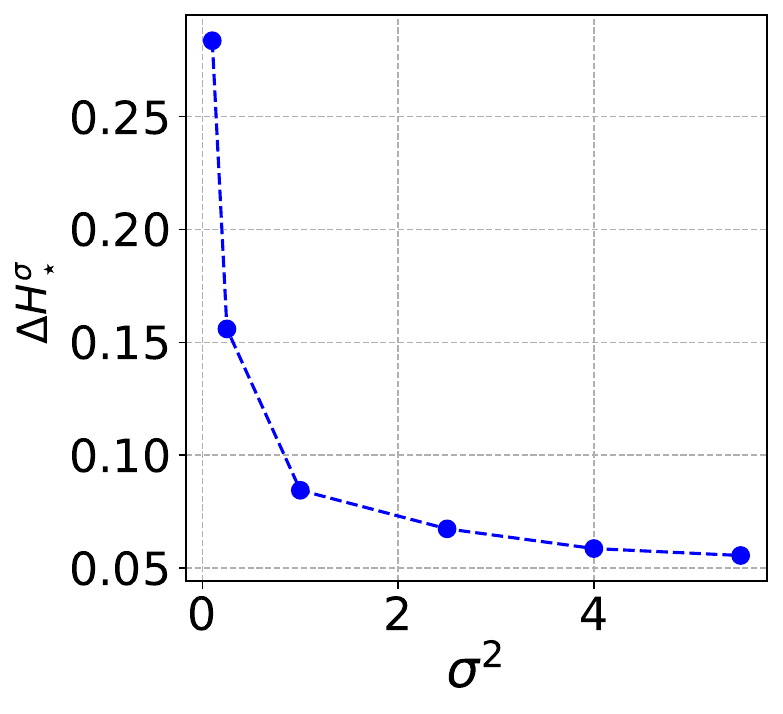}
    \includegraphics[width=0.23\textwidth]{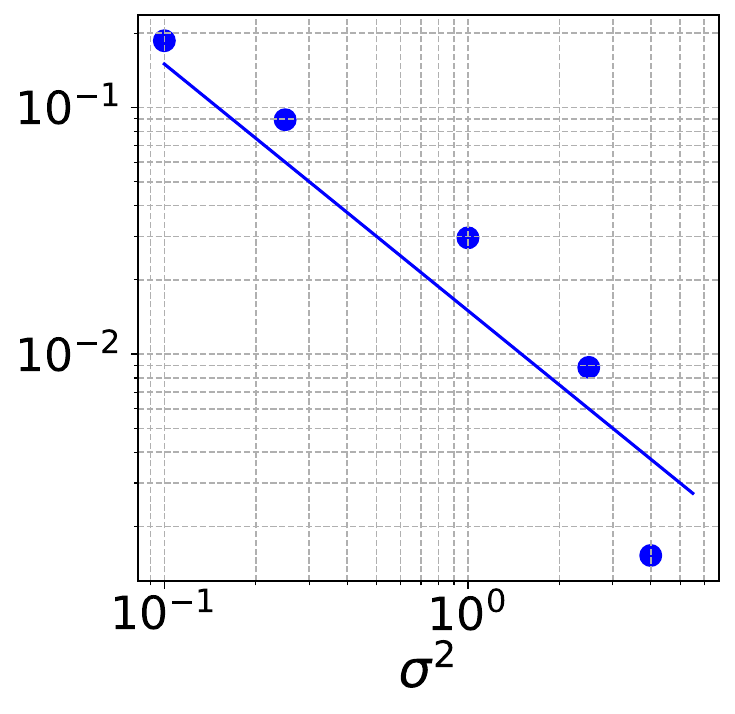}

    ~~~~~~~~~(a)~~~~~~~~~~~~~~~~~~~~~~~~~~~~~~~~~~~~~~~(b)
    \caption{(a) Negentropy estimate $\Delta H_*^\sigma$ from \eqref{neg-entrop-estim} versus $\sigma^2$ for S\&P 500 log-returns. (b) Convergence of $H_*^\sigma$, measured by $d^{-1}(H_*^{\sigma_{\max}} - H_*^\sigma)$ with $\sigma_{\max}^2 = 5.5$, versus $\sigma^2$. The plain line shows $\sigma^{-2}$ decay.}
     \label{fig:analysis_1D}
\end{figure}

Figure~\ref{fig:analysis_1D}(b) shows the convergence rate of $d^{-1}(H_*^{\sigma_{\max}} - H_*^\sigma)$ as a function of $\sigma^2$, for sufficiently large $\sigma_{\max}$. The negentropy estimate $\Delta H_*^\sigma$ converges as $O(\sigma^{-2})$, consistent with Conjecture~\ref{conj:entropybound}. However, since $H(p_*)$ is unknown, we cannot guarantee convergence to $H(p_*)$.

\begin{figure}[H]
    \centering
    \includegraphics[width=0.42\textwidth]{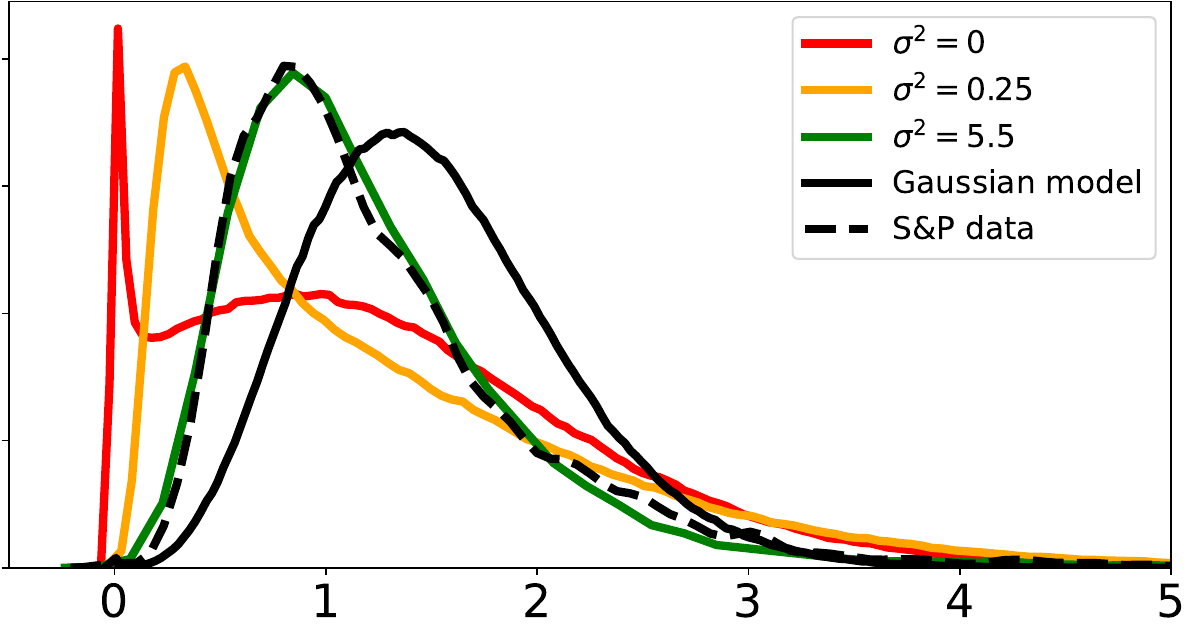}
    \caption{Histograms of rolling volatility $\mathrm{vol}(u)$ computed over $w = 5$ days. Dashed: S\&P log-returns $X$. Red to green: MGD samples $X_1^\sigma$ for increasing $\sigma^2$. Black: Gaussian process with the same covariance. At $\sigma^2 = 5.5$, the histogram of $X_1^\sigma$ matches the S\&P and differs from the Gaussian.}
    \label{fig:rolling_vol}
\end{figure}

If the volatility $\sigma$ is too small, $p_1^\sigma$ does not reach the maximum entropy density $p_*$. This manifests as excess intermittency, measured by the rolling volatility of $X_1^\sigma$. For zero-mean price increments $X(u)$, the rolling volatility is defined as the local standard deviation over time windows of size $w$:
\begin{equation*}
\mathrm{vol}(u) = \Big(w^{-1} \sum_{v=0}^{w-1} |X(u-v)|^2 \Big)^{1/2}.
\end{equation*}
Figure~\ref{fig:rolling_vol} shows histograms of rolling volatility: the original S\&P increments $X$ (dashed), MGD samples $X_1^\sigma$ for various $\sigma$ (colored), and a Gaussian process (black) having the same quadratic moments $\mathbb{E}[\phi_2(X)]$ as the S\&P. The mismatch between the rolling volatility of the S\&P and the Gaussian process confirms that the S\&P is non-Gaussian.

When $\sigma$ is too small, the histogram exhibits a sharp peak at low volatility and a heavier tail, indicating stronger bursts of energy interspersed with more regular variations. At $\sigma^2 = 5.5$, where the entropy $H_*^\sigma$ has nearly converged to its maximum, the volatility histogram matches that of the S\&P increments. This agreement is a partial validation of the model since rolling volatility was not explicitly incorporated into the wavelet scattering model.

\subsection{Generation of Two-Dimensional Physical Fields}
\label{sec:2d}

Similar numerical experiments are performed on two-dimensional physical fields. We consider cosmological and turbulent fluid fields, which are non-Gaussian stationary fields with long-range spatial dependencies and coherent geometric structures. Estimating energy models of out-of-equilibrium systems is central to statistical physics \cite{brossollet2025effective,boffi2024deep}. 

Original samples are shown in the top row of Figure~\ref{fig:sampling_2D_synthesis}. Figure~\ref{fig:sampling_2D_synthesis}(a) shows a cosmic web field, constructed by extracting a 2D slice from a 3D simulation of the large-scale dark matter distribution \cite{villaescusa2020quijote} with a logarithmic transformation \cite{Cheng2023ScatteringSM}. Figure~\ref{fig:sampling_2D_synthesis}(b) shows a turbulent vorticity field from a 2D incompressible Navier--Stokes simulation \cite{SCHNEIDER01012006}, with periodic boundary conditions. These fields have dimension $d = 128^2$ and are modeled with $r = 2392$ scattering moments, estimated on batches of $100$ replicas in MGD.

\begin{figure}[H]
    \centering
    
    \includegraphics[width=.22\textwidth]{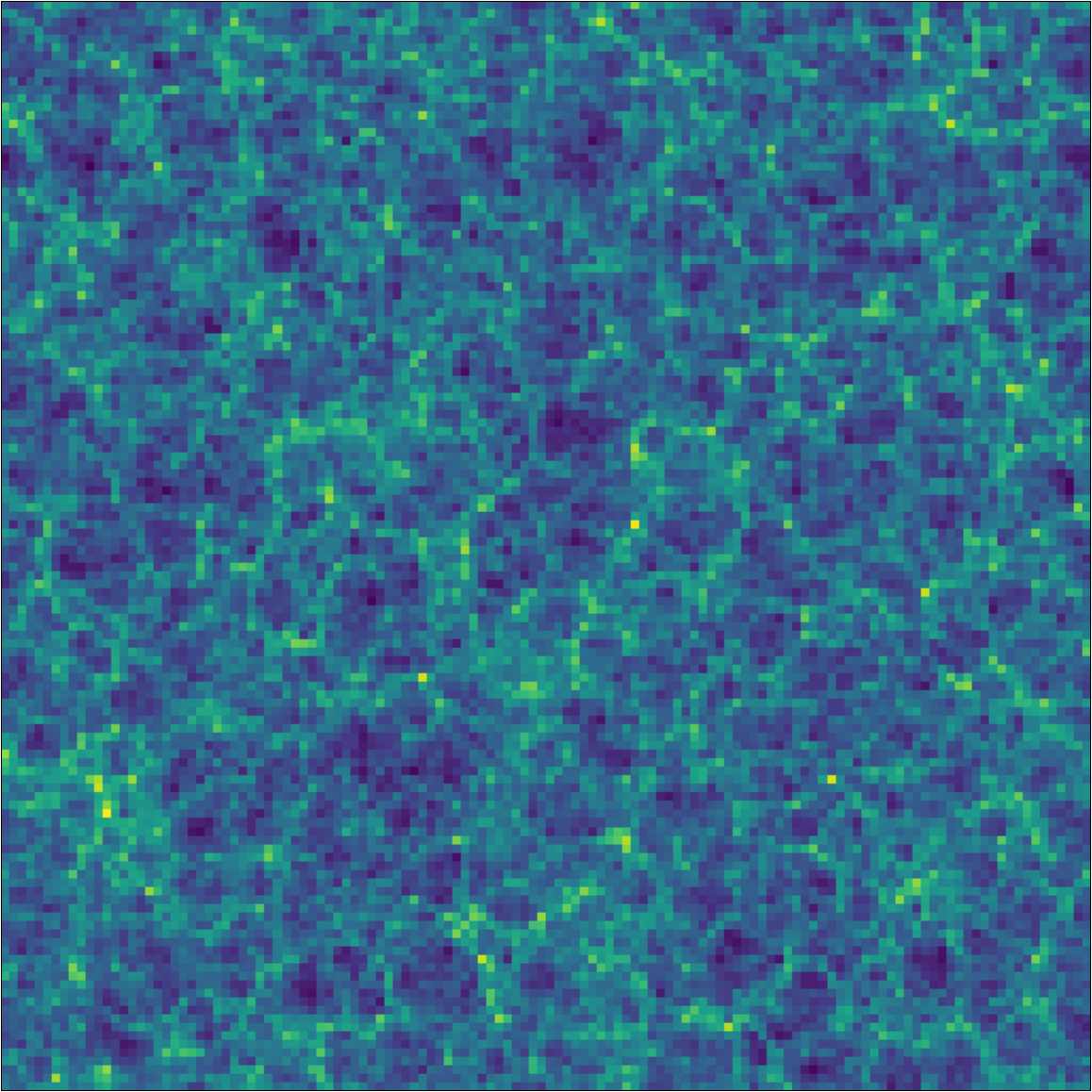}
    \includegraphics[width=0.22\textwidth]{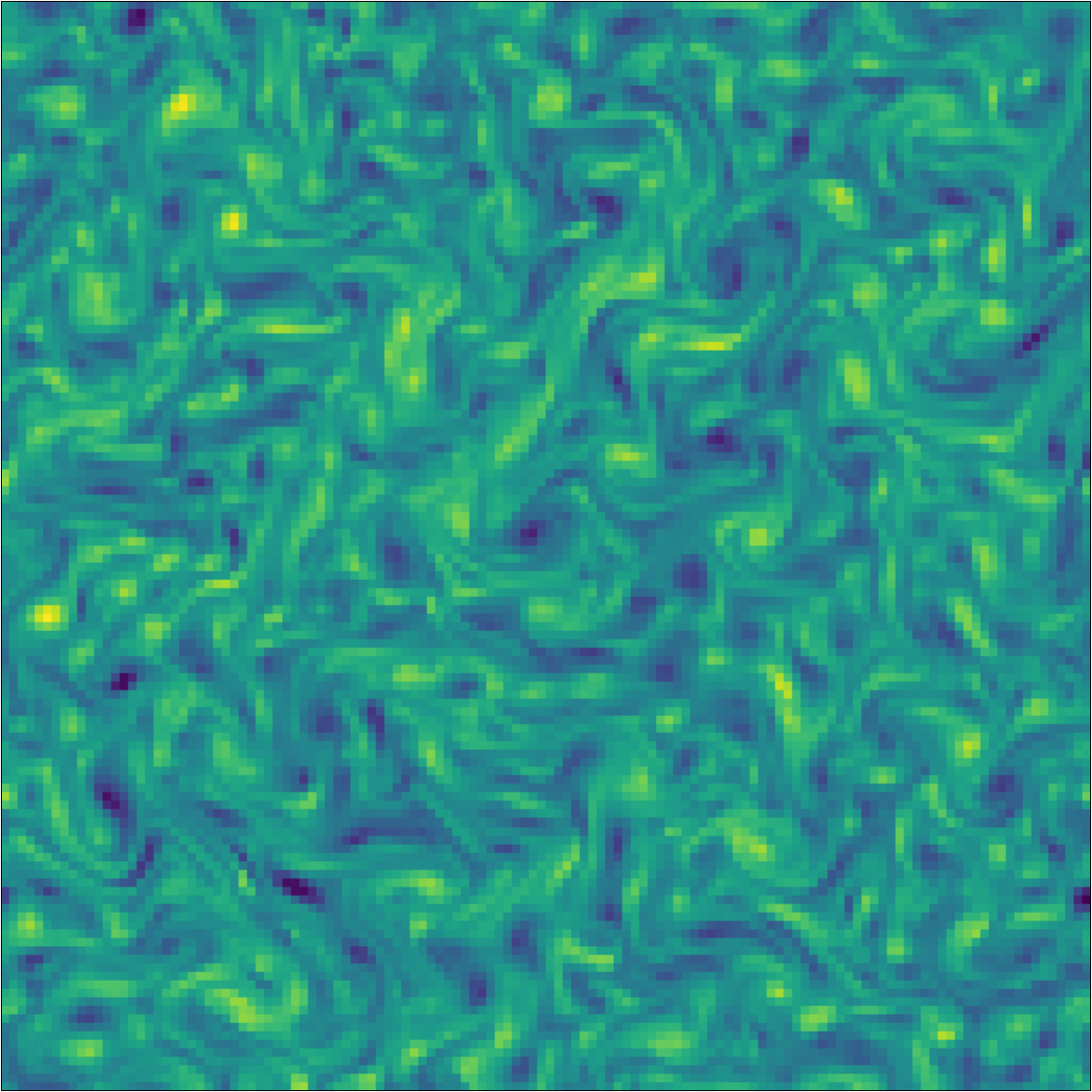}\\
    (a)~~~~~~~~~~~~~~~~~~~~~~~~~~~~~~~(b)
    
    \vspace{.2cm}
    
    \includegraphics[width=0.22\textwidth]{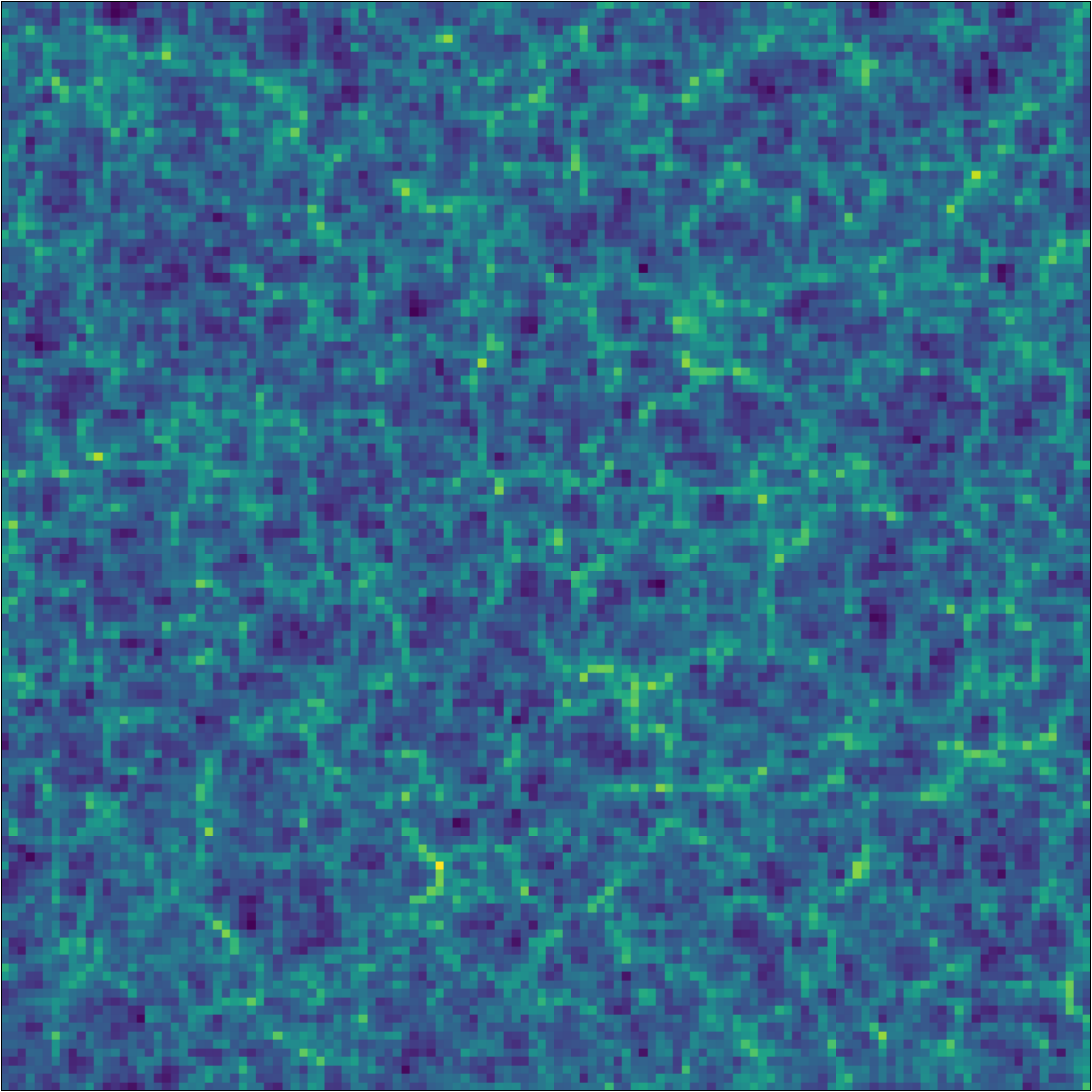}
    \includegraphics[width=0.22\textwidth]{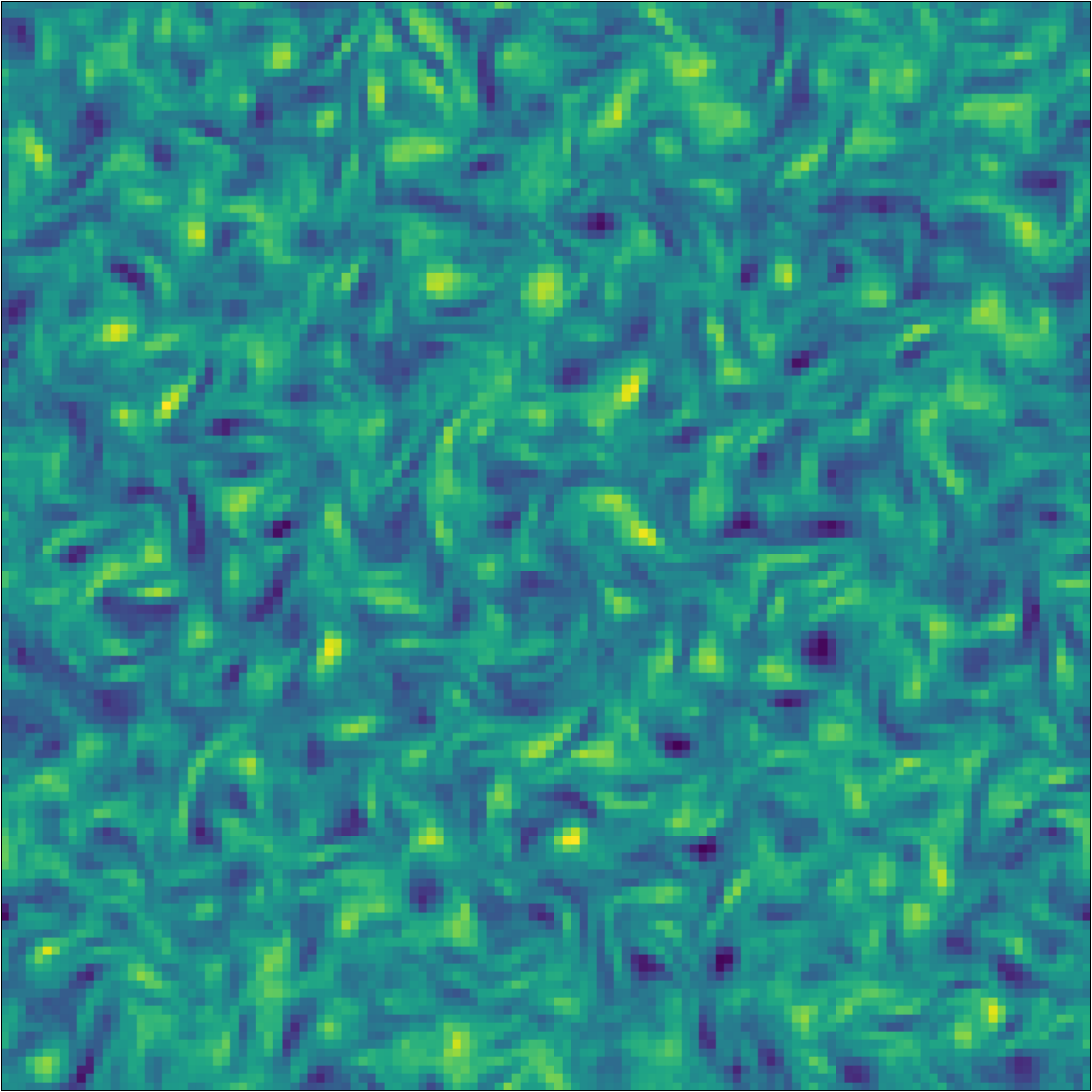}

    (c)~~~~~~~~~~~~~~~~~~~~~~~~~~~~~~~(d)
    
    \caption{(a) Cosmic web field: 2D slice from a 3D dark matter simulation. (b) Turbulent vorticity field from 2D incompressible Navier--Stokes. Both images are $128 \times 128$ pixels. (c,d) MGD samples with wavelet scattering moments at $\sigma^2 = 5.5$.}
    \label{fig:sampling_2D_synthesis}
\end{figure}

The samples in Figures~\ref{fig:sampling_2D_synthesis}(c,d), generated by MGD with $\sigma_{\max}^2 = 5.5$, are visually similar to the originals. The quality is comparable to results from the ad-hoc microcanonical algorithm of \cite{Cheng2023ScatteringSM}, which performs moment matching without controlling entropy.

As for financial time series, we test convergence of $p_1^\sigma$ through the lower bound $H_*^\sigma$ of its entropy, via the negentropy estimate $\Delta H_*^\sigma = d^{-1}(H(g) - H_*^\sigma)$. Figure~\ref{fig:sampling_2D_entropy}(a) shows that $\Delta H_*^\sigma$ decreases and reaches a plateau for $\sigma^2 \geq 2.5$, similar to the S\&P time series. Figure~\ref{fig:sampling_2D_entropy}(b) displays the convergence of $H_*^\sigma$ by computing $d^{-1}(H_*^{\sigma_{\max}} - H_*^\sigma)$ for $\sigma_{\max}^2 = 5.5$, for turbulence (red) and cosmic web (blue). The decay is proportional to $\sigma^{-2}$, supporting Conjecture~\ref{conj:entropybound}.

At $\sigma_{\max}^2 = 5.5$, the negentropy estimate is $\Delta H_*^{\sigma_{\max}} = 0.34$ for turbulence, much larger than $\Delta H_*^{\sigma_{\max}} = 0.07$ for the cosmic web. This reflects the stronger geometric regularity of turbulent fields, with filaments wrapping around vortices—structures that are highly non-Gaussian.

\begin{figure}[H]
    
    \includegraphics[width=0.242\textwidth]{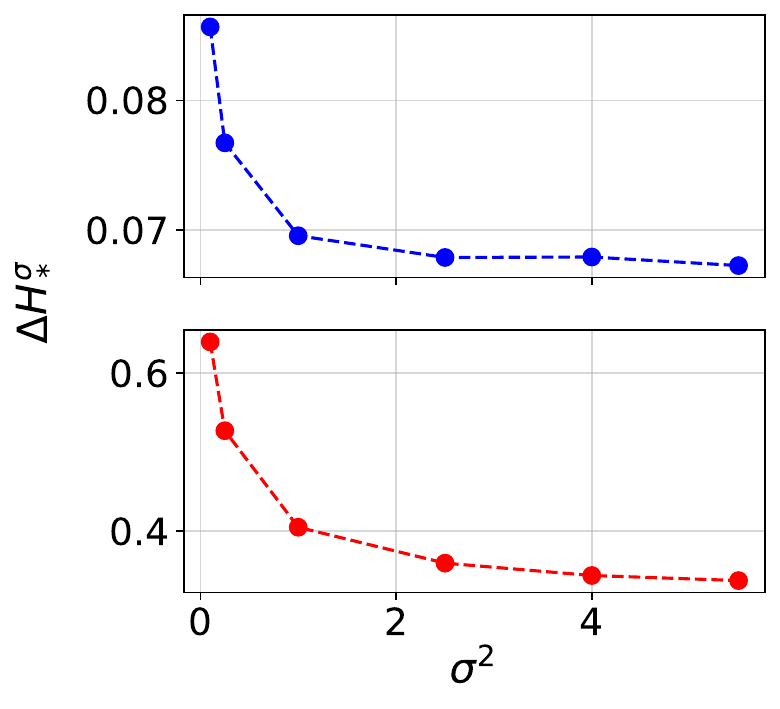}    
    \includegraphics[width=0.223\textwidth]{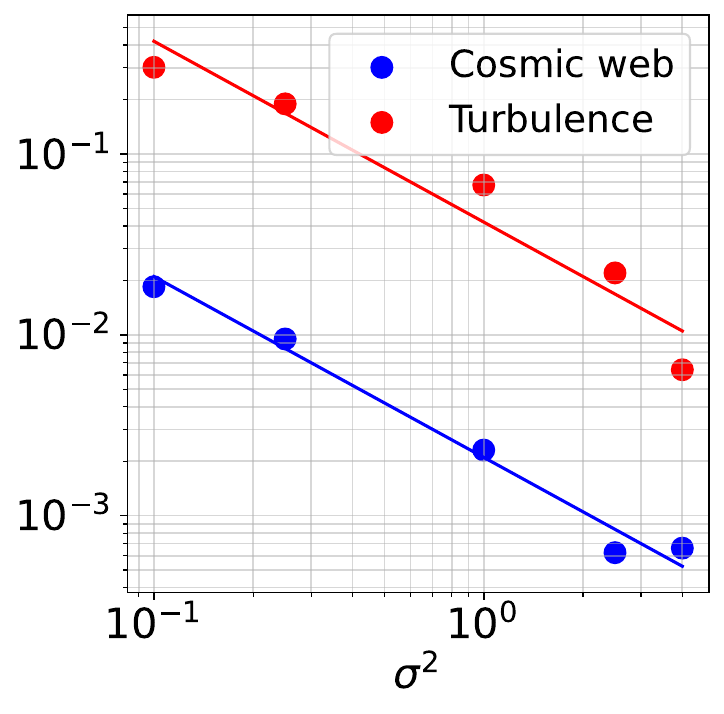}

    ~~~~~~~~~~~~~~~~~~~~~~~~~(a)~~~~~~~~~~~~~~~~~~~~~~~~~~~~~~~~~~~~~~~(b)
    
    \caption{(a) Negentropy estimate $\Delta H_*^\sigma$ from \eqref{neg-entrop-estim} versus $\sigma^2$ for cosmic web (blue) and turbulence (red). (b) Convergence of $H_*^\sigma$, measured by $d^{-1}(H_*^{\sigma_{\max}} - H_*^\sigma)$ with $\sigma_{\max}^2 = 5.5$, versus $\sigma^2$. Plain lines show $\sigma^{-2}$ decay.}
    \label{fig:sampling_2D_entropy}
    
\end{figure}

\begin{table}[H]
\centering
\caption{Negentropy estimate $\Delta H_*^\sigma$ at $\sigma = \sigma_{\max}$.}
\label{table:negentropy}
\begin{tabular}{|c|c|}
\hline
\textbf{Dataset} & \textbf{Estimated Normalized Negentropy} \\
\hline 
Laplacian & 0.07 \\
S\&P 500 & 0.05 \\
Cosmic Web & 0.07 \\
2D Turbulence & 0.34 \\
\hline
\end{tabular}
\end{table}

The effect of an excessively small $\sigma$ is visible in the histograms of fine-scale wavelet coefficients $\mathrm{Re}(X_1^\lambda)$ for $j = 0$, $\ell = 0$, and $X_1 \sim p_1^\sigma$, which exhibit a spike at zero (Figure~\ref{fig:2D_small_D_wav_hist}). As $\sigma$ increases, this artifact disappears and the histogram converges toward that of the original data, even though this marginal distribution is not imposed by the moment map $\phi$.

As with rolling volatility in the one-dimensional setting, increasing $\sigma$ raises the entropy of the sampled process, which translates into increased entropy of wavelet coefficient marginals. This progressively improves the match with the original data, whose entropy lies below that of the maximum entropy distribution.

\begin{figure}[H]
    \centering
    \includegraphics[width=.3\textwidth]{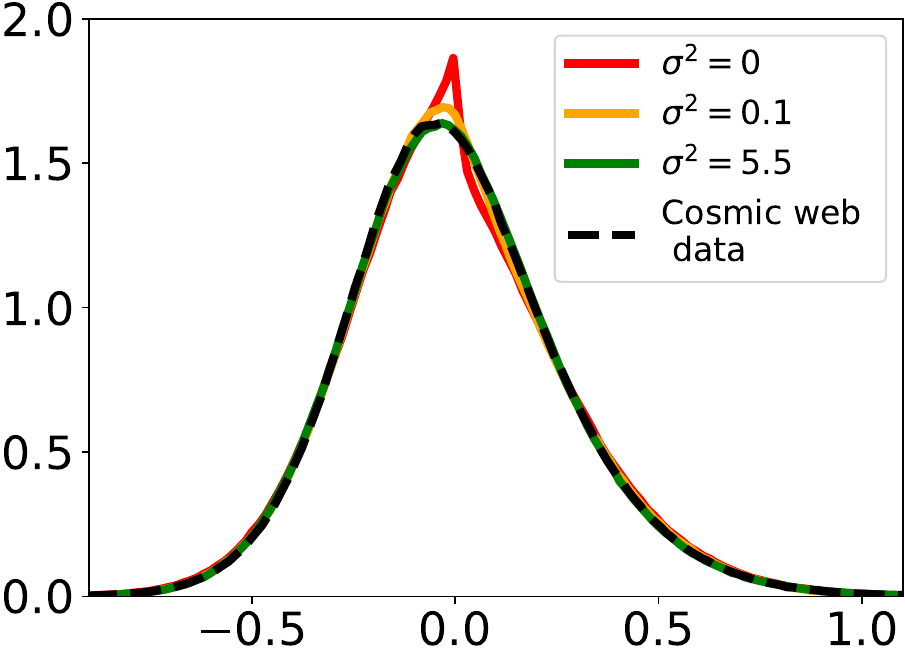}
 
    \caption{Histograms of finest-scale wavelet coefficients $\mathrm{Re}(X_1^\lambda)$ for $\lambda = \xi$ ($j = 0$, $\ell = 0$) of cosmic web samples from the scattering MGD model, with $\sigma^2 \in \{0, 0.1, 5.5\}$. Dashed black: original data. All histograms are computed over $500$ samples. Larger $\sigma^2$ yields better tail reproduction; small $\sigma^2$ produce more regular samples which have too many small wavelet coefficients, as in Figure~\ref{fig:scalar_convergence}(a).}
    \label{fig:2D_small_D_wav_hist}
\end{figure}

\subsection{Convergence with Model Error}
\label{sec:cv_modelisation_error}

The previous experiments consider processes where the maximum entropy model closely approximates the unknown data distribution: $p_* \approx p$. We now consider an example where the model error $D_{\mathrm{KL}}(p \| p_*)$ is large, to verify that MGD can still efficiently sample $p_*$ even when it is a poor approximation of $p$.

\begin{figure}[H]
    \centering

    \hspace{.6cm}
    \includegraphics[width=0.20\textwidth]{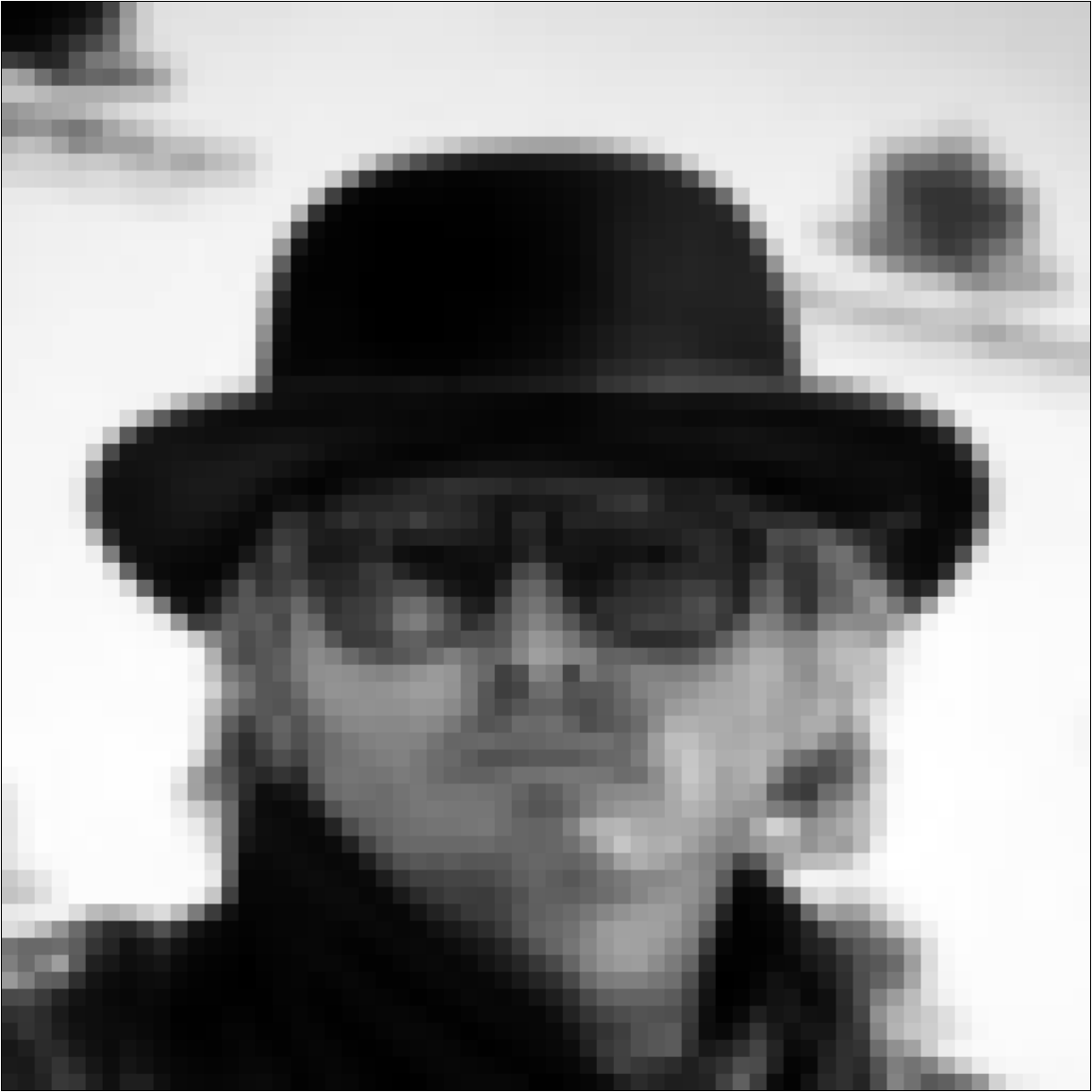}
    \hspace{.4cm}
    \includegraphics[width=0.20\textwidth]{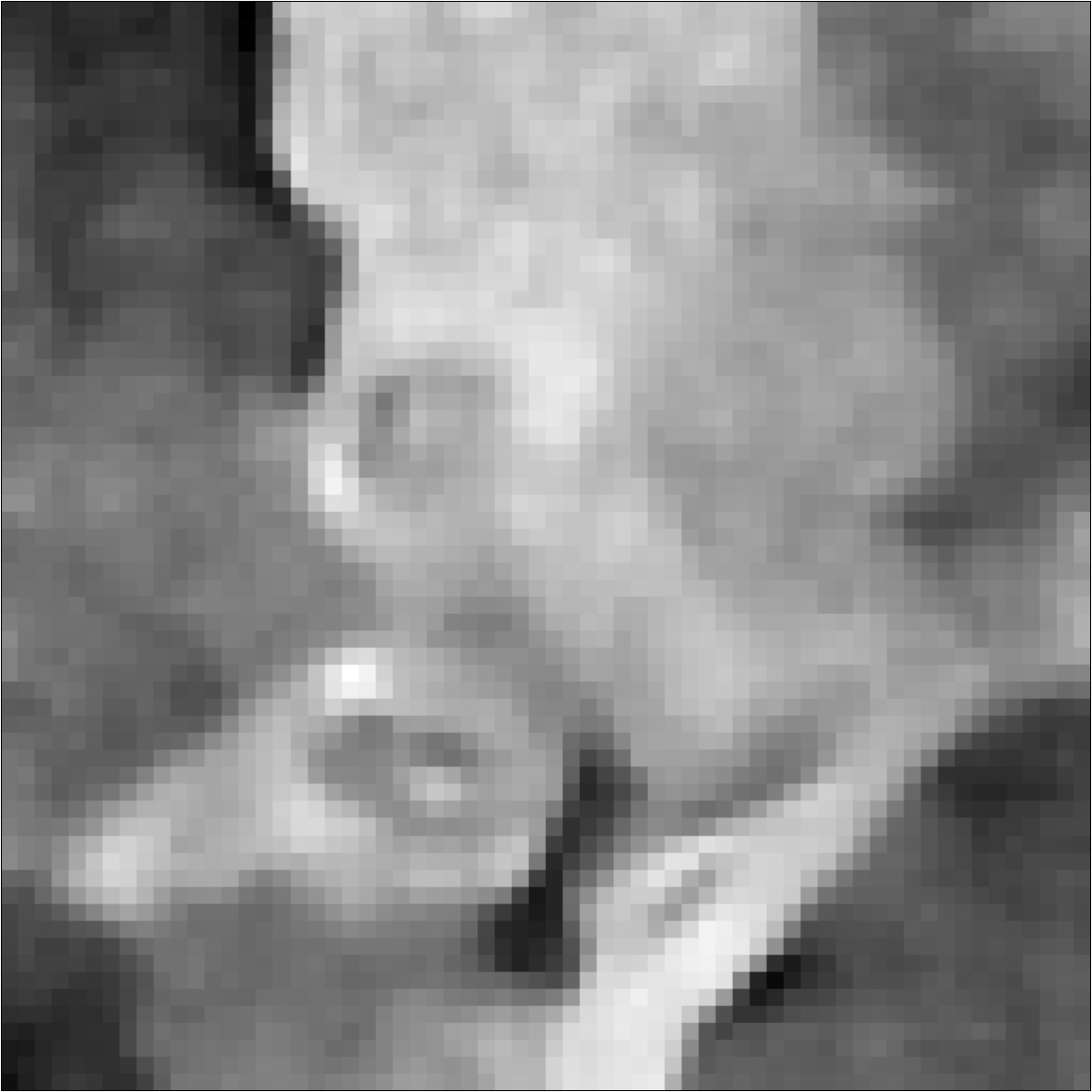}
    
    ~~~~~~~~~~(a)~~~~~~~~~~~~~~~~~~~~~~~~~~~~~~~~~~~~~~~~~(b)~

    \includegraphics[width=0.25\textwidth]{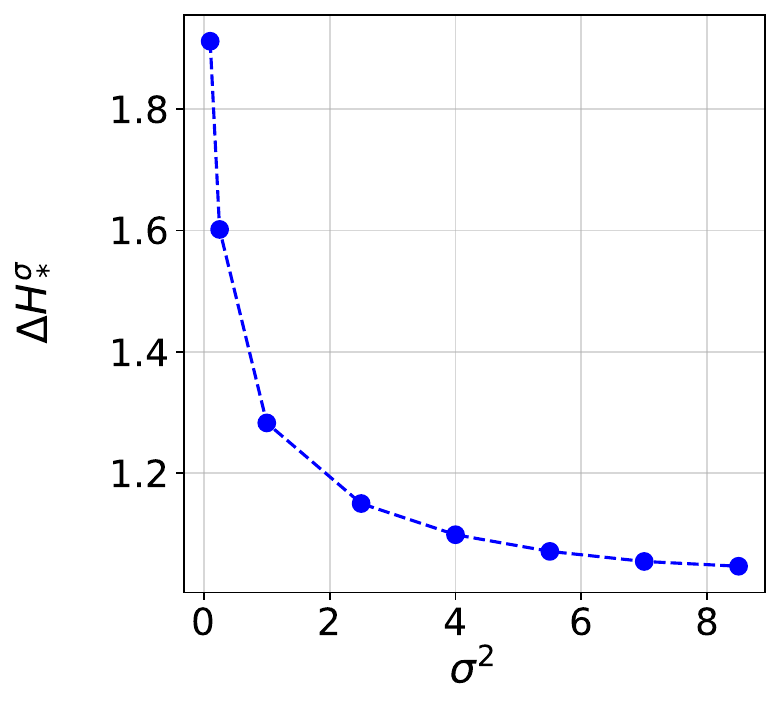}
    \includegraphics[width=0.228\textwidth]{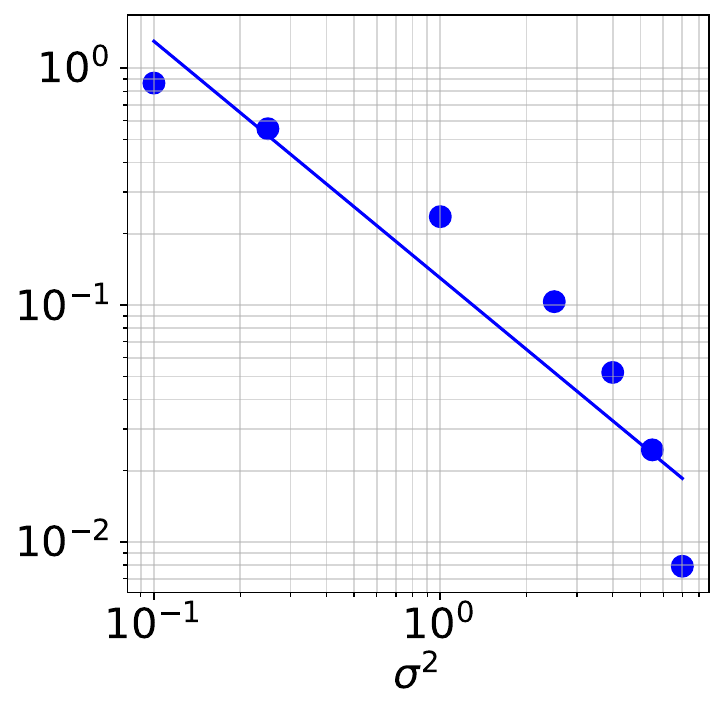}

    ~~~~~~~~~~(c)~~~~~~~~~~~~~~~~~~~~~~~~~~~~~~~~~~~~~~~~~(d)~
    \caption{MGD with large model error on CelebA faces ($64 \times 64$). (a)~Original sample. (b)~MGD sample with wavelet scattering moments at $\sigma^2 = 1$. (c)~Negentropy estimate $\Delta H_*^\sigma$ from~\eqref{neg-entrop-estim} versus $\sigma^2$. (d)~Convergence of $H_*^\sigma$, measured by $d^{-1}(H_*^{\sigma_{\max}} - H_*^\sigma)$ with $\sigma_{\max}^2 = 8.5$; the line shows $\sigma^{-2}$ decay.}
    \label{fig:celeba}
\end{figure}

We choose $p$ as a distribution whose samples are centered human faces from the CelebA dataset~\cite{liu2015faceattributes} (Figure~\ref{fig:celeba}(a)), with a $\phi$ which computes wavelet scattering moments as before. The resulting maximum entropy model $p_*$ is therefore stationary whereas the data distribution is highly non-stationary. Figure~\ref{fig:celeba}(b) shows a sample generated by the MGD. It is a sample of a stationary process, which therefore mixes structures across the whole image space. It reproduces edges and regular regions but destroys the face structure. As expected it has a large model error. Nonetheless, MGD converges quickly to $p_*$.

Figure~\ref{fig:celeba}(c) shows that the negentropy estimate $\Delta H_*^\sigma$ reaches a plateau for $\sigma^2 \approx 6$, and Figure~\ref{fig:celeba}(d) confirms convergence at rate $O(\sigma^{-2})$. The volatility required for convergence is comparable to the physics and finance examples, confirming that for scattering spectra, MGD reaches the maximum entropy distribution for the same range of  $\sigma$, regardless of the model error.

\section{Conclusion}
\label{sec:conclu}

We introduced Moment-Guided Diffusion (MGD), a sampler for maximum entropy distributions estimated from data. Its homotopic path avoids the computational bottleneck of energy barrier crossing that plagues MCMC methods for non-convex distributions. This represents a paradigm shift in maximum entropy modeling: rather than estimating parameters, MGD directly generates samples from the target distribution. A key by-product is a tractable entropy estimator, which we use to compute the negentropy of complex high-dimensional datasets.

We validated MGD on synthetic examples and real-world data, including financial time series, turbulent vorticity fields, and cosmological dark matter distributions. In all cases, the sampled distributions converge to the target maximum entropy distribution as the volatility $\sigma$ increases, with entropy gaps decaying as $O(\sigma^{-2})$ across all tested domains. The negentropy estimates reveal the degree of non-Gaussianity and structure in these datasets, providing a principled measure of statistical complexity.

MGD opens promising avenues in computational physics and biology, where it can replace microcanonical samplers \cite{Cheng2023ScatteringSM,allys:cea-02290738} or be adapted to molecular dynamics with restraints \cite{roux2013statistical}. More broadly, while our formulation uses an explicit moment map $\phi$, the framework naturally accommodates neural network parametrizations, suggesting a principled maximum entropy foundation for diffusion-based generative models.

Several theoretical questions remain open. Although we provide convergence guarantees under specific conditions, a proof of convergence in full generality remains an important challenge. A further question concerns the behavior of MGD when the maximum entropy distribution constrained by moments $m = \mathbb{E}[\phi(X)]$ does not exist. Another important issue is to understand how the convergence of MGD to the maximum entropy distribution depends upon the choice of the moment generating function $\phi$, which needs to be sufficiently flexible. Computing a moment interpolation path $m_t$ directly from $m = \E_p [\phi]$ and $\phi$ is also a promising research direction, to apply the MGD to sample maximum entropy distributions, even if we do not have access to samples of $p$.

\section*{Acknowledgments}
This work was supported by PR[AI]RIE-PSAI-ANR-23-IACL-0008 and the DRUIDS projet ANR-24-EXMA-0002. It was granted access to the HPC resources of IDRIS under the allocations 2025-AD011016159R1 and 2025-A0181016159 made by GENCI.
The authors thank Antonin Chodron de Courcel and Louis-Pierre Chaintron for their fruitful discussions on Mckean-Vlasov equations. {
The authors also thank the reviewers and the editors for their careful and detailed reading, which greatly improved the manuscript.}

\bibliographystyle{IEEEtran}
\bibliography{bibliography}

@misc{mis2026measureflowpathrecovery,
      title={Measure flow path recovery in Bayes Hilbert spaces}, 
      author={S. David Mis and Maarten V. de Hoop},
      year={2026},
      eprint={2603.20329},
      archivePrefix={arXiv},
      primaryClass={stat.ML} 
}

@article{brossollet2025effective,
  title={Effective energy, interactions and out of equilibrium nature of scalar active matter},
  author={Brossollet, Antonin and Lempereur, Etienne and Mallat, St{\'e}phane and Biroli, Giulio},
  journal={Communications Physics},
  year={2025},
  publisher={Nature Publishing Group UK London}
}

@article{boffi2024deep,
  title={Deep learning probability flows and entropy production rates in active matter},
  author={Boffi, Nicholas M and Vanden-Eijnden, Eric},
  journal={Proceedings of the National Academy of Sciences},
  volume={121},
  number={25},
  pages={e2318106121},
  year={2024},
  publisher={National Academy of Sciences}
}

@article{roberts1998optimal,
  title={Optimal scaling of discrete approximations to Langevin diffusions},
  author={Roberts, Gareth O and Rosenthal, Jeffrey S},
  journal={Journal of the Royal Statistical Society: Series B (Statistical Methodology)},
  volume={60},
  number={1},
  pages={255--268},
  year={1998},
  publisher={Wiley Online Library}
}

@book{mallat1999wavelet,
  title={A wavelet tour of signal processing},
  author={Mallat, St{\'e}phane},
  year={1999},
  publisher={Elsevier},
}

@article{mckean1966class,
  title={A class of Markov processes associated with nonlinear parabolic equations},
  author={McKean Jr, Henry P},
  journal={Proceedings of the National Academy of Sciences},
  volume={56},
  number={6},
  pages={1907--1911},
  year={1966}
}

@article{chaintron2022propagation,
  title={Propagation of chaos: a review of models, methods and applications. I. Models and methods},
  author={Chaintron, Louis-Pierre and Diez, Antoine},
  journal={arXiv preprint arXiv:2203.00446},
  year={2022}
}

@article{chewi2025analysis,
  title={Analysis of langevin monte carlo from poincare to log-sobolev},
  author={Chewi, Sinho and Erdogdu, Murat A and Li, Mufan and Shen, Ruoqi and Zhang, Matthew S},
  journal={Foundations of Computational Mathematics},
  volume={25},
  number={4},
  pages={1345--1395},
  year={2025},
  publisher={Springer}
}

@inproceedings{
li2022sqrtd,
title={Sqrt(d) Dimension Dependence of Langevin Monte Carlo},
author={Ruilin Li and Hongyuan Zha and Molei Tao},
booktitle={International Conference on Learning Representations},
year={2022}
}

@article{higham2001algorithmic,
  title={An algorithmic introduction to numerical simulation of stochastic differential equations},
  author={Higham, Desmond J},
  journal={SIAM review},
  volume={43},
  number={3},
  pages={525--546},
  year={2001},
  publisher={SIAM}
}

@article{kloeden1977numerical,
  title={The numerical solution of stochastic differential equations},
  author={Kloeden, Peter E and Pearson, RA},
  journal={The ANZIAM Journal},
  volume={20},
  number={1},
  pages={8--12},
  year={1977},
  publisher={Cambridge University Press}
}

@article{besag1994comments,
  title={Comments on “Representations of knowledge in complex systems” by U. Grenander and MI Miller},
  author={Besag, Julian},
  journal={J. Roy. Statist. Soc. Ser. B},
  volume={56},
  number={591-592},
  pages={4},
  year={1994}
}

@book{robert1999monte,
  title={Monte Carlo statistical methods},
  author={Robert, Christian P and Casella, George and Casella, George},
  volume={2},
  year={1999},
  publisher={Springer}
}

@article{schneider2006coherent,
  title={Coherent vortex extraction and simulation of 2D isotropic turbulence},
  author={Schneider, Kai and Ziuber, J{\"o}rg and Farge, Marie and Azzalini, Alexandre},
  journal={Journal of Turbulence},
  number={7},
  pages={N44},
  year={2006},
  publisher={Taylor \& Francis}
}

@article{villaescusa2020quijote,
  title={The quijote simulations},
  author={Villaescusa-Navarro, Francisco and Hahn, ChangHoon and Massara, Elena and Banerjee, Arka and Delgado, Ana Maria and Ramanah, Doogesh Kodi and Charnock, Tom and Giusarma, Elena and Li, Yin and Allys, Erwan and others},
  journal={The Astrophysical Journal Supplement Series},
  volume={250},
  number={1},
  pages={2},
  year={2020},
  publisher={IOP Publishing}
}

@book{schrodinger1944life,
  title={What is Life? The Physical Aspect of the Living Cell},
  author={Schr{\"o}dinger, Erwin},
  year={1944},
  publisher={Cambridge University Press}
}

@inproceedings{
kadkhodaie2024generalization,
title={Generalization in diffusion models arises from geometry-adaptive harmonic representations},
author={Zahra Kadkhodaie and Florentin Guth and Eero P Simoncelli and St{\'e}phane Mallat},
booktitle={The Twelfth International Conference on Learning Representations},
year={2024}
}

@inproceedings{
bonnaire2025why,
title={Why Diffusion Models Don{\textquoteright}t Memorize:  The Role of Implicit Dynamical Regularization in Training},
author={Tony Bonnaire and Rapha{\"e}l Urfin and Giulio Biroli and Marc Mezard},
booktitle={The Thirty-ninth Annual Conference on Neural Information Processing Systems},
year={2025}
}

@inproceedings{
brochard2022generalized,
title={Generalized rectifier wavelet covariance models for texture synthesis},
author={Antoine Brochard and Sixin Zhang},
booktitle={International Conference on Learning Representations},
year={2022}
}

@book{zinn2021quantum,
  title={Quantum field theory and critical phenomena},
  author={Zinn-Justin, Jean},
  volume={171},
  year={2021},
  publisher={Oxford university press}
}

@article{sokal1991beat,
  title={How to beat critical slowing-down: 1990 update},
  author={Sokal, Alan D},
  journal={Nuclear Physics B-Proceedings Supplements},
  volume={20},
  pages={55--67},
  year={1991},
  publisher={Elsevier}
}

@book{kullback1997information,
  title={Information theory and statistics},
  author={Kullback, Solomon},
  year={1997},
  publisher={Courier Corporation}
}

@book{cover1999elements,
  title={Elements of information theory},
  author={Cover, Thomas M},
  year={1999},
  publisher={John Wiley \& Sons}
}

@book{bishop2006pattern,
  title={Pattern recognition and machine learning by Christopher M. Bishop},
  author={Bishop, Christopher M},
  volume={400},
  year={2006},
  publisher={Springer Science+ Business Media, LLC Berlin, Germany:}
}

@article{goodfellow2020generative,
  title={Generative adversarial networks},
  author={Goodfellow, Ian and Pouget-Abadie, Jean and Mirza, Mehdi and Xu, Bing and Warde-Farley, David and Ozair, Sherjil and Courville, Aaron and Bengio, Yoshua},
  journal={Communications of the ACM},
  volume={63},
  number={11},
  pages={139--144},
  year={2020},
  publisher={ACM New York, NY, USA}
}

@article{lai2025principles,
  title={The principles of diffusion models},
  author={Lai, Chieh-Hsin and Song, Yang and Kim, Dongjun and Mitsufuji, Yuki and Ermon, Stefano},
  journal={arXiv preprint arXiv:2510.21890},
  year={2025}
}

@article{jaynes1957information,
  title={Information theory and statistical mechanics},
  author={Jaynes, Edwin T},
  journal={Physical review},
  volume={106},
  number={4},
  pages={620},
  year={1957},
  publisher={APS}
}

@article{mallat2012group,
  title={Group invariant scattering},
  author={Mallat, St{\'e}phane},
  journal={Communications on Pure and Applied Mathematics},
  volume={65},
  number={10},
  pages={1331--1398},
  year={2012},
  publisher={Wiley Online Library}
}

@article{koehler2022statistical,
  title={Statistical Efficiency of Score Matching: The View from Isoperimetry},
  author={Koehler, Frederic and Heckett, Alexander and Risteski, Andrej},
  journal={arXiv preprint arXiv:2210.00726},
  year={2022}
}

@article{hyvarinen2005estimation,
  title={Estimation of non-normalized statistical models by score matching.},
  author={Hyv{\"a}rinen, Aapo and Dayan, Peter},
  journal={Journal of Machine Learning Research},
  volume={6},
  number={4},
  year={2005}
}

@misc{albergo2022stochastic,
      title={Building Normalizing Flows with Stochastic Interpolants}, 
      author={Michael S. Albergo and Eric Vanden-Eijnden},
      year={2022},
      eprint={2209.15571},
      archivePrefix={arXiv},
      primaryClass={cs.LG}
}

@misc{lipman2022flow,
      title={Flow Matching for Generative Modeling}, 
      author={Yaron Lipman and Ricky T. Q. Chen and Heli Ben-Hamu and Maximilian Nickel and Matt Le},
      year={2023},
      eprint={2210.02747},
      archivePrefix={arXiv},
      primaryClass={cs.LG}
}

@misc{liu2022flow,
      title={Flow Straight and Fast: Learning to Generate and Transfer Data with Rectified Flow}, 
      author={Xingchao Liu and Chengyue Gong and Qiang Liu},
      year={2022},
      eprint={2209.03003},
      archivePrefix={arXiv},
      primaryClass={cs.LG} 
}

@inproceedings{
  song2021scorebased,
  title={Score-Based Generative Modeling through Stochastic Differential Equations},
  author={Yang Song and Jascha Sohl-Dickstein and Diederik P Kingma and Abhishek Kumar and Stefano Ermon and Ben Poole},
  booktitle={International Conference on Learning Representations},
  year={2021}
}

@article{ho2020denoising,
  title={Denoising diffusion probabilistic models},
  author={Ho, Jonathan and Jain, Ajay and Abbeel, Pieter},
  journal={Advances in neural information processing systems},
  volume={33},
  pages={6840--6851},
  year={2020}
}

@article{albergo2023stochastic,
  title={Stochastic interpolants: A unifying framework for flows and diffusions},
  author={Albergo, Michael S and Boffi, Nicholas M and Vanden-Eijnden, Eric},
  journal={arXiv preprint arXiv:2303.08797},
  year={2023}
}

@article{Cheng2023ScatteringSM,
  title={Scattering spectra models for physics},
  author={Sihao Cheng and Rudy Morel and Erwan Allys and Brice M'enard and St{\'e}phane Mallat},
  journal={PNAS Nexus},
  year={2023},
  volume={3}
}

@article{Morel2022ScaleDA,
  title={Scale Dependencies and Self-Similarity Through Wavelet Scattering Covariance},
  author={Rudy Morel and Gaspar Rochette and Roberto F. Leonarduzzi and Jean-Philippe Bouchaud and St{\'e}phane Mallat},
  journal={ArXiv},
  year={2022},
  volume={abs/2204.10177}
}

@article{bruna2019multiscale,
  title={Multiscale sparse microcanonical models},
  author={Bruna, Joan and Mallat, St{\'e}phane},
  journal={Mathematical Statistics and Learning},
  volume={1},
  number={3},
  pages={257--315},
  year={2019}
}

@article{roux2013statistical,
  title={On the statistical equivalence of restrained-ensemble simulations with the maximum entropy method},
  author={Roux, Beno{\^\i}t and Weare, Jonathan},
  journal={The Journal of chemical physics},
  volume={138},
  number={8},
  year={2013},
  publisher={AIP Publishing}
}

@article{SCHNEIDER01012006,
author = {Kai Schneider and Jörg Ziuber and Marie Farge and Alexandre Azzalini},
title = {Coherent vortex extraction and simulation of 2D isotropic turbulence},
journal = {Journal of Turbulence},
volume = {7},
number = {},
pages = {N44},
year = {2006},
publisher = {Taylor \& Francis},
doi = {10.1080/14685240600601061},
eprint = {https://doi.org/10.1080/14685240600601061}

}

@article{allys:cea-02290738,
  TITLE = {{The RWST, a comprehensive statistical description of the non-Gaussian structures in the ISM}},
  AUTHOR = {Allys, Erwan and Boulanger, Francois and Levrier, Fran{\c c}ois and Zhang, Sixin and Colling, C. and Regaldo-Saint Blancard, B. and Hennebelle, P. and Mallat, S.},
  JOURNAL = {{Astronomy \& Astrophysics - A\&A}},
  PUBLISHER = {{EDP Sciences}},
  VOLUME = {629},
  PAGES = {A115},
  YEAR = {2019},
  DOI = {10.1051/0004-6361/201834975},
  HAL_ID = {cea-02290738},
  HAL_VERSION = {v1},
}

@article{Welch1967TheUO,
  title={The use of fast Fourier transform for the estimation of power spectra: A method based on time averaging over short, modified periodograms},
  author={Peter D. Welch},
  journal={IEEE Transactions on Audio and Electroacoustics},
  year={1967},
  volume={15},
  pages={70-73}
}

@misc{favero2025biggerisntmemorizingearly,
      title={Bigger Isn't Always Memorizing: Early Stopping Overparameterized Diffusion Models}, 
      author={Alessandro Favero and Antonio Sclocchi and Matthieu Wyart},
      year={2025},
      eprint={2505.16959},
      archivePrefix={arXiv},
      primaryClass={cs.LG}
}

@inproceedings{liu2015faceattributes,
  title = {Deep Learning Face Attributes in the Wild},
  author = {Liu, Ziwei and Luo, Ping and Wang, Xiaogang and Tang, Xiaoou},
  booktitle = {Proceedings of International Conference on Computer Vision (ICCV)},
  month = {December},
  year = {2015} 
}

\appendices

\section{Proofs}
\label{app:proof}

In this appendix, we prove Theorem~\ref{prop:moments}, {Proposition~\ref{prop:minimal-rate}}, Proposition~\ref{prop:entropy}, and the additional Proposition~\ref{proposition:dentropydD_app}, which shows that the entropy increases along MGD's dynamic when the moments are fixed ($dm_t/dt=0$). For the reader's convenience we state again the results from main text. 

\propmoments*

\begin{proof}
By Ito's lemma 
    \begin{align*}
        d\phi(X_t) =& \nabla\phi(X_t) dX_t +\sigma^2\Delta\phi(X_t)dt \\
        =& \nabla\phi(X_t)\cdot\nabla \phi(X_t)^\top (\eta_{t}-\sigma^2\theta_{t})dt \\
        &+\sigma^2\Delta\phi(X_t)dt +\sqrt {2}\sigma\nabla\phi(X_t)\cdot dW_t .
    \end{align*}
Taking the expected value of this equation we obtain that
\begin{align*}
        \frac{d }{dt}\mathbb{E}[\phi(X_t)]=& \mathbb{E}[\nabla\phi(X_t)\cdot\nabla \phi(X_t)^\top ](\eta_{t}-\sigma^2\theta_{t}) \\
        &+ \sigma^2\mathbb{E}[\Delta\phi(X_t)].
\end{align*}
Since we require that $\E[\phi(X_t)] = \E[\phi(I_t)] = m_t $ for all $t\in[0,1]$, we must also have 
\begin{align*}
    \frac{d }{dt}\mathbb{E}[\phi(X_t)] =& \frac{d}{dt} m_t .
\end{align*}
Combining these two last equations we deduce that 
\begin{align*}
G_t(\eta_{t}-\sigma^2\theta_{t}) + \sigma^2\mathbb{E}[\Delta\phi(X_t)] =\frac{d }{dt}m_t~,
\end{align*}
where $G_t =\mathbb{E}[\nabla \phi(X_t) \cdot \nabla \phi(X_t)^\top ]$.
{ This is how the system~\eqref{eq:eta}--\eqref{eq:theta} is
obtained. Conversely,} this equation is satisfied since $\eta_t$ and $\theta_t$ are solutions to \eqref{eq:eta}
and \eqref{eq:theta}, respectively.  Therefore $\frac{d }{dt}\mathbb{E}[\phi(X_t)] =  \frac{d }{dt}m_t$ which implies that 
\[\forall t\in[0,1] \, : \quad \mathbb{E}[\phi(X_t)] = m_t\]
since $\mathbb{E}[\phi(X_0)] = m_0$.
\end{proof}

{
\bigskip
\propminimalrate*
\bigskip

\begin{proof}
Fix $t \in [0,1]$. By It\^o's lemma applied to~\eqref{eq:uprocess}, and
proceeding as in the proof of Theorem~\ref{prop:moments},
\begin{equation*}
    \frac{d}{dt}\mathbb{E}\big[\phi(X_t)\big]
    = \mathbb{E}\big[\nabla\phi(X_t)\cdot u_t(X_t)\big]
    + \sigma^2\,\mathbb{E}\big[\Delta\phi(X_t)\big],
\end{equation*}
so the drift $u_t$ is admissible if and only if
\begin{equation}
\label{eq:admissible}
    \mathbb{E}\big[\nabla \phi(X_t)\cdot u_t(X_t)\big]
    = \frac{d}{dt}m_t - \sigma^2\,\mathbb{E}\big[\Delta \phi(X_t)\big].
\end{equation}
A drift with $\mathbb{E}[|u_t(X_t)|^2]$ infinite cannot be minimal, so we work
in $L^2(p_t^\sigma;\mathbb{R}^d)$ and let
$S = \mathrm{span}\{\nabla\phi_1,\dots,\nabla\phi_r\}$, finite dimensional and
hence closed.

Decompose an admissible $u_t = v + w$ with $v \in S$ and $w \perp S$. The left
side of~\eqref{eq:admissible} depends only on $v$, while
$\|u_t\|^2 = \|v\|^2 + \|w\|^2$, so any minimizer has $w = 0$ and lies in $S$.
Writing $u_t = \lambda_t^\top\nabla\phi$ and substituting
into~\eqref{eq:admissible} gives
$G_t\lambda_t = \tfrac{d}{dt}m_t - \sigma^2\,\mathbb{E}[\Delta\phi(X_t)]$, whose
unique solution is $\lambda_t = \eta_t - \sigma^2\theta_t$
by~\eqref{eq:eta}--\eqref{eq:theta}. The minimizer is therefore unique and given
by~\eqref{eq:mgd_drift}, the drift of~\eqref{eq:sde}.
\end{proof}
}

\bigskip
\propentropy*
\bigskip

\begin{proof}[Proof:]
The PDF $p^\sigma_{t}$ of the solution to the MGD SDE~\eqref{eq:sde} for a given $\sigma\ge0$ obeys the Fokker-Planck Equation 
\begin{align*}
    \partial_t p_{t}^{\sigma} =&\nabla((-\eta_{t}+\sigma^2\theta_{t})^\top \nabla\phi p^\sigma_{t}) +\sigma^2\Delta p_{t}^{\sigma}.
\end{align*}
{ Thanks to integrations by part, we can use this equation to express the time derivative of $H(p_t^\sigma)$ with respect to multipliers $\eta_t$ and $\theta_t$.}

\begin{align*}
\frac{d }{d t}H(p_t^\sigma) =& -\int \partial_tp_t^{\sigma}\log p_t^{\sigma}dx  -\int \partial_tp_t^{\sigma}dx \\ 
=& -\int\nabla\cdot ((-\eta_{t}+\sigma^2\theta_{t})^\top \nabla\phi p^\sigma_{t})\log p_t^{\sigma} dx  \\
&-\sigma^2\int\Delta p_t\log p_t^{\sigma}dx \\
=& \int(-\eta_t+\sigma^2\theta_t)^\top\nabla\phi \cdot \nabla p_t^{\sigma}dx
\\
&+\sigma^2\int \nabla \log p_t^{\sigma}\cdot \nabla p^\sigma_tdx \\
=& \int\Big((\eta_t-\sigma^2\theta_t)^\top\Delta\phi +\sigma^2|\nabla \log p_t^{\sigma}|^2\Big) p^\sigma_tdx~,
\end{align*}
where we used a few integration by parts and the identity $p^\sigma_t \nabla \log p^\sigma_t = \nabla p^\sigma_t$. Writing the integral in the last equation as an expectation, we deduce that
\begin{align*}
\frac{d }{d t}H(p_t^\sigma) =& (\eta_t-\sigma^2\theta_t)^\top \mathbb{E}[\Delta\phi(X_t)] \\
&+\sigma^2\mathbb{E}\big[|\nabla \log p_t^{\sigma}(X_t)|^2\big].
\end{align*}
Using  $G_t \theta_t =\mathbb{E}[\Delta \phi(X_t)]$ from~\eqref{eq:theta}, we obtain that
\begin{align*}
    \mathbb{E}\big[|\nabla \log p_t^{\sigma}(X_t)|^2\big]=& \mathbb{E}\big[|\nabla \log p_t^\sigma(X_t)+\theta_t^\top \nabla \phi(X_t)|^2\big]\\
    &-2\mathbb{E}\big[\nabla\log p_t^\sigma(X_t) \cdot \theta_t^\top \nabla \phi(X_t)\big]\\
      &-\mathbb{E}\big[|\theta_t^\top \nabla \phi(X_t)|^2\big] \\
      =& \mathbb{E}\big[|\nabla \log p_t^\sigma(X_t)+\theta_t^\top \nabla \phi(X_t)|^2\big]\\
    &+2\theta_t^\top\mathbb{E}\big[\Delta \phi(X_t)\big]\\
      &-\theta_t^\top\mathbb{E}\big[\nabla \phi(X_t)\cdot\nabla \phi(X_t)^\top\big]\\
      =& \mathbb{E}\big[|\nabla \log p_t^\sigma(X_t)+\theta_t^\top \nabla \phi(X_t)|^2\big]\\
    &+2\theta_t^\top\mathbb{E}\big[\Delta \phi(X_t)\big]-\theta_t^\top G_t \theta_t\\
    =& \mathbb{E}\big[|\nabla \log p_t^\sigma(X_t)+\theta_t^\top \nabla \phi(X_t)|^2\big]\\
    &+\theta_t^\top\mathbb{E}\big[\Delta \phi(X_t)\big] ~.
\end{align*}
Combining these last two equations, we deduce that 
\begin{align*}
    \frac{d}{d t} H(p_t^\sigma) & = \eta_t^\top\mathbb{E}[\Delta\phi(X_t)]+\sigma^2\mathbb{E}\big[|\nabla \log p_t^{\sigma}(X_t)|^2\big]\\
    &= \eta_t^\top G_t\theta_t+\sigma^2\mathbb{E}\big[|\nabla \log p_t^{\sigma}(X_t)|^2\big]\\   
    &\geq  \eta_t^\top G_t\theta_t ~.
\end{align*}
Finally, since $G_t\eta_t = dm_t/dt$ from~\eqref{eq:eta}, we arrive at
\begin{align*}
    \frac{d}{d t} H(p_t^\sigma)\ge \theta_t^\top \frac{d}{dt}m_t.
\end{align*}
\end{proof}

This proposition shows that $\frac{d}{d t} H(p_t^\sigma)\ge 0$ if $dm_t/dt=0$ (i.e. if the moments are preserved). Our next proposition shows that in that setup the entropy also increases as a function of the volatility $\sigma$ at any given time $t$.

\begin{proposition}
\label{proposition:dentropydD_app}
Let $p_t^\sigma$ be the PDF of the solution to the MGD SDE~\eqref{eq:sde}. If we assume that $dm_t/dt=0$, then at any time $t\in(0,1]$, we have
\begin{equation}
    \label{eq:dH/dt_app}
         \frac{d}{d\sigma}H(p_t^\sigma) \geq 0 .
    \end{equation}  
\end{proposition}

\begin{proof}
Since $dm_t/dt=0$, Proposition \ref{prop:entropy} implies that, for all $\sigma>0$,
\begin{align*}
    \frac{d}{d t}H(p_t^\sigma)\geq 0.
\end{align*}
Because $\eta_t =0$ when $dm_t/dt=0$ by~\eqref{eq:eta}, in this setup the Fokker-Planck equation for $p^\sigma_t$ reduces to 
\begin{align*}
    \partial_t p_{t}^{\sigma} =\sigma^2\nabla(\theta_{t}^\top \nabla\phi p^\sigma_{t}) +\sigma^2\Delta p_{t}^{\sigma}.
\end{align*}
By rescaling time as $\tau =t\sigma^2$, we see from this equation that 
\begin{equation*}
    p_t^\sigma = p_\tau^{\sigma=1}.
\end{equation*}
Therefore 
\[
\frac{d}{d\sigma}H(p_t^\sigma) = \frac{2\tau}{\sigma}\frac{d}{d\tau} H(p_\tau^{\sigma=1}),
\]
and hence
\begin{equation*}
     \frac{d}{d \sigma}H(p_t^\sigma)  \geq 0 .
\end{equation*}
\end{proof}

{\section{Global variational problem}
\label{app:global}
\label{app:global-system}

In this appendix, we review some aspects of the global minimization
problem~\eqref{eq:global-problem} of Remark~\ref{rem:global-variational}. We
derive its optimality system and question whether MGD is the solution to this
minimization problem.

Let $X_t^u$ with density $p^u_t$ follow over $[0,1]$ the SDE
\begin{equation*}
    dX_t^u  = u_t(X_t^u)\,dt+\sqrt{2}\,\sigma\, dW_t .
\end{equation*}
If $\mathbb{P}^u$ is the path law of $X_t^u$, then the relative entropy of
$\mathbb{P}^u$ with respect to $\mathbb{P}_{\rm ref}$ reads
\begin{equation*}
    D_{\rm KL}\big(\mathbb{P}^u \,\|\, \mathbb{P}_{\rm ref}\big)
    = \tfrac{1}{4\sigma^2}\int_0^1\!\!\int
    |u_t|^2\, p^u_t \,dx\,dt~.
\end{equation*}
The global minimization problem~\eqref{eq:global-problem} minimizes this
integral with respect to the drift $u_t$, under the moment constraints
\begin{equation*}
\label{eq:moment-u}
    \int \phi\, p^u_t \,dx = m_t, \qquad t \in [0,1].
\end{equation*}
The loss is quadratic in $u_t$, but the expectations are taken over $p^u_t$,
which is itself determined by $u_t$ through the Fokker--Planck equation
\begin{equation*}
    \partial_t p^u_t + \nabla\cdot(p^u_t u_t) = \sigma^2\Delta p^u_t.
\end{equation*}
The problem is therefore treated as a joint minimization over $(p^u_t, u_t)$,
coupled by this equation. Differentiating the moment constraints in $t$ and
using it, they read as the linear conditions
\begin{equation*}
\label{eq:instant-constraint-global}
    \int \big(u_t\cdot\nabla\phi + \sigma^2\Delta\phi\big)\, p^u_t \,dx
    = \tfrac{d}{dt} m_t .
\end{equation*}
In the pair $(p^u_t, p^u_t u_t)$, the loss is quadratic and both the moment
conditions and the Fokker--Planck equation are linear. We introduce a multiplier
for the latter, written as $\mu_t/2\sigma^2$ for convenience, and
$\lambda_t \in \mathbb{R}^r$ for the former. This gives the Lagrangian to
minimize in $(p^u_t, u_t)$
\begin{equation*}
\label{eq:lagrangian-global}
\begin{aligned}
    \mathcal{L} = \int_0^1\!\!\int \Big[
    &\tfrac{1}{4\sigma^2}|u_t|^2 p^u_t \\
    &- \tfrac{1}{2\sigma^2}\mu_t\big(\partial_t p^u_t
    + \nabla\cdot(p^u_t u_t) - \sigma^2\Delta p^u_t\big) \\
    &- \lambda_t^\top\big(u_t\cdot\nabla\phi
    + \sigma^2\Delta\phi\big) p^u_t \Big] dx\,dt \\
    &+ \int_0^1 \lambda_t^\top \tfrac{d}{dt}m_t\,dt .
\end{aligned}
\end{equation*}

The optimal drift is a gradient: setting $\delta\mathcal{L}/\delta u_t = 0$ and
integrating the divergence term by parts gives
\begin{equation*}
    u_t = \nabla\mu_t,
\end{equation*}
which is where the normalization above is convenient.

The potential $\mu_t$ solves an optimality system: setting
$\delta\mathcal{L}/\delta p^u_t = 0$ gives an equation for $\mu_t$, which the
Fokker--Planck equation closes once the drift is substituted into it,
\begin{align}
\label{eq:opt-fp}
\partial_t p^u_t + \nabla\cdot(p^u_t \nabla \mu_t)
&= \sigma^2 \Delta p^u_t,\\
\label{eq:opt-hjb}
\partial_t \mu_t + \tfrac12 |\nabla \mu_t|^2 + \sigma^2 \Delta \mu_t
&= \lambda_t^\top\phi .
\end{align}
The first equation runs forward in time, the second backward. The optimal drift
at time $t$ therefore depends on the whole trajectory.

MGD is this system restricted to a finite-dimensional family: its drift is
$\nabla\mu_t$ for $\mu_t = \zeta_t^\top\phi + c_t$, with
$\zeta_t = \eta_t - \sigma^2\theta_t$. Assume that the global minimization
admits a solution of this form for every moment trajectory $m_t$, so that MGD
solves it. The next proposition shows that this requires $\phi$ to follow a
Hamilton--Jacobi equation, which couples $\phi$ and its gradients over
$\mathbb{R}^d$.

\begin{proposition}[Exactness of the reduction]
\label{prop:closure}
Assume $\phi$ is twice differentiable, and $G_0 = \int
\nabla\phi\cdot\nabla\phi^\top p_0\,dx$ invertible. If the optimality
system~\eqref{eq:opt-fp}--\eqref{eq:opt-hjb} admits a solution of the form
$\mu_t = \zeta_t^\top\phi + c_t$ for every differentiable moment trajectory $m_t$, then $X_t^u$ is the solution of MGD SDE \eqref{eq:sde} and
\begin{equation}
\label{eq:closure}
    \nabla\phi_i\cdot\nabla\phi_j \in \mathrm{span}\{\phi,1\}
    ~\text{and}~
    \Delta\phi_k \in \mathrm{span}\{\phi,1\}
\end{equation}
for all $i,j,k$.
\end{proposition}

\begin{proof}
Substituting $\mu_t = \zeta_t^\top\phi + c_t$ into~\eqref{eq:opt-hjb} gives the
Hamilton--Jacobi equation
\begin{equation*}
\label{eq:substituted}
    \dot\zeta_t^\top\phi + \dot c_t
    + \tfrac12\big|\zeta_t^\top\nabla\phi\big|^2
    + \sigma^2\,\zeta_t^\top\Delta\phi = \lambda_t^\top\phi .
\end{equation*}
This is an identity in $x$, and all terms but the two middle ones lie in
$\mathrm{span}\{\phi,1\}$. Hence
\begin{equation}
\label{eq:closure-zeta}
    \tfrac12\big|\zeta^\top\nabla\phi\big|^2 + \sigma^2\,\zeta^\top\Delta\phi
    \in \mathrm{span}\{\phi,1\}
\end{equation}
for every value $\zeta$ taken by $\zeta_t$ as $m_t$ varies.

The drift being $\zeta_t^\top\nabla\phi$, differentiating the moment constraint
$\int \phi\, p^u_t\,dx = m_t$ along~\eqref{eq:opt-fp} gives
\begin{equation}
\label{eq:zeta-relation}
    G_t\,\zeta_t = \tfrac{d}{dt} m_t - \sigma^2\int p_t^u\Delta \phi ,
    ~
    G_t = \int \nabla\phi\cdot\nabla\phi^\top p^u_t\,dx .
\end{equation}

Take $t = 0$, where $p_0$ is the law of $Z$, so that $G_0$ and
$\int \Delta\phi\, p_0$ are fixed while $\tfrac{d}{dt}m_0$ may be prescribed
arbitrarily. Since $G_0$ is invertible, \eqref{eq:zeta-relation} makes
$\tfrac{d}{dt}m_0 \mapsto \zeta_0$ an affine bijection of $\mathbb{R}^r$, so
\eqref{eq:closure-zeta} holds for every $\zeta \in \mathbb{R}^r$. This concludes the proof.
\end{proof}

Proposition~\ref{prop:closure} assumes that a solution of the reduced form
exists for every moment trajectory, which is not guaranteed.

Condition~\eqref{eq:closure} constrains $\phi$ severely. We do not characterize
the moment functions satisfying it in general, but the next proposition settles
the polynomial case.

\begin{proposition}[Polynomial moment functions]
\label{prop:degree}
Let the components of $\phi$ be polynomials of degree at most $n$, and assume
that some $\nabla\phi_i\cdot\nabla\phi_j$ has degree $2n-2$. Then
\eqref{eq:closure} fails for $n \geq 3$.
\end{proposition}

\begin{proof}
The elements of $\mathrm{span}\{\phi,1\}$ have degree at most $n$, and
$2n-2 > n$ for $n \geq 3$.
\end{proof}

Quadratic moment functions satisfy condition~\eqref{eq:closure}, and are the
case where both problems can be solved explicitly. Theorem~\ref{th:gaussian}
gives MGD in closed form, as the unique solution of~\eqref{eq:sde:quad}, and the
next proposition shows that this process also solves the global problem.

\begin{proposition}[Global optimality for quadratic $\phi$]
\label{prop:quad-global}
Let $\phi(x) = (x_i x_j)_{1\leq i \leq j \leq d}$, and assume that $C_t$ and
$\dot C_t$ commute. Then the global problem~\eqref{eq:global-problem} admits a
unique solution, namely the law of the MGD SDE~\eqref{eq:sde:quad}: the pair
$(p_t,\mu_t)$ solves~\eqref{eq:opt-fp}--\eqref{eq:opt-hjb} with
\begin{equation}
\label{eq:mu-quad}
\begin{gathered}
    \mu_t(x) = \tfrac12 x\cdot S_t\, x + c_t, \quad\text{where}\\
    S_t = \tfrac12 \dot C_t C_t^{-1} - \sigma^2 C_t^{-1}, \\
    \dot c_t = -\sigma^2 \operatorname{tr} S_t,
\end{gathered}
\end{equation}
and $\lambda_t^\top \phi(x) = \tfrac12 x\cdot(\dot S_t + S_t^2)\,x$.
\end{proposition}

\begin{proof}
By Theorem~\ref{th:gaussian}, the MGD SDE~\eqref{eq:sde:quad} has a unique
solution, with drift $S_tx$ and $S_t$ as in~\eqref{eq:mu-quad}. This drift is a
gradient precisely when $S_t$ is symmetric, that is when $C_t$ and $\dot C_t$
commute, and then $S_tx = \nabla\mu_t$ for $\mu_t$ as in~\eqref{eq:mu-quad}.
Substituting into~\eqref{eq:opt-hjb} gives
\begin{equation*}
    \tfrac12 x\cdot\big(\dot S_t + S_t^2\big)x + \dot c_t
    + \sigma^2\operatorname{tr}S_t = \lambda_t^\top\phi(x),
\end{equation*}
which holds with the stated $\lambda_t$ and
$\dot c_t = -\sigma^2\operatorname{tr}S_t$. Theorem~\ref{th:gaussian} gives the
moment constraints, so the pair $(p_t,\mu_t)$ is feasible and
solves~\eqref{eq:opt-fp}--\eqref{eq:opt-hjb}.

The objective is convex in $(p^u_t, p^u_tu_t)$ and the constraints are linear,
so the minimizer is unique.
\end{proof}

}

\section{Conjectures}
\label{app:conjecture}

In this appendix, we support Conjectures \ref{conj:convergence} and \ref{conj:entropybound} by performing a Taylor expansion of the Fokker-Planck equation (\ref{eq:fpe}). This formal derivation provides convergence rates for the entropy of the MGD solution and its lower bound towards the entropy of the maximum entropy distribution.
Let us write the Fokker-Planck Equation for the MGD SDE \eqref{eq:sde} as
\begin{equation*}
     \sigma^{-2}\partial_tp_{t}^{\sigma} =\nabla\cdot(p_{t}^{\sigma}((-\sigma^{-2}\eta_{t}+\theta_{t})^\top \nabla\phi))+\Delta p_{t}^{\sigma}. 
\end{equation*}
When $\sigma$ goes to infinity, we expect $\sigma^{-2}\partial_tp_{t}^{\sigma}$ and $\sigma^{-2}\eta_{t}$ to vanish. Assuming that $\eta_t$, $\theta_{t}$, and $p_t^{\sigma}$ admit a limit as $\sigma$ goes to infinity, and denoting
\begin{equation*}
    \lim_{\sigma\to\infty} p_{t}^{\sigma} = p^*_t \mbox{ and }\lim_{\sigma\to\infty} \theta_{t} = \theta^*_t,
\end{equation*}
the Fokker-Planck Equation gives
\begin{equation*}
    0 =\nabla\cdot(p_t^* ({\theta_t^*}^\top \nabla\phi))+\Delta p_t^*.
\end{equation*}
The solution to this equation is 
\[p_{t}^*(x) = e^{-{\theta_t^*}^\top \phi(x)} / {\cal Z}_{t}^*~,
\] 
for  ${\cal Z}_{t}^*= \int e^{-{\theta_t^*}^\top \phi(x)}dx$. 
Taking the limit as $\sigma \to\infty$ in the moments equality, we obtain
\begin{equation*}
    \int\phi(x) p_t^*(x) dx = \lim_{\sigma\to\infty}\int\phi(x) p_{t}^{\sigma}(x)dx = m_t.
\end{equation*}
This shows that the distribution with density $p_t^*$ is exponential, with moments $m_t$. Therefore, $p_{t}^*$ is the unique maximizer of the entropy $H(q)$,  under the constraints $\mathbb{E}_q[\phi] = \mathbb{E}[\phi(I_t)]$,  and $\theta_t^*$ are the associated Lagrange multipliers. This also implies that $p_{t=1}^*=p_*$ and $\theta_{t=1}^*= \theta_*$.

The term led by $\sigma^{-2}$ in the Fokker-Planck Equation suggests to Taylor expand $p^\sigma_t$ in $\sigma^{-2}$: 
\begin{align*}
   p_{t}^{\sigma}(x) &= p_{t}^*(x)\big(1+\sigma^{-2} q_{t}(x) +o(\sigma^{-2})\big).
\end{align*}
Injecting this expansion in the entropy of $p_t^{\sigma}$, we obtain
\begin{align*}
H(p_t^\sigma) -H(p^*_t) &=- \sigma^{-2}\int q_{t}(x) p_{t}^*(x)dx+o(\sigma^{-2})\\
& = o(\sigma^{-2})
  \end{align*}  
where we used $\int q_{t} p_{t}^*dx = 0$, which follows from integrating the expansion for $p^\sigma_t$ above since $\int p_t^\sigma dx=\int p_t^* dx=1$. 
As a consequence, 
\begin{align*}
    |H(p_t^\sigma) -H(p_t^*)|
    =&o(\sigma^{-2}) \leq C \sigma^{-2}
\end{align*}
for some $C$. At $t=1$, since $p_{t=1}^* = p_*$, we recover \eqref{eq:conj_entrop} for Conjecture~\ref{conj:convergence}.

Assuming that we can also perform an asymptotic expansion for $\theta_t$: 
\begin{equation*}
    \theta_t
     = \theta_{t}^* + \sigma^{-2} \tilde \theta_{t} + o(\sigma^{-2}),
\end{equation*} 
we deduce that the lower bound (\ref{eq:entropy}) follows
\begin{align*}
   \int_{0}^1\theta_{t}^\top \frac{d}{dt}m_tdt = &\int_{0}^1{\theta_{t}^*}^\top  \frac{d}{dt}m_t dt \\&+\sigma^{-2}\int_{0}^1\tilde \theta_{t}^\top \frac{d}{dt}m_tdt +o(\sigma^{-2}).
\end{align*}
We also deduce that the Fisher divergence vanishes as $o(\sigma^{-2})$ since
\begin{align*}
&\sigma^2\mathbb{E}\big[|\nabla\log  p_t^{\sigma}(X_t) +\theta_t^\top \nabla\phi|^2\big]\\
&\quad =\sigma^2\mathbb{E}\big[|\nabla\log (p_t^{\sigma}/ p_t^*)(X_t) +(\theta_t-\theta_t^*)^\top \nabla\phi(X_t)|^2\big]\\
&\quad =\sigma^{2}\mathbb{E}\big[|\sigma^{-2}\nabla q_t(X_t) +\sigma^{-2}\tilde \theta_t^\top \nabla\phi(X_t) +o(\sigma^{-2})|^2\big] \\
&\quad =\sigma^{-2}\mathbb{E}\big[|\nabla q_t(X_t) +\tilde\theta_t^\top \nabla\phi(X_t) +o(1)|^2\big] \\
&\quad = O(\sigma^{-2}).
\end{align*}
Therefore, if we denote
\begin{equation*}
H_*^{\sigma} =   H(p_0) +\int_{0}^1\theta_{t}^\top \frac{d}{dt}m_t dt,
\end{equation*}
from \eqref{eq:dHdt_equality} we formally deduce that there exists $C'$ such that 
\begin{equation*}
    |H_*^\sigma -H(p_*)|  =\sigma^{-2}C'.
\end{equation*}
The entropy lower bound $H_*^\sigma\leq H(p_1^\sigma)$ thus converges towards $H(p_*)$ as $O(\sigma^{-2})$ as suggested in \eqref{eq:H_epsilon_bound_conj} of Conjecture~\ref{conj:entropybound}.

\section{Alternative Numerical Implementations}
\label{app:alt_implementations}

Algorithm~\ref{alg:mgd} computes $\eta_t$ and $\theta_t$ on-the-fly using the current particle ensemble. This section describes two alternatives that may be preferable depending on the application: the first prioritizes speed and scalability, while the second preserves the interpretation of $\eta_t$ and $\theta_t$ as intrinsic parameters of the generative model.

\subsection{Precomputed Transport via Interpolant Regression}
\label{sec:precomputed}

The MGD SDE~\eqref{eq:sde} can also be written as
\begin{equation}
    \label{eq:sde:lump}
    dX_t = \lambda^\top_t \nabla \phi(X_t) dt + \sqrt{2} \sigma dW_t
\end{equation}
where $\lambda_t$ is a Lagrange multiplier used to enforce $\mathbb{E}[\phi(X_t)] = m_t$. That is, the decomposition $\lambda_t = \eta_t - \sigma^2 \theta_t$ used in text is not unique. In particular, we can also use $\lambda_t = \tilde \eta_t - \sigma^2 \tilde \theta_t$, with $\tilde \eta_t$ computed using the Gram matrix evaluated on the interpolant $I_t$ rather than on the particles $X_t$.  This changes the predictor step in Algorithm~\ref{alg:mgd}, but the corrector step still enforces exact moment preservation. It allows precomputation of $\tilde \eta_t$ before sampling.

Specifically, instead of solving $G_t \eta_t = \frac{d}{dt} m_t$ while sampling, we can precompute $\tilde \eta_t$ via solution of the regression problem:
\begin{equation}
\label{eq:eta_regression}
\tilde\eta_t = \argmin_{\hat\eta_t} \, \mathbb{E}\big[ |\hat\eta_t^\top \nabla\phi(I_t) - \dot{I}_t|^2 \big],
\end{equation}
where $\dot{I}_t = \frac{d}{dt} I_t$. This can be solved by SGD without matrix inversion, using mini-batches of fresh samples $Z \sim \mathcal{N}(0, \mathrm{Id})$:
\begin{equation}
\label{eq:eta_sgd}
\tilde\eta_t^{k+1} = \tilde\eta_t^k - h \, \mathbb{E}\big[\nabla\phi(I_t) \cdot \big((\tilde\eta_t^k)^\top \nabla\phi(I_t) - \dot{I}_t\big)\big].
\end{equation}
Variants such as Adam or L-BFGS can also be used. The resulting scheme is summarized in Algorithm~\ref{alg:precomputed}. Note that this algorithm still requires to solve a linear system to obtain $\tilde \theta_k$, but this too could be modified by solving \[\frac{1}{n_{\rm{rep}}} \sum_{i=1}^{n_{\rm{rep}}} \phi(y_k^i- h\sigma^2 \tilde{\theta}_k^\top \nabla\phi(y_k^i)) - m_{(k+1)h}=0\] for $\tilde \theta_k$ differently.

\begin{algorithm}[H]
\caption{MGD with Precomputed Transport}\label{alg:precomputed}
\begin{algorithmic}
\STATE
\textbf{Input:} volatility $\sigma$; number of steps $n_\sigma$; time step $h=1/n_\sigma$; number of replicas $n_{\rm{rep}}$; moments  $m_t = \mathbb{E}[\phi(I_t)]$
\STATE \textbf{Precomputation:} On the time grid $\{t_j = \frac{j}{n_\sigma}\}_{j=0}^{n_\sigma}$, solve~\eqref{eq:eta_regression} via SGD to obtain $\{\eta_{t_j}\}$
\STATE \textbf{Initialize:} $x_0^i \sim \mathcal{N}(0,\mathrm{Id})$ for $i = 1, \ldots, n_{\rm{rep}}$
\FOR{$k = 0, \ldots, n_\sigma-1$}
    \STATE \textit{Predictor (using precomputed $\eta_{t_k}$)}
    \FOR{$i = 1, \ldots, n_{\rm{rep}}$}
        \STATE Sample $\xi_k^i \sim \mathcal{N}(0, \mathrm{Id})$
        \STATE Set $y_k^i = x_k^i + h \, \tilde\eta_{t_k}^\top \nabla\phi(x_k^i) + \sqrt{2h}\sigma \, \xi_k^i$
    \ENDFOR
    \STATE \textit{Corrector (project to preserve moments)}
    \STATE Compute $\hat{G}'_k = \frac{1}{n_{\rm{rep}}} \sum_{i=1}^{n_{\rm{rep}}} \nabla\phi(y_k^i) \cdot\nabla\phi(y_k^i)^\top$
    \STATE Solve $h\sigma^2 \hat{G}'_k \tilde{\theta}_k = \frac{1}{n_{\rm{rep}}} \sum_{i=1}^{n_{\rm{rep}}} \phi(y_k^i) - m_{(k+1)h}$ for $\tilde{\theta}_k$
    \FOR{$i = 1, \ldots, n_{\rm{rep}}$}
        \STATE Set $x_{k+1}^i = y_k^i - h\sigma^2 \tilde{\theta}_k^\top \nabla\phi(y_k^i)$
    \ENDFOR
\ENDFOR
\STATE \textbf{Output:} Samples $(x_{n_\sigma}^i)_{1\leq i\leq n_{\rm{rep}}}$
\end{algorithmic}
\end{algorithm}

\subsection{Offline Learning of Coefficients}
\label{sec:offline}

If the coefficients $\eta_t$ and $\theta_t$ are of intrinsic interest, one can learn them in a preprocessing phase on a time grid, then sample by propagating one particle at a time using these fixed coefficients. This trades computation time for memory and enables fully parallel sampling.

The coefficients are built sequentially: use $\eta_t$, $\theta_t$ to propagate particles to time $t + \Delta t$, collect statistics to estimate the Gram matrix at this new time, then compute $\eta_{t+\Delta t}$, $\theta_{t+\Delta t}$. Crucially, the Gram matrix can be estimated by accumulating contributions one particle (or batch) at a time, without storing all positions simultaneously. The procedure is summarized in Algorithm~\ref{alg:offline}.

\begin{algorithm}[H]
\caption{MGD with Offline Coefficient Learning}\label{alg:offline}
\begin{algorithmic}
\STATE
\textbf{Input:} volatility $\sigma$; number of steps $n_\sigma$; time step $h=1/n_\sigma$; number of replicas $n_{\rm{rep}}$; moments  $m_t = \mathbb{E}[\phi(I_t)]$
\STATE \textbf{Learning phase:}
\STATE Compute $\hat{G}_0 = \frac{1}{n_{\rm{rep}}} \sum_{i=1}^{n_{\rm{rep}}} \nabla\phi(z^i) \cdot\nabla\phi(z^i)^\top$ with $z^i \sim p_0$
\STATE Solve for $\hat{\eta}_0$, $\hat{\theta}_0$
\FOR{$k = 1, \ldots, n_\sigma$}
    \STATE Initialize accumulator $\hat{G}_k = 0$
    \FOR{batch $b = 1, \ldots, B$}
        \STATE Propagate $n_b$ particles from $t=0$ to $t=kh$ using $\{\hat{\eta}_\ell, \hat{\theta}_\ell\}_{\ell < k}$
        \STATE Accumulate: $\hat{G}_k \gets  \hat{G}_k + \frac{1}{n_b} \sum_{i=1}^{n_b} \nabla\phi(x_k^i)\cdot \nabla\phi(x_k^i)^\top$
    \ENDFOR
    \STATE Normalize $\hat{G}_k\gets \frac1B\hat{G}_k$ and solve for $\hat{\eta}_k$, $\hat{\theta}_k$
\ENDFOR
\STATE \textbf{Sampling phase:} (can be done one particle at a time)
\STATE Initialize $x_0 \sim \mathcal{N}(0, \mathrm{Id})$
\FOR{$k = 0, \ldots, n_\sigma - 1$}
    \STATE Sample $\xi_k \sim \mathcal{N}(0, \mathrm{Id})$
    \STATE Set $x_{k+1} = x_k + h(\hat{\eta}_k - \sigma^2\hat{\theta}_k)^\top \nabla\phi(x_k) + \sqrt{2h}\sigma\,\xi_k$
\ENDFOR \\
\textbf{Output:} 
Samples $(x_{n_\sigma}^i)_{1\leq i\leq n_{\rm{rep}}}$
\end{algorithmic}
\end{algorithm}

\section{Experimental details}
\label{app:exp}

This appendix reviews experimental details of the numerical experiments performed in Sections \ref{sec:convergence_numerics} and \ref{sec:scat}.
In all experiments, the number of steps required by MGD to reach convergence increases with $\sigma$, ranging from $10^3$ to $10^4$.

\subsection{Regularization}
\label{app:regu}

Algorithm \ref{alg:mgd} requires inverting empirical Gram matrices, which we stabilize through a simple regularization procedure. First, we discard any potential $\phi_k$ satisfying
\begin{equation}
\frac{1}{m}\sum_{i=1}^m\nabla\phi_k(x^i)\cdot\nabla\phi_k(x^i)^\top = 0,
\end{equation}
as these correspond to vanishing Lagrange multipliers in $p_*$ (for the scattering spectra case, symmetries produce exact zeros \cite{Cheng2023ScatteringSM}). We then normalize the remaining potentials by their empirical norm, setting the Gram matrix diagonal to unity so that all potentials contribute at comparable scale. Finally, we add a small regularization $\delta \, \mathrm{Id}$ with $\delta = 10^{-7}$ before inversion.

\subsection{Entropy and $D_{\rm KL}$ Estimation}
\label{sec:entrop}

In Section~\ref{sec:convergence_numerics}, we estimate entropies and Kullback--Leibler divergences from one-dimensional histograms. We use $n_{\mathrm{rep}} = 10^6$ replicas and $n_\sigma = 10^4$ discretization steps for MGD, except for the right column of Figure \ref{fig:scalar_convergence} where up to $n_\sigma = 3.10^5$ discretization steps were used. Histograms are constructed from $n_{\mathrm{quantiles}} = 500$ quantiles, yielding discrete density estimates from which we compute the entropies and divergences.

\subsection{Financial Time Series and Physical Fields}

The S\&P time series is preprocessed following \cite[Appendix E]{Morel2022ScaleDA}. Cosmic web fields are 2D slices from 3D dark matter simulations \cite{villaescusa2020quijote} with a logarithmic transformation, as described in \cite[Appendix G]{Cheng2023ScatteringSM}. Turbulent vorticity fields are obtained from 2D incompressible Navier--Stokes simulations \cite{SCHNEIDER01012006}. All signals are standardized, and covariance determinants for negentropy estimation are computed via Welch's method \cite{Welch1967TheUO}.

Figures~\ref{fig:analysis_1D}, \ref{fig:sampling_2D_entropy} and \ref{fig:celeba} show convergence of the negentropy estimator $\Delta H_*^\sigma$ for $\sigma^2 \in \{0.1, 0.25, 1, 2.5, 4, 5.5\}$, using $m = 100$ particles.
The number of discretization steps is adapted to each dataset and volatility (in order of increasing $\sigma^2$): S\&P uses $\{1000, 1000, 1200, 1200, 2200, 2500\}$, cosmic web $\{1000, 1000, 1000, 1000, 2500, 2700\}$, 2D turbulence $\{1000, 1000, 1000, 1000, 2000, 3300\}$, and CelebA $\{1500, 1500, 2500, 4000, 4000, 6000\}$. For CelebA, we also consider $\sigma^2=7$ and $\sigma^2=8.5$, respectively computed with $6000$ and $7000$ steps. Convergence is verified by moment matching throughout the dynamics and confirming stability under additional steps.

The histograms in Figures~\ref{fig:rolling_vol} and \ref{fig:2D_small_D_wav_hist} use $500$ samples for $\sigma^2 \in \{0, 0.1, 5.5\}$, with $1000$ discretization steps for $\sigma^2=0$.

\section{Proofs for quadratic function $\phi$}
\label{proof:th-gaussian}

{
\begin{proof}[Proof of Theorem~\ref{th:gaussian}]

For the quadratic moment generating function $\phi$, $\nabla\phi\nabla\phi^\top$ is a set of quadratic functions, such that
\begin{align*}
    \mathbb{E}(\phi(X_t)) &= \mathbb{E}(\phi(I_t))  \\\Leftrightarrow~ \mathbb{E}(\nabla\phi(X_t)\cdot\nabla\phi(X_t)^\top) &= \mathbb{E}(\nabla\phi(I_t)\cdot\nabla\phi(I_t)^\top).
\end{align*}
In this case, the MGD SDE is equal to
\begin{equation*}
    dX_t  =  \big(\eta_t^\top-\sigma^2 \theta_t^\top\big) \nabla \phi(X_t)\, dt  + \sqrt {2}\, \sigma\, dW_t,
\end{equation*}
where
\begin{align*}
\mathbb{E}\big[\nabla \phi(I_t) \cdot \nabla \phi(I_t)^\top \big]\, \eta_{t} &= \mathbb{E}\big[\frac{d}{dt}{I_t}\nabla\phi(I_t)\big]
\\
\mathbb{E}\big[\nabla \phi(I_t) \cdot \nabla \phi(I_t)^\top \big]\, \theta_{t} &= \mathbb{E}\big[\Delta \phi(I_t)\big],
\end{align*}
because $\Delta \phi$ is constant. Since $\mathbb{E}\big[\Delta \phi(I_t)\big] = -\mathbb{E}[\nabla\log q_t(I_t)\cdot\nabla\phi(I_t)]$, for $q_t$ the PDF of $I_t$, this system is the solution to the minimization problem
\begin{align*}
\eta_t =&\arg\min_{\eta} \mathbb{E}\big[|\eta^\top\nabla \phi(I_t) -\frac{d}{dt}I_t|^2\big],
\\
\theta_t =\arg&\min_{\theta} \mathbb{E}\big[|\theta^\top\nabla \phi(I_t)+\nabla\log q_t(I_t)|^2\big].
\end{align*}

Because $\nabla\phi$ is a set of linear functions, we can set that $\eta^\top\nabla \phi(I_t) = \tilde\eta^\top I_t$ and $\theta^\top\nabla \phi(I_t) = \tilde\theta^\top I_t$, solve for $\tilde\eta$ and $\tilde\theta$, and prove that the system's solutions are
\begin{align*}
\eta_t^\top\nabla \phi(I_t) =\frac{1}{2}\dot C_t C_t^{-1}I_t~,
\\
\theta_t^\top\nabla \phi(I_t) = C_t^{-1}I_t~.
\end{align*}
This shows that the MGD SDE is~\eqref{eq:sde:quad}.
This SDE is not a McKean Vlasov equation but a classical SDE, whose unique strong solution exists because its drift is continuous in time and Lipschitz in space.
\end{proof}

\begin{proof}[Proof of Theorem~\ref{th:gaussian:2}]

Since the matrices $C_0$ and $C_1$ commute, all the matrices $C_t$ commute. Thus, an integrating factor method shows that the solution to the MGD SDE~\eqref{eq:sde:quad} is
\begin{equation*}
\begin{split}
X_t = &\exp\left(\int_0^t (\tfrac12 \dot C_s C_s^{-1}- \sigma^2 C_s^{-1}) \, ds\right) X_0 \\
&+ \sqrt{2\sigma^2} \int_0^t \exp\left(\int_s^t (\tfrac12 \dot C_u C_u^{-1}- \sigma^2 C_u^{-1}) \, du\right) \, dW_s
\end{split}
\end{equation*}
The conditional law of $X_t | X_0$ is gaussian, with mean $\exp\left(-\int_0^t (\tfrac12 \dot C_s C_s^{-1}- \sigma^2 C_s^{-1}) \, ds\right)X_0 = \mu_tX_0$ and covariance
\begin{equation*}
\Sigma_t = 2\sigma^2\int_0^{t} \exp\left(2 \int_0^{s} (\tfrac12 \dot C_u C_u^{-1}- \sigma^2 C_u^{-1}) \, du\right) \, ds~.
\end{equation*}
A change of variable gives that $\frac{\Sigma_t}{2} =$
\begin{align*}
      &\int_0^{t\sigma^2} \exp\Big({\int_0^{s} (\sigma^{-2}\tfrac{d}{dt}\log C_{t-u\sigma^{-2}}- 2 C_{t-u\sigma^{-2}}^{-1}) du}\Big ) ds  \\
    =&  \int_0^{t\sigma^2}( C_{t-s\sigma^{-2}}C_0^{-1})^{\sigma^{-2}}\exp\Big({-2\int_0^{s}  C_{t-u\sigma^{-2}}^{-1} du}\Big ) ds.
\end{align*}

Using that $C_{t}$ is bounded and has strictly positive eigenvalues, dominated convergence shows that 
\begin{equation*}
     \Sigma_t \underset{\sigma \to \infty}{\to} 2\int_0^{\infty} \exp\Big({-2 \int_0^{s} C_{t}^{-1} \, du} \Big)\, ds  = C_t.
\end{equation*}

A similar argument shows that
\begin{equation*}
    \mu_t  \underset{\sigma \to \infty}{\to}  0.
\end{equation*}

We derive $p_t^\sigma$, the density of $X_t$, with the law of total probability
\begin{equation*}
p^\sigma_t(x) = c_t{\mathbb{E}} \left[\exp\Big({-{\frac{1}{2}|\Sigma_t^{-1/2} \left(x - \mu_tX_0 \right) |^2}}\big)\right]~,
\end{equation*}
where $c_t =  {(2\pi)^{-d/2} {\rm det} (\Sigma_t)^{-1/2}} $. It is straightforward with dominated convergence, with respect to $X_0$ for a fixed $x$, to show that, when $\sigma$ goes to infinity 
\begin{align*}
   p_t^\sigma(x) &\underset{\rm pointwise}{\to} p_t^*(x) \\& \underset{\rm{def}}{:=} (2\pi)^{-d/2} \det (C_t)^{-1/2} \exp\Big({-\frac{1}{2}|C_t^{-1/2}x|^2}\big).
\end{align*}

The Kullback-Leibler divergence between $p_t^\sigma$ and $p_t^*$ is given by
$$ D_{\rm{KL}}(p_t^\sigma\|p_t^*) =\int\log(p_t^\sigma(x)/p_t^*(x))p_t^\sigma(x)dx .$$
With a change of variable,
\begin{align*}
    &-\int\log(p_t^*(x)) p_t^{\sigma}(x)dx -\log((2\pi)^{d/2}{\rm det} (C_t)^{1/2})\\ =& \tfrac12\int |C_t^{-1/2}x |^2  p_t^{\sigma}(x)dx\\=& \tfrac{c_t}{2}\int |C_t^{-1/2}x |^2e^{-\tfrac{1}{2}|\Sigma_t^{-1/2} \big(x- \mu_tx_0 \big) |^2} p_0(x_0) dx dx_0 \\ =& \tfrac{c_t}{2}\int |C_t^{-1/2}(y+\mu_t x_0) |^2 e^{-\tfrac{1}{2}|\Sigma_t^{-1/2}y|^2} p_0(x_0) dy dx_0.
\end{align*}
 By dominated convergence 
$$\tfrac12\int |C_t^{-1/2}x |^2 p_t^{\sigma}(x)dx \underset{\sigma \to \infty}{\to} \tfrac12\int |C_t^{-1/2}y|^2p_t^*(y)  dy =\tfrac d2,$$ 
and thus 
$$-\int \log (p_t^*(x)) p_t^{\sigma}(x)dx \underset{\sigma \to \infty}{\to}\log((2\pi)^{d/2}{\rm det} (C_t)^{1/2})+\tfrac d2.$$
The same dominated convergence argument is used to show that
$$\int\log(p_t^\sigma(x))p_t^\sigma(x)dx\underset{\sigma \to \infty}{\to}-\log((2\pi)^{d/2}{\rm det} (C_t)^{1/2})-\tfrac d2,$$
proving that $D_{\rm{KL}}(p_t^\sigma\|p_t^*)$ converges towards $0$.
\end{proof}
}

\section{Additional Theoretical Results}
\label{app:proofs}

This appendix establishes a rigorous existence and convergence result for the Moment Guided Diffusion (MGD) dynamics, when moments $m_t$ are fixed. We introduce a regularized version of MGD that includes a confining potential, which provides the analytical control needed to prove convergence to maximum entropy distributions.

\subsection{Overview and Main Result}
\label{sec:overview}

We state our main theorem upfront, then provide detailed proofs in subsequent sections.

\subsubsection{Setup and Notation}

Throughout this appendix, we work with a \emph{regularized} MGD dynamics that includes a confining potential $\frac12\epsilon |x|^2$ for some $\epsilon > 0$. This regularization serves two purposes: (i) it ensures solutions remain well-behaved (bounded cross-entropy), and (ii) it provides a reference measure $p_\epsilon(x) = Z_\epsilon^{-1} e^{-\frac12\epsilon |x|^2}$ with good functional inequalities.

The key objects are:
\begin{itemize}
    \item The \emph{cross-entropy} (negative KL divergence to reference): 
    \[
    H_\epsilon(p) = -\int p(x) \log \frac{p(x)}{p_\epsilon(x)} \, dx = -D_{\mathrm{KL}}(p \| p_\epsilon)
    \]
    \item The \emph{regularized maximum entropy distribution}: \[p_*^\epsilon(x) = Z_*^{-1} e^{-\theta_*^\top \phi(x) - \frac12\epsilon |x|^2}\] satisfying $\mathbb{E}_{p_*^\epsilon}[\phi] = \mathbb{E}[\phi(X)]$
    \item The \emph{Gram matrix}: $G_t = \mathbb{E}[\nabla\phi(X_t) \cdot \nabla\phi(X_t)^\top]$
\end{itemize}

\begin{remark}[Role of $\epsilon$]
The regularization parameter $\epsilon > 0$ is held fixed throughout. The resulting limit $p_*^\epsilon$ is the maximum entropy distribution with an additional Gaussian confining term. Taking $\epsilon \to 0$ would recover the unregularized maximum entropy distribution $p_*$, but this limit is not analyzed here.
\end{remark}

Throughout this appendix, $| \cdot|$ will be the $\ell^2$ norm of a vector with respect to coordinates ( e.g. $|x| = (\sum_u |x(u)|^2)^{1/2}$ and $|\Delta\phi|=|\Delta\phi(x)| =(\sum_k |\phi_k(x)|^2)^{1/2}$) while $\| \cdot\|_\infty$ will be the $\ell^\infty$ norm with respect to domain and coordinates ( e.g. $\|\theta_t\|_\infty = \underset{1\leq k\leq r}{\max}|\theta_{t,k}|$ for coordinates $\theta_{t,k}$ and  $\|\nabla\phi\|_\infty = \underset{x\in\mathbb{R}^d,1\leq i\leq d ,1\leq k\leq r}{\max}|\frac{\partial}{\partial x_i}\phi_k(x)|$). When specified, the $\ell^\infty$ norm can be taken with respect to a restricted domain (e.g. $\|p_t\|_{K,\infty} = \underset{x\in K}{\max}|p_t(x)|$). Finally, $\| \cdot \|_{\rm op}$ is the operator norm of a matrix.

\subsubsection{Hypotheses}

We require the following regularity conditions:

\begin{hypothesis}[Regularity of $\phi$]
\label{hypothesis_th}
The family of $\mathcal{C}^4$ functions $(\phi_k)_k$ is linearly independent and bounded, with bounded derivatives. The functions $(\nabla\phi_k)_k$ are linearly independent. For all $k$, the map $x \mapsto x\cdot\nabla\phi_k(x)$ is bounded.
\end{hypothesis}

\begin{hypothesis}[Regularity of $p_0$]
\label{hypothesis_th_p_0}
The initial density $p_0$ is $\mathcal{C}^4$, has finite variance, finite entropy, and $p_0$ and its derivatives are bounded.
\end{hypothesis}




\subsubsection{Main Theorem}

\begin{theorem}[Convergence with Fixed Moments]
\label{th:convergence}
Let $\phi$ and $p_0$ satisfy Hypotheses~\ref{hypothesis_th} and \ref{hypothesis_th_p_0}. Assume that the interpolant $I_t$ has constant moments:
\[
\forall t \in [0,1], \quad \frac{d}{dt} m_t = 0.
\]
Then, for any $\epsilon > 0$, the strong solutions $X_t$ of the regularized MGD~\eqref{eq:sde_app} with PDF $p_t^\sigma$ exist for all $t \in [0,1]$ and $\sigma \in \mathbb{R}_+$.

If the density $p_*^\epsilon(x) = Z_*^{-1} e^{-\theta_*^\top \phi(x) - \frac12\epsilon |x|^2}$ with $\mathbb{E}_{p_*^\epsilon}[\phi] = \mathbb{E}[\phi(X)]$ exists, then:
\[
\lim_{\sigma\to\infty}D_{\mathrm{KL}}(p_t^\sigma \| p_*^\epsilon) = 0.
\]
\end{theorem}

{
\begin{remark}[Dynamics with fixed moments]
When the moments are fixed, MGD dynamics increases the entropy of the initial distribution of $X_0$ while preserving its moments $\mathbb{E}(\phi(X))$. Here, $p_0$ can be taken either as the data distribution $p$, or as the outcome of any run of MGD with non-fixed moments but target density $p$. In the latter case, this shows that combining MGD at fixed $\sigma$ with MGD at fixed moments yields maximum entropy sampling. However, MGD with fixed moments corresponds to an algorithm that does not benefit from the homotopic path, and is similar to the Langevin dynamics. Numerically, it provides a stable way to perform Langevin dynamics, since the reprojection of the moments at each step ensures stability; on the other hand, it may suffer from slower convergence on multimodal distributions.
\end{remark}}

\subsubsection{Proof Strategy}

\paragraph{Theorem~\ref{th:convergence} (Convergence with Fixed Moments).}
The proof proceeds in two stages:

\textbf{Stage 1: Existence of solutions} (Section~\ref{sec:proofth_existence})
\begin{enumerate}
    \item Introduce a regularized SDE with parameter $\delta > 0$ that ensures the Gram matrix $G_t + \delta I$ is invertible.
    \item Show that solutions $p_t^\delta$ remain bounded in cross-entropy $H_\epsilon$ (Lemma~\ref{lemma:bound_delta_KL}).
    \item Use this bound to establish tightness (Lemma~\ref{lemma:delta_tight}) and uniform bounds on the Gram matrix (Lemma~\ref{lemma:invertible_G_delta}).
    \item Apply Kunita's theory to bound derivatives of $p_t^\delta$ (Lemma~\ref{lemma:bounds_p_delta}).
    \item Extract a convergent subsequence via Arzel\`a-Ascoli as $\delta \to 0$ (Lemma~\ref{lemma:subsequence_delta}).
    \item Verify the limit satisfies the original Fokker-Planck equation (Lemma~\ref{lemma:limit_FPE}).
\end{enumerate}

\textbf{Stage 2: Convergence to maximum entropy} (Section~\ref{sec:proofth_convergence})
\begin{enumerate}
    \item Show that $H_\epsilon(p_t)$ is a Lyapunov function (non-decreasing in $t$).
    \item Extract a subsequence $t_n \to \infty$ along which the Fisher divergence vanishes (Lemma~\ref{lemma:extraction}).
    \item Conclude $D_{\mathrm{KL}}$ convergence to $p_*^\epsilon$ using the Poincaré inequality (Lemmas~\ref{lemma:KL_subsequence}--\ref{lemma:full_convergence}).
\end{enumerate}

\subsection{Regularized MGD Dynamics}
\label{sec:reg:dyn}

\subsubsection{Motivation: Wasserstein Gradient Flow}

The Fokker-Planck equation for MGD can be interpreted as a constrained Wasserstein gradient flow. Consider maximizing the entropy $q \mapsto H(q)$ subject to the time-dependent moment constraint $\mathbb{E}_q[\phi] = m_t$. With Lagrange multipliers $\lambda$, this amounts to minimizing at each time $t$ the functional:
\[
\mathcal{F}_t(q, \lambda) = -H(q) + \lambda^\top \big(\mathbb{E}_q(\phi) - \mathbb{E}(\phi(I_t))\big).
\]

The constrained Wasserstein gradient flow is:
\[
\frac{\partial p_t}{\partial t} = -\nabla \cdot \Big(p_t \nabla \frac{\delta \mathcal{F}_t}{\delta q}(p_t, \lambda_t)\Big),
\]
where $\lambda_t$ is chosen to satisfy $\mathbb{E}_{p_t}(\phi) = m_t$. A calculation shows this requires:
\[
\lambda_t = G_t^{-1} \Big(\mathbb{E}_{p_t}[\Delta\phi] - \frac{d}{dt} m_t\Big),
\]
where $G_t = \mathbb{E}_{p_t}[\nabla\phi \cdot\nabla\phi^\top]$. Expanding the Wasserstein gradient flow recovers the MGD Fokker-Planck equation for $\sigma = 1$.

\subsubsection{The Confined Dynamics}

The existence and uniqueness of solutions to the MGD SDE is not guaranteed a priori. MGD is a McKean-Vlasov equation~\cite{mckean1966class, chaintron2022propagation} with a drift that is not Lipschitz continuous in the density $p_t^\sigma$. This can cause the Gram matrix $G_t$ to become singular, making the drift blow up.

To ensure solutions remain regular, we replace the entropy $H(p_t^\sigma)$ with the cross-entropy relative to a Gaussian reference measure $p_\epsilon(x) = Z_\epsilon^{-1} e^{-\frac12\epsilon |x|^2}$:
\[
H_\epsilon(p_t^\sigma) = -\int p_t^\sigma(x) \log \frac{p_t^\sigma(x)}{p_\epsilon(x)} \, dx.
\]
The maximizer of $H_\epsilon(q)$ subject to $\mathbb{E}_q[\phi] = \mathbb{E}[\phi(X)]$ is:
\[
p_*^\epsilon(x) = Z_\theta^{-1} e^{-\theta_*^\top \phi(x) - \frac12\epsilon |x|^2}.
\]
A bounded cross-entropy ensures the solution remains regular. The corresponding Wasserstein gradient flow leads to:

\begin{theorem}[Regularized MGD]
\label{prop:moments_app}
{ Assume that $\phi$ is twice differentiable and that $t \mapsto m_t$ is
differentiable.} Consider the SDE
\begin{equation}
\label{eq:sde_app}
dX_t = \big((\eta_t^\top - \sigma^2 \theta_t^\top) \nabla\phi(X_t) - \sigma^2\epsilon X_t\big) \, dt + \sqrt{2} \, \sigma \, dW_t,
\end{equation}
with $X_0=Z$, where $W_t$ is a Brownian motion, $\sigma\geq0$, and $\eta_t$ and $\theta_t$ solve
\begin{align}
\label{eq:eta_app}
G_t \, \eta_t &= \frac{d}{dt} m_t, \\
\label{eq:theta_app}
G_t \, \theta_t &= \mathbb{E}\big[\Delta\phi(X_t) - \epsilon X_t\cdot\nabla\phi(X_t)\big],
\end{align}
and $G_t = \mathbb{E}\big[\nabla\phi(X_t) \cdot \nabla\phi(X_t)^\top\big]$.

{ Assume that this system admits a solution on $[0,1]$. Then $X_t$ satisfies the moment constraint~\eqref{eq:th-moments}.}
\end{theorem}

The proof follows the same argument as Theorem~\ref{prop:moments} in Appendix~\ref{app:proof}. The term $-\sigma^2\epsilon X_t$ confines solutions, preventing mass from escaping to infinity.

The corresponding Fokker-Planck equation is:
\begin{equation}
\label{eq:FPE_epsilon}
\partial_t p_t^\sigma = \nabla \cdot \big(p_t(-\eta_t + \sigma^2 \theta_t)^\top \nabla\phi + \sigma^2 \epsilon x\big) + \sigma^2 \Delta p_t^\sigma.
\end{equation}

\subsubsection{Cross-Entropy as Lyapunov Function}

When moments are fixed ($\frac{d}{dt} m_t = 0$), the cross-entropy is a Lyapunov function:

\begin{proposition}
\label{proposition:dentropydD_app_epsilon}
Assume $X_t$ with density $p_t^\sigma$ follows the regularized MGD~\eqref{eq:sde_app}. If $d m_t /dt= 0$, then:
\[
\frac{d}{d\sigma} H_\epsilon(p_t^\sigma) \geq 0.
\]
\end{proposition}
The proof adapts Proposition~\ref{prop:entropy}.

\begin{remark}[Non-constant moments]
When $dm_t/dt \neq 0$, the Lyapunov function becomes:
\[
\frac{d}{dt} \Big(H_\epsilon(p_t^\sigma) - \int_0^t \theta_s^\top \frac{d}{ds} m_s \, ds\Big) \geq 0.
\]
However, we cannot rule out the possibility that $H(p_t^\sigma) \to -\infty$ while the integral diverges in a compensating way.
\end{remark}

\begin{remark}[Choice of reference measure]
We use $p_\epsilon \propto e^{-\frac12\epsilon |x|^2}$ for simplicity, but any reference measure $\propto e^{-f(x)}$ works if $f$ grows to infinity and has a Lipschitz gradient.
\end{remark}

\subsection{Proof of Theorem~\ref{th:convergence}: Existence and Convergence}
\label{proof:th-convergence}

We prove existence in Section~\ref{sec:proofth_existence} and convergence in Section~\ref{sec:proofth_convergence}.

\subsubsection{Existence of Solutions}
\label{sec:proofth_existence}

We introduce a regularized SDE with parameter $\delta > 0$, prove bounds uniform in $\delta$, then extract a convergent subsequence as $\delta \to 0$.

\paragraph{Step 1: The $\delta$-regularized Dynamics.}

Consider the regularized SDE for $\delta > 0$ and $t \in \mathbb{R}_+$:
\begin{equation}
\label{eq:delta_SDE}
dX_t^\delta = -\Big({\theta_t^\delta}^\top \nabla\phi(X_t^\delta) + \epsilon X_t^\delta\Big) dt + \sqrt{2} \, dW_t,
\end{equation}
where
\[
\theta_t^\delta = ({G_t^\delta} + \delta I)^{-1} \mathbb{E}[\Delta\phi(X_t^\delta) - \epsilon {X_t^\delta}\cdot\nabla\phi(X_t^\delta)^\top],
\]
with ${G_t^\delta} = \mathbb{E}[\nabla\phi(X_t^\delta) \cdot \nabla\phi(X_t^\delta)^\top]$.

Using that $({G_t^\delta} + \delta I)^{-1} \preceq\delta^{-1} I$, along with Hypothesis \ref{hypothesis_th}, we prove that the drift is Lipschitz in both the density of $X_t^\delta$ and space. By standard McKean-Vlasov theory, for any $p_0$ with finite variance, the SDE admits a unique strong solution with density $p_t^\delta$ (at least $\mathcal{C}^4$ by hypotheses) satisfying:
\begin{equation}
\label{eq:fokker_planck_delta}
\partial_t p_t^\delta(x) = \nabla \cdot \big(p_t^\delta({\theta_t^\delta}^\top \nabla\phi(x) + \epsilon x)\big) + \Delta p_t^\delta(x).
\end{equation}

\paragraph{Step 2: Cross-entropy Bounds.}

\begin{lemma}[Cross-entropy is bounded]
\label{lemma:bound_delta_KL}
The relative entropy $H_\epsilon(p_t^\delta) = -\int p_t^\delta(x) \log \frac{p_t^\delta(x)}{p_\epsilon(x)} dx$ satisfies:
\[
\forall (\delta, t) \in \mathbb{R}_+^* \times \mathbb{R}_+, \quad 0 \leq -H_\epsilon(p_t^\delta) \leq -H_\epsilon(p_0).
\]
\end{lemma}

\begin{proof}

Since $p_0$ has finite variance and entropy by hypothesis \ref{hypothesis_th_p_0}, and since the drift of the SDE over $X_t^\delta$ is Lipschitz, $p_t^\delta$ admits a finite entropy and finite second order moments at each time $t$. It thus admits a finite cross entropy $H_\epsilon(p_t^\delta)$ at each time $t$.

Computing as in Proposition~\ref{prop:entropy},
\begin{align*}
\frac{d}{dt} H_\epsilon(p_t^\delta) &= -{\theta_t^\delta}^\top \mathbb{E}[\Delta\phi(X_t^\delta) - \epsilon X_t^\delta\cdot\nabla\phi(X_t^\delta)]\\
&\quad + \mathbb{E}\big[|\nabla \log p_t^\delta(X_t^\delta) + \epsilon X_t^\delta|^2\big].
\end{align*}

Since $H_\epsilon(p_t^\delta)$ is finite, $p_t^\delta$ is not singularly supported, so ${G_t^\delta}$ is invertible (as the $\nabla\phi_k$ are linearly independent). Thus, $
\mathbb{E}\big[|\nabla \log p_t^\delta(X_t^\delta) + \epsilon X_t^\delta|^2\big]
\geq \mathbb{E}[\Delta\phi(X_t^\delta) - \epsilon {X_t^\delta}\cdot\nabla\phi(X_t^\delta)^\top] {G_t^\delta}^{-1} \mathbb{E}[\Delta\phi(X_t^\delta) - \epsilon {X_t^\delta}\cdot\nabla\phi(X_t^\delta)^\top].$

Combining and using that, since ${G_t^\delta}\succeq 0$, we have ${G_t^\delta}^{-1} - ({G_t^\delta} + \delta I)^{-1} \succeq 0$, 
\[
\frac{d}{dt} H_\epsilon(p_t^\delta) \geq 0 . \qedhere 
\]
\end{proof}

\paragraph{Step 3: Tightness.}

\begin{lemma}[Tightness]
\label{lemma:delta_tight}
The family $(p_t^\delta)_{t,\delta}$ is tight:
\begin{equation}
\forall \kappa > 0, \; \exists K \subset \mathbb{R}^d \text{ compact}, \; \forall t, \; \int_K p_t^\delta(x) dx \geq 1 - \kappa.
\end{equation}
\end{lemma}

\begin{proof}
Apply the variational inequality $\mathbb{E}_\mu[f] \leq D_{\mathrm{KL}}(\mu \| \nu) + \mathbb{E}_\nu[e^f]$ with $\mu = p_t^\delta$, $\nu = p_\epsilon$, $f(x) = |x|$:
\begin{align*}
\mathbb{E}\big[|X_t^\delta|\big]&\leq -H_\epsilon(p_t^\delta) + (2\pi\epsilon)^{-d/2} \int e^{-\frac12\epsilon |x|^2 + |x|} dx\\
&\leq -H_\epsilon(p_0) + C_\epsilon.
\end{align*}
By Chebyshev's inequality, for the euclidean ball $B_R$ with radius $R$ $$\int_{B_R} p_t^\delta dx \geq 1 - \frac{-H_\epsilon(p_0) + C_\epsilon}{R},$$ which exceeds $1 - \kappa$ for $R$ large enough.
\end{proof}

\paragraph{Step 4: Gram Matrix Bounds.}

\begin{lemma}[Gram matrix invertibility]
\label{lemma:invertible_G_delta}
Let $T > 0$ and $\delta_0 > 0$. There exists $\alpha > 0$ such that:
\begin{equation}
\forall t \in [0,T], \; \forall \delta \in (0, \delta_0], \quad {G_t^\delta} \succeq \alpha I.
\end{equation}
Consequently, $\sup_{(\delta,t) \in (0,\delta_0] \times [0,T]} \|\theta_t^\delta\|_\infty < \infty$.
\end{lemma}

\begin{proof}
The proof is by contradiction, which is equivalent to assuming $\liminf_{\delta \to 0} \det {G_t^\delta} = 0$ for some $t$, since ${G_t^\delta}$ is bounded (Hypothesis \ref{hypothesis_th}, $\nabla\phi$ is bounded). We extract $\delta_n \to 0$ such that $\det G_t^{\delta_n} \to 0$, and without loss of generality, by tightness (Lemma~\ref{lemma:delta_tight}) and Prokhorov's theorem assume that it is a weakly convergent subsequence $p_t^{\delta_n} \rightharpoonup p_\infty$.

Since $\nabla\phi$ is bounded and continuous, $$G_t^{\delta_n} \to \mathbb{E}_{p_\infty}\big[\nabla\phi \cdot\nabla\phi^\top\big]\implies\det \mathbb{E}_{p_\infty}\big[\nabla\phi \cdot\nabla\phi^\top\big] = 0.$$
At the same time,  by upper semi-continuity of cross-entropy $$H_\epsilon(p_0) \leq \lim_n H_\epsilon(p_t^{\delta_n}) \leq H_\epsilon(p_\infty)\leq 0.$$ Thus $p_\infty$ has finite cross-entropy, so it is not singularly supported, contradicting the singularity of $\mathbb{E}_{p_\infty}[\nabla\phi \cdot\nabla\phi^\top]$ (since $\nabla\phi_k$ are linearly independent).

The bound on $\theta_t^\delta$ follows since $\Delta\phi$ and $x \mapsto x\cdot\nabla\phi(x)$ are bounded.
\end{proof}

\paragraph{Step 5: Density Bounds via Kunita's Theory.}

\begin{lemma}[Bounds on $p_t^\delta$ and derivatives]
\label{lemma:bounds_p_delta}
The following are finite:
\begin{align}
&\sup_{(\delta,t) \in (0,\delta_0] \times [0,T]} \|\nabla p_t^\delta\|_\infty < \infty, \\
&\sup_{(\delta,t) \in (0,\delta_0] \times [0,T]} \|\nabla_x^2 p_t^\delta\|_\infty < \infty.
\end{align}
For any compact $K \subset \mathbb{R}^d$:
\begin{align}
&\sup_{(\delta,t) \in \mathbb{R}_+^* \times [0,T]} \|\partial_t p_t^\delta\|_{K,\infty} < \infty, \\
&\sup_{(\delta,t) \in (0,\delta_0] \times [0,T]} \|\partial_t \nabla p_t^\delta\|_{K,\infty} < \infty.
\end{align}
\end{lemma}

\begin{proof}
The density $p_t^\delta$ follows the Feynman-Kac formula

\begin{align*}
p^{\delta}_t(x) &= \mathbb{E}\big[\Lambda_{t}(x) p_0(Y^{\delta}_t(x))\big]
\end{align*}
where $\Lambda_t(x)= \exp\Big(-\int_0^t \nabla \cdot b^{\delta}_{t-s}(Y^{\delta}_s(x)) ds\Big)$
for the backward process $dY^{\delta}_s(x) = -b^{\delta}_{t-s}(Y^{\delta}_s(x)) ds + \sqrt{2} dB_s$ with $Y_0(x) = x$ and $b^{\delta}_t(x) = {\theta_t^{\delta}}^\top \nabla\phi(x)+\epsilon x$.
Since $\Delta\phi$ is bounded (hypothesis \ref{hypothesis_th}), the divergence $\nabla\cdot b^{\delta}_t(x)$ is bounded, there exists $C_b$ such that
\begin{align*}
\sup_{(\delta,t)\in(0,\delta_0]\times[0,T]}\|\nabla\cdot b^{\delta}_t\|_{\infty} \leq C_b .  
\end{align*}

Using this inequality in the Feynman-Kac formula, we prove that  
\begin{align*}
 \sup_{(\delta,t)\in(0,\delta_0]\times[0,T]}\|p^{\delta}_t\|_{\infty} \leq e^{C_bt} \|p_0\|_\infty.  
\end{align*}

Because $\nabla\cdot b_t^{\delta}$ is continuous and bounded with continuous and bounded spatial derivatives, we can use Kunita's theory to compute the derivative of $p_t^{\delta}(x)$ with respect to $x$ from Feynman-Kac formula 

\begin{align*}
\nabla p^{\delta}_t&(x) = \mathbb{E}\Big(\Lambda_t(x)\nabla _xp_0(Y^{\delta}_t(x)) J_{t,t}(x)\Big) \\
&-\mathbb{E}\Big(\int_0^t J_{t,t-s}(x)\Delta b^{\delta}_{t-s}(Y^{\delta}_s(x)) ds\Lambda_t(x)p_0(Y^{\delta}_t(x))\Big)
\end{align*}
where $J_{t,s}(x) = \nabla Y^{\delta}_s(x)$. We derive from the SDE that 
\begin{align*}
    dJ_{t,s}(x) = -\nabla\cdot b^{\delta}_{t-s}(Y_s(x))\cdot J_{t,s}(x)ds.
\end{align*}

Using that $\nabla\cdot b^{\delta}_{t-s}$ is bounded, we prove with Grönwall's lemma, using that $J_{t,0}(x) = \rm Id$, that 

\begin{equation*}
  \forall 0\leq s\leq t\leq T,~ \|J_{t,s} \|_{\infty} \leq e^{sC_b}
\end{equation*}

From this inequality, and using Hypothesis \ref{hypothesis_th_p_0}, we derive that
\begin{align*}
   \sup_{(\delta,t)\in(0,\delta_0]\times[0,T]}\|\nabla p^{\delta}_t \|_{\infty}\leq  e^{2TC_b}( \|\nabla p_0\|_{\infty}+\|p_0\|_{\infty}).
\end{align*}

Because $\phi$ and $\Delta\phi$ have bounded third and fourth order continuous derivatives, we can similarly prove that $\nabla J_{t,s}$ is bounded too, and finally that 
\begin{align*}
 \sup_{(\delta,t)\in(0,\delta_0]\times[0,T]}\|\nabla^2_x p^{\delta}_t \|_{\infty}\leq C(\|\Delta_x p_0\|_{\infty},\|\nabla p_0\|_{\infty},\|p_0\|_{\infty}) 
\end{align*}
for some function $C<\infty$.

The Fokker-Planck Equation proves that $\partial_t p^{\delta}_t$ is bounded on any compact  
\begin{align*}
    \forall K ~\rm{compact}\subset\mathbb{R}^d ,\exists C_K,~ \sup_{(\delta,t)\in\mathbb{R}_+^*\times[0,T]} \| \partial_t p^{\delta}_t\|_{K,\infty} \leq C_K
\end{align*}

$\sup_{(\delta,t)\in(0,\delta_0]\times[0,T]}\|\partial_t \nabla p^{\delta}_t\|_{K,\infty}$ can be bounded with a similar argument.

\end{proof}

\paragraph{Step 6: Extraction of Convergent Subsequence.}

\begin{lemma}[Convergent subsequence]
\label{lemma:subsequence_delta}
There exist $\delta_n \to 0$ and $p_t$ ($\mathcal{C}^2$ with bounded second moment) such that:
\begin{equation}
(p_t^{\delta_n}, \nabla p_t^{\delta_n}) \xrightarrow{\text{pointwise}} (p_t, \nabla p_t).
\end{equation}
Additionally, $\mathbb{E}_{p_t}[\nabla\phi \cdot\nabla\phi^\top]$ is invertible and:
\begin{equation}
\theta_t^{\delta_n} \xrightarrow{\text{uniformly}} \mathbb{E}_{p_t}[\nabla\phi \cdot\nabla\phi^\top]^{-1} \mathbb{E}_{p_t}[\Delta\phi - \epsilon x \cdot\nabla\phi^\top] \underset{\mathrm{def}}{:=} \theta_t.
\end{equation}
\end{lemma}

\begin{proof}
By Lemma~\ref{lemma:bounds_p_delta}, $p^{\delta}_t$ and $\nabla p^{\delta}_t$ are bounded and equicontinuous on $[0,T] \times K$ for any compact $K$. Using Arzel`a-Ascoli, along with a diagonal extraction argument, we can extract a subsequence $p_t^{\delta_n}$ that converges uniformly towards $p_t$ over $[0,T]\times K$, for any compact $K$, which implies pointwise convergence.

Because the family is tight (Lemma \ref{lemma:delta_tight}), dominated convergence shows that $p_t^{\delta_n}$  converges weakly towards $p_t$ uniformly in $t\in[0,T]$, and thus that $p_t$ is a density. 

Using boundedness from Hypothesis \ref{hypothesis_th}, weak convergence implies that 
\begin{equation*}
    \theta_t^{\delta_n} \to \mathbb{E}_{p_t}\big[\nabla\phi\cdot\nabla\phi^\top \big]^{-1} \mathbb{E}_{p_t}\big[\Delta\phi-\epsilon x\cdot\nabla\phi^\top\big]  \underset{\mathrm{def}}{:=} \theta_t
\end{equation*}
where $\mathbb{E}_{p_t}\big[\nabla\phi\nabla\phi^\top \big]$ is invertible because $p_t$ has finite cross entropy.
\end{proof}

\paragraph{Step 7: The Limit Satisfies Fokker-Planck.}

\begin{lemma}[Limit is a solution]
\label{lemma:limit_FPE}
The limit $p_t$ satisfies:
\begin{equation}
\partial_t p_t(x) = \nabla \cdot \big(p_t(\theta_t^\top \nabla\phi(x) + \epsilon x)\big) + \Delta p_t(x),
\end{equation}
with $\theta_t = \mathbb{E}_{p_t}[\nabla\phi \cdot \nabla\phi^\top]^{-1} \mathbb{E}_{p_t}[\Delta\phi - \epsilon x\cdot\nabla\phi^\top]$.
\end{lemma}

\begin{proof}
The density $p_t^\delta$ satisfies Duhamel's formula:
\begin{equation*}
\begin{split}
p_t^\delta(x) = &~~~~(g_t * p_0)(x) \\ &+ \int_0^t \Big(g_{t-s} * \nabla \big(p_s^\delta({\theta_s^\delta}^\top \nabla\phi + \epsilon x)\big)\Big)(x) ds,
\end{split}
\end{equation*}
where $g_t(x) = (4\pi t)^{-d/2} e^{-|x|^2/4t}$. By dominated convergence (using the bounds from Lemma~\ref{lemma:bounds_p_delta}), the same formula holds for $p_t$, where $\theta_s^\delta$ is replaced by $\theta_s$. Taking the time derivative, we show that $p_t$ satisfies the Fokker Planck equation.

\end{proof}

\subsubsection{Convergence to Maximum Entropy}
\label{sec:proofth_convergence}

We now prove that $p_t \to p_*^\epsilon$ in $D_{\mathrm{KL}}$ as $\sigma \to \infty$ (equivalently, as $t \to \infty$ for fixed $\sigma = 1$, since $p_t^\sigma = p_{\sigma^2 t}^1$).

\begin{lemma}[Extraction of convergent subsequence]
\label{lemma:extraction}
There exist $t_n \to \infty$ and $p_\infty$ (with invertible $G(p_\infty)$) such that $p_{t_n} \rightharpoonup p_\infty$ weakly and:
\[
\mathbb{E}_{p_{t_n}}\big[|\nabla \log p_{t_n} + \epsilon x + \theta_\infty^\top \nabla\phi|^2\big] \to 0,
\]
where $\theta_\infty = \mathbb{E}_{p_\infty}[\nabla\phi \cdot\nabla\phi^\top]^{-1} \mathbb{E}_{p_\infty}[\Delta\phi - \epsilon x\cdot \nabla\phi^\top]$.
\end{lemma}

\begin{proof}
Since $H_\epsilon(p_t)$ is increasing (Proposition~\ref{proposition:dentropydD_app_epsilon}) and bounded above, it converges. Thus there exists $t_n \to \infty$ with $\frac{d}{dt} H_\epsilon(p_{t_n}) \to 0$, which equals $\mathbb{E}_{p_{t_n}}[|\nabla \log p_{t_n} + \epsilon x + \theta_{t_n}^\top \nabla\phi|^2] \to 0$.

By tightness and Prokhorov, without loss of generality, we say that $p_{t_n} \rightharpoonup p_\infty$. Upper semi-continuity gives $H_\epsilon(p_\infty) \geq \lim_n H_\epsilon(p_{t_n})$, so $p_\infty$ has finite cross-entropy and is not singularly supported. Thus $\mathbb{E}_{p_\infty}[\nabla\phi \cdot\nabla\phi^\top]$ is invertible.

Weak convergence of $p_{t_n}$ implies $\theta_{t_n} \to \theta_\infty$. Since the Fisher divergence vanishes along $t_n$ and $\nabla\phi$ is bounded (Hypothesis \ref{hypothesis_th}), the same holds with $\theta_\infty$.
\end{proof}

\begin{lemma}[$D_{\mathrm{KL}}$ convergence of subsequence]
\label{lemma:KL_subsequence}
We have $\theta_\infty = \theta_*^\epsilon$ and $D_{\mathrm{KL}}(p_{t_n} \| p_*^\epsilon) \to 0$.
\end{lemma}

\begin{proof}
Since $\phi$ is bounded, $p_{\theta_\infty}(x) = Z_\infty^{-1} e^{-\theta_\infty^\top \phi(x) - \frac12 \epsilon |x|^2}$ has a bounded log-Sobolev constant $c$ by Holley-Stroock. Thus:
\[
D_{\mathrm{KL}}(p_{t_n} \| p_{\theta_\infty}) \leq c \, \mathbb{E}_{p_{t_n}}\big[|\nabla \log p_{t_n} + \epsilon x + \theta_\infty^\top \nabla\phi|^2\big] \to 0.
\]

The distribution $p_{\theta_\infty}$ is exponential with moments $\mathbb{E}_{p_{\theta_\infty}}[\phi] = \mathbb{E}[\phi(X)]$. By uniqueness, $\theta_\infty = \theta_*^\epsilon$ and $p_{\theta_\infty} = p_*^\epsilon$.
\end{proof}

\begin{lemma}[Full convergence]
\label{lemma:full_convergence}
We have $D_{\mathrm{KL}}(p_t \| p_*^\epsilon) \to 0$ as $t \to \infty$.
\end{lemma}

\begin{proof}
For any weakly convergent sequence $p_{t'_n} \rightharpoonup p'_\infty$ with $t'_n \to \infty$, upper semi-continuity gives $-H_\epsilon(p'_\infty) \leq -H_\epsilon(p_*^\epsilon)$, so $p'_\infty = p_*^\epsilon$. By uniqueness of the limit in Prokhorov's theorem, $p_t \rightharpoonup p_*^\epsilon$ and $\theta_t \to \theta_*^\epsilon$.

In the expression
\begin{align*}
D_{\mathrm{KL}}(p_t \| p_*^\epsilon) =& -H_\epsilon(p_t) + \log Z_* Z_\epsilon^{-1} \\
& + {\theta_*^\epsilon}^\top \int p_t(x) \phi(x)dx, 
\end{align*}

each term converges, so $D_{\mathrm{KL}}(p_t \| p_*^\epsilon) \to 0$.
\end{proof}

\end{multicols}
\end{document}